\documentclass{article}

\usepackage[preprint]{neurips_2026}
\usepackage{amsmath}
\usepackage{amsthm}
\usepackage{bbm}

\usepackage[utf8]{inputenc} % allow utf-8 input
\usepackage[T1]{fontenc}    % use 8-bit T1 fonts
\usepackage{hyperref}       % hyperlinks
\hypersetup{colorlinks=true, allcolors=black} % delete for camera-ready
\usepackage{url}            % simple URL typesetting
\usepackage{booktabs}       % professional-quality tables
\usepackage{amsfonts}       % blackboard math symbols
\usepackage{nicefrac}       % compact symbols for 1/2, etc.
\usepackage{microtype}      % microtypography
\usepackage{xcolor}         % colors
\usepackage{colortbl}       % \cellcolor in tables
\definecolor{covheat}{RGB}{44,160,44}   % heatmap green
\usepackage{graphicx}       % \includegraphics
\usepackage{tikz}
\usetikzlibrary{arrows.meta}
\usepackage{placeins}       % \FloatBarrier
\usepackage{multirow}
\numberwithin{equation}{section} 
\newcommand{\lp}{\left(}
\newcommand{\rp}{\right)}

\newcommand{\History}{H}
\newcommand{\history}{h}

\newcommand{\Pbb}{\mathbb{P}}

\newcommand{\Tcal}{\mathcal{T}}

\DeclareMathAlphabet{\mathdutchcal}{U}{dutchcal}{m}{n}

\newcommand\numberthis{\addtocounter{equation}{1}\tag{\theequation}}

\newcommand{\prob}{\Pbb}

\newcommand{\indicator}{\mathbbm{1}}

\newcommand{\Q}{\boldsymbol{Q}}

\newcommand{\beq}{\begin{eqnarray*}}
\newcommand{\eeq}{\end{eqnarray*}}
\newcommand{\beqn}{\begin{eqnarray}}
\newcommand{\eeqn}{\end{eqnarray}}
\newcommand{\ben}{\begin{enumerate}}
\newcommand{\een}{\end{enumerate}}
\newcommand{\bit}{\begin{itemize}}
\newcommand{\eit}{\end{itemize}}

\newcommand{\eps}{\varepsilon}
\newcommand{\vertiii}[1]{{\left\vert\kern-0.25ex\left\vert\kern-0.25ex\left\vert #1 
    \right\vert\kern-0.25ex\right\vert\kern-0.25ex\right\vert}}

\renewcommand{\epsilon}{\eps}

\newcommand{\pow}[1]{^{(#1)}}

\newtheorem{definition}{Definition}
\newtheorem{lemma}{Lemma}

\newtheorem{proposition}[lemma]{Proposition}
\newtheorem{theorem}{Theorem}

\newtheorem{assumption}{Assumption}

\title{Model-based Bootstrap of Controlled Markov Chains}

\author{%
  Ziwei Su\\
  Department of Industrial Engineering and Management Sciences\\
  Northwestern University\\
  Evanston, IL 60208 \\
  \texttt{ziwei.su@northwestern.edu} \\
  \And
  Imon Banerjee \\
  Department of Industrial Engineering and Management Sciences\\
  Northwestern University\\
  Evanston, IL 60208 \\
  \texttt{imon.banerjee@northwestern.edu} \\
  \AND
  Diego Klabjan \\
  Department of Industrial Engineering and Management Sciences\\
  Northwestern University\\
  Evanston, IL 60208 \\
  \texttt{d-klabjan@northwestern.edu}\\
}

\begin{document}

\maketitle

\begin{abstract}
We propose and analyze a model-based bootstrap for transition kernels in finite controlled Markov chains (CMCs) with possibly nonstationary or history-dependent control policies, a setting that arises naturally in offline reinforcement learning (RL) when the behavior policy generating the data is unknown.
We establish distributional consistency of the bootstrap transition estimator in both a single long-chain regime and the episodic offline RL regime.
The key technical tools are a novel bootstrap law of large numbers (LLN) for the visitation counts and a novel use of the martingale central limit theorem (CLT) for the bootstrap transition increments.
We extend bootstrap distributional consistency to the downstream targets of offline policy evaluation (OPE) and optimal policy recovery (OPR) via the delta method by verifying Hadamard differentiability of the Bellman operators, yielding asymptotically valid confidence intervals for value and $Q$-functions.
Experiments on the RiverSwim problem show that the proposed bootstrap confidence intervals (CIs), especially the percentile CIs, outperform the episodic bootstrap and plug-in CLT CIs, and are often close to nominal ($50\%$, $90\%$, $95\%$) coverage, while the baselines are poorly calibrated at small sample sizes and short episode lengths.
\end{abstract}

\section{Introduction}
\label{sec:introduction}
The bootstrap—first proposed by \citet{efron1979bootstrap}—quickly became a cornerstone of modern statistical inference. 
Even when analytic variance formulas are available, bootstrap methods may deliver sharper finite-sample confidence intervals (CIs) with improved coverage; see \citet[p.~14]{hall2013bootstrap}. 
Despite these advantages, existing bootstrap approaches face two important limitations on reinforcement learning setting: (i) bootstrap methods for Markov transition matrices remain comparatively sparse, and (ii) existing approaches are not designed to handle non-stationarity that arises naturally in reinforcement learning.
This limitation is particularly relevant in offline reinforcement learning (RL), where the control policy (behavior policy) that generates the offline data may be nonstationary and history-dependent.
To address this gap, we propose a new \emph{model-based bootstrap} method for transition kernels in finite controlled Markov chains (CMCs) with possibly nonstationary or history-dependent control policies.

To set the stage, we introduce some notation. 
Assume that the state process $\{X_i\}_{i \geq 0}$ takes values in a finite state space $\mathcal S$ with $|\mathcal S| = S$, and the action process $\{A_i\}_{i \geq 0}$ takes values in a finite action space $\mathcal A$ with $|\mathcal A| = A$.
A CMC is a paired state-action process $\bigl\{\lp X_i, A_i\rp\bigr\}_{i\geq 0}$; conditioned on $A_i$, the state sequence $\{X_i\}$ follows a Markov transition kernel $M$, where  
\[M_{s,t}^{(a)} := \mathbb P\lp X_{i+1} = t \;\mid\; X_i = s, A_i = a \rp\] for all time points $i$, $s,t\in\mathcal{S}$, and $a\in\mathcal{A}$ \citep{borkar1991topics}. We observe an offline dataset $\mathcal{D}_n
= \bigl\{ (X^{(k)}_i, A^{(k)}_i, X^{(k)}_{i+1}): k = 1,\dots,K,\; i = 0,\dots,T-1 \bigr\}$ of $K$ episodes generated by an unknown behavior policy $\pi_b$, each of length $T$, with $n=KT$. 
Note that $K$ is allowed to be $1$ and the policy $\pi_b$ may be nonstationary and history-dependent. 

We compute from $\mathcal{D}_n$ the counts $N^{(a)}_{s,t} := \sum_{i,k}\indicator\!\bigl\{X^{(k)}_i=s,\, A^{(k)}_i=a,\, X^{(k)}_{i+1}=t\bigr\}$, $N^{(a)}_s := \sum_t N^{(a)}_{s,t}$, and $N_s := \sum_a N_s^{(a)}$.
We use the counts to compute the count-based empirical transition estimator $\hat M^{(a)}_{s,t} := N^{(a)}_{s,t} / N^{(a)}_s$ and the empirical behavior policy $\hat{\pi}_b(a \mid s) := N_s^{(a)}/N_s$.

\textbf{Model-based Bootstrap:} We now describe the proposed model-based bootstrap method.
Denote the block matrices stacking the entries $M^{(a)}_{s,t}$ and $\hat M^{(a)}_{s,t}$ by $\mathbf{M}$ and $\hat{\mathbf{M}}$, respectively.  
The bootstrap CMC $\{(X_i^\ast, A_i^\ast)\}_{i\ge0}$ has the empirical transition kernel $\hat M$ and the empirical behavior policy $\hat\pi_b$.
We generate the bootstrap dataset $\mathcal{D}_n^\ast = \bigl\{ (X^{\ast (k)}_i, A^{\ast (k)}_i, X^{\ast (k)}_{i+1}): k = 1,\dots,K,\; i = 0,\dots,T-1 \bigr\}$ from the bootstrap CMC with the same episodic structure as $\mathcal{D}_n$, and compute the bootstrap transition estimator $\hat M_{s,t}^{\ast(a)} := N_{s,t}^{\ast(a)} / N_s^{\ast(a)}$ from $\mathcal{D}_n^\ast$, where $N^{\ast (a)}_{s,t} := \sum_{i,k}\indicator\!\bigl\{X^{\ast (k)}_i=s,\, A^{\ast (k)}_i=a,\, X^{\ast (k)}_{i+1}=t\bigr\}$ and $N^{\ast (a)}_s := \sum_t N^{\ast (a)}_{s,t}$.
We denote the block transition matrix stacking the entries $\hat M_{s,t}^{\ast(a)}$ by $\hat{\mathbf{M}}^\ast$.
We repeat the bootstrap $B$ times independently to obtain $B$ bootstrap replicates $\hat{\mathbf{M}}^{\ast(1)},\dots,\hat{\mathbf{M}}^{\ast(B)}$.
The total bootstrap computational cost is $O(Bn) = O(BKT)$, linear in dataset size $n$ and number of bootstrap replicates $B$.

We describe two downstream applications of offline policy evaluation (OPE) and optimal policy recovery (OPR) in offline infinite-horizon discounted RL.
In OPE, for a target policy $\pi$, we use $\hat{\mathbf{M}}$ in the Bellman equations to obtain plug-in estimators $\hat V_\pi$ and $\hat Q_\pi$ of the value function $V_\pi$ and action-value function $Q_\pi$, respectively.
We then substitute $\hat{\mathbf{M}}^{\ast(j)}$ in place of $\hat{\mathbf{M}}$ to obtain bootstrap counterparts $\hat V_\pi^{\ast(j)}$ and $\hat Q_\pi^{\ast(j)}$.
In OPR, we solve the optimal Bellman equations under $\hat{\mathbf{M}}$ to obtain the plug-in estimators $\hat V_\star$ and $\hat Q_\star$ of the optimal value function $V_\star$ and optimal action-value function $Q_\star$, respectively.
We then substitute $\hat{\mathbf{M}}^{\ast(j)}$ in place of $\hat{\mathbf{M}}$ to obtain bootstrap counterparts $\hat V_\star^{\ast(j)}$ and $\hat Q_\star^{\ast(j)}$.
We compute $100(1-\alpha)\%$ CIs for $V_\pi$, $Q_\pi$, or $V_\star$, $Q_\star$ from the empirical $\alpha/2$ and $1-\alpha/2$ quantiles of their $B$ bootstrap replicates.

The focus of this paper is to establish distributional consistency of $\hat{\mathbf{M}}^\ast$ in two settings: the single-chain regime $(K=1,\,T\to\infty)$ in Theorem~\ref{thm:bootstrap-M}, and the episodic offline-RL regime $(T$ fixed$,\,K\to\infty)$ in Proposition~\ref{prop:bootstrap-M-episodic}.
In both settings, conditional on $\mathcal{D}_n$, $\sqrt{n}\,\mathrm{vec}(\hat{\mathbf{M}}^\ast - \hat{\mathbf{M}})$ converges in distribution to the same Gaussian limit as $\sqrt{n}\,\mathrm{vec}(\hat{\mathbf{M}} - \mathbf{M})$.
Proposition~\ref{prop:bootstrap-vq} extends this to OPE and OPR targets via the delta method.
\citet{su2025centrallimittheoremstransition} provides explicit formulas for various distributional limits.
We validate the finite-sample performance of our proposed model-based bootstrap CIs on the RiverSwim problem, where the proposed model-based bootstrap percentile CIs achieve near-nominal coverage while baselines systematically undercover.

We now briefly summarize the key contributions of this paper.
\begin{itemize}
    \item \textbf{Bootstrap distributional consistency.}
    To the best of our knowledge, we establish the first distributional consistency result for the proposed model-based bootstrap in finite CMCs with possibly nonstationary or history-dependent control policies.
    Theorem~\ref{thm:bootstrap-M} covers the single-chain regime $(K=1,\,T\to\infty)$ and Proposition~\ref{prop:bootstrap-M-episodic} covers the episodic offline-RL regime $(T\text{ fixed},\,K\to\infty)$. 
    In both regimes, we establish that conditional on $\mathcal{D}_n$, $\sqrt{n}\,\mathrm{vec}(\hat{\mathbf{M}}^\ast-\hat{\mathbf{M}})$ converges in distribution to the same Gaussian law as $\sqrt{n}\,\mathrm{vec}(\hat{\mathbf{M}}-\mathbf{M})$.
    The proof hinges on a novel bootstrap law of large numbers (LLN) $N_s^{\ast(a)}/n\to p_s^{(a)}$ a.s. and a novel use of the martingale CLT for the bootstrap transition increments.
    (See Section~\ref{sec:bootstrap-consistency}.)

    \item \textbf{Bootstrap validity for OPE and OPR.}
    In Proposition~\ref{prop:bootstrap-vq}, we are the first one to extend bootstrap distributional consistency to the downstream RL targets $V_\pi$, $Q_\pi$, $V_\star$ and $Q_\star$ with non-stationary or history-dependent behavior policies.
    For OPE, we establish that conditional on $\mathcal{D}_n$, $\sqrt{n}(\hat V_\pi^\ast - \hat V_\pi)$ and $\sqrt{n}(\hat Q_\pi^\ast - \hat Q_\pi)$ converge in distribution to the same Gaussian limits as $\sqrt{n}(\hat V_\pi - V_\pi)$ and $\sqrt{n}(\hat Q_\pi - Q_\pi)$, respectively. 
    For OPR, we establish that conditional on $\mathcal{D}_n$, $\sqrt{n}(\hat V_\star^\ast - \hat V_\star)$ converges in distribution to the same Gaussian limit as $\sqrt{n}(\hat V_\star - V_\star)$.
    The proof exploits the delta method by verifying Hadamard differentiability (see Appendix~\ref{app:proof-cor-vq}) of the Bellman operators with respect to the transition block matrix.
    (See Section~\ref{sec:ope-opr-applications}.)

    \item \textbf{Finite-sample coverage on RiverSwim.}
    Proposition~\ref{prop:bootstrap-vq} guarantees the asymptotic validity of the proposed model-based bootstrap CIs for OPE and OPR. 
    In the setting of \citet{zhu2024uncertainty}, we validate the model-based bootstrap CIs empirically on the RiverSwim problem \citep{strehl2004empirical}.
    In a coverage study with $B=1{,}000$ bootstrap replicates and $N_{\mathrm{reps}}=1{,}000$ Monte Carlo replications, the  model-based bootstrap CIs, especially the percentile CI, substantially outperform the episodic bootstrap and plug-in CLT CIs across the OPE and OPR tasks considered.
    The percentile CI attains near-nominal coverage at nominal $50\%$ (about $0.48$--$0.54$), $90\%$ (about $0.87$--$0.94$), and $95\%$ (about $0.92$--$0.97$) in the better-sampled regimes, notably for OPE and for OPR with $T\in\{50,100\}$ at $n\ge 500$.
    (See Section~\ref{sec:experiments}.)
\end{itemize}
The rest of the paper is organized as follows. Section~\ref{sec:background} reviews background and related work. 
Section~\ref{sec:theory} formalizes the problem setup and states the main theoretical results. 
Section~\ref{sec:experiments} presents the numerical experiments.
Section~\ref{sec:conclusion} concludes the paper.

%%%%%%% Current edit ends here %%%%%%%

\section{Related Work}
\label{sec:background}
The practice of using CLTs for statistical inference, such as for Markov chains and controlled Markov chains \citep{billingsley1961statistical,zhu2024uncertainty,su2025centrallimittheoremstransition}, is well established in the literature.
These plug-in CIs, however, require covariance estimation, and the resulting errors degrade finite-sample CI coverage.
The bootstrap addresses this: the intuitive appeal, computational simplicity, and better performance in complex models or nonstandard error structure of the bootstrap \citep{hall1992mean} lead to its widespread use in modern statistical inference \citep{modayil2004towards,han2016bootstrap,nakkiran2020deep}.
The applications of the bootstrap span sparse principal component analysis \citep{babamoradi2013bootstrap,rahoma2021sparse}, time-series forecasting \citep{ruiz2002bootstrapping}, econometric and financial inference \citep{horowitz2019bootstrap,gonccalves2023bootstrapping}, and clustering and classification \citep{jedra2023nearly,jain1987bootstrap}.

Directly relevant to this paper, bootstrap methods are widely applied in Markov chains \citep{kulperger1989bootstrapping,athreya1992bootstrapping}.
Closest in spirit to this work, \citet{horowitz2003bootstrap} introduces a model-based bootstrap for stationary Markov chains by resampling from an estimated transition kernel under geometric mixing and a Cram\'{e}r-type moment condition, while \citet{bertail2006regenerative} develop a regenerative block bootstrap for Harris recurrent chains via Nummelin splitting, requiring explicit identification of a minorization condition and an accessible small set. 
However, these methods do not provide guarantees for the finite CMC setting with possibly nonstationary or history-dependent control policies.

RL is another domain where bootstrap methods are widely applied \citep{krishnamurthy2016partially,faradonbeh2019applications,ramprasad2023online,banerjee2023probably,banerjee2025off,banerjeesmall_2025}. For OPE, \citet{kostrikov2020statistical} apply Efron's bootstrap to empirical transition tuples treated as i.i.d.\ draws from a stationary occupancy distribution. This requires a stationary behavior policy and ignores within-episode dependence.
\citet{hanna2017bootstrapping} propose model-based bootstrap methods that learn transition dynamics from offline data, simulate trajectories under the target policy using the learned model, and bootstrap trajectory-level cumulative rewards to construct high-confidence lower bounds on a policy value. This approach does not establish distributional consistency of the bootstrap estimator. 
\citet{hao2021bootstrapping} establishes bootstrap validity for a finite-horizon episodic bootstrap via fitted Q-evaluation, under the assumption that a large number of independent i.i.d.\ episodes are available.
This assumption does not cover the finite CMC setting with possibly nonstationary or history-dependent control policies, and the episodic bootstrap can be poorly calibrated when the number of episodes is small, which brings us to the following open problems.

\textbf{Open problems.} We address two open problems in the theory of \textit{bootstrap for finite CMCs:} \textbf{(i)} Is it possible to develop a model-based bootstrap for finite CMCs that does not require stationarity or Markovianity of the control policy? \textbf{(ii)} Can we then establish distributional consistency of the bootstrap transition estimator, then lift this result to common RL tasks like OPE and OPR? 

We move on to formally state our main results.
\section{Main Theoretical Results}
\label{sec:theory}

The goal of this section is to establish our main results: distributional consistency of the model-based bootstrap in finite CMCs for the bootstrap transition estimator $\hat M^\ast$ and, via the delta method, for the bootstrap plug-in value and $Q$-function estimators. 

In order to formalize our results, we introduce some additional notation.
% We consider a finite CMC $\{(X_i,A_i)\}_{i\ge0}$ with finite state space $\mathcal{S}$ and finite action space $\mathcal{A}$, with  $S=|\mathcal{S}|$ and $A=|\mathcal{A}|$.
All random variables are defined on a filtered probability space $(\Omega,\mathcal{F},\mathbb{F},\mathbb{P})$, where $\mathbb F := \{\mathcal F_j\}_{j \geq 0}$, $\mathcal{F}_j := \sigma(H_0^j)$ is a filtration with $\mathcal F_j \subset \mathcal F$, and $H_p^j := \{(X_i,A_i)\}_{i=p}^j$ denotes the history from time $p$ to $j$.
% The environment dynamics are governed by an unknown transition kernel $M^{(a)}_{s,t} = \mathbb{P}(X_{i+1}=t \mid X_i=s, A_i=a)$ for $s,t\in\mathcal{S}$, $a\in\mathcal{A}$.
The transition kernel $M$ satisfies the Markov property $M_{s_i,s_{i+1}}^{(a_i)} = \mathbb P\lp X_{i+1} = s_{i+1} \;\mid\; X_i = s_i, A_i = a_i \rp = \mathbb P \lp X_{i+1} = s_{i+1} \;\mid\; H_0^i = h_0^i \rp$,
where $h_0^i := \{X_0=s_0, A_0=a_0, \ldots, X_i=s_i, A_i=a_i\}$ is the sample history up to time $i$.
For each $s\in\mathcal{S}$, let $M_s := [M^{(a)}_{s,t}]_{t\in\mathcal{S},\,a\in\mathcal{A}} \in \mathbb{R}^{S\times A}$, and let $\mathbf{M} := [M_1,\dots,M_S]^\top \in \mathbb{R}^{SA\times S}$ be the block transition matrix stacking $M_{s,t}^{(a)}$ entries.
The empirical block transition matrix is $\hat{\mathbf{M}} := [\hat M_1,\dots,\hat M_S]^\top \in \mathbb{R}^{SA\times S}$, where $\hat M_s := [\hat M^{(a)}_{s,t}]_{t\in\mathcal{S},\,a\in\mathcal{A}}$.

We now turn to the formal statements of our main results.
We begin by introducing some standard assumptions.

\subsection{Assumptions}
\label{sec:assumptions}

The three assumptions below control, respectively, the frequency of return visits to each state--action pair, the rate of mixing of the paired state--action process $\{(X_i,A_i)\}_{i\ge 0}$, and the positivity of the limiting state--action visitation frequencies. Together they underpin the CLT for the empirical transition estimator established by \citet{su2025centrallimittheoremstransition} and ensure that the bootstrap chain inherits the same asymptotic properties.

For every $(s,a)\in\mathcal{S}\times\mathcal{A}$, we recursively define the hitting and return times to state--action pair $(s,a)$ as follows.
\begin{definition}[Hitting and Return Times]
\label{def:return-times}
The first hitting time to $(s,a)$ is $\tau_{s,a}\pow{1} := \min\{i>0: X_i=s,\,A_i=a\}$. For $i\ge 2$, the $i$-th return time is
\[
\tau_{s,a}\pow{i} := \min\!\left\{k>0: X_{\sum_{j=1}^{i-1}\tau_{s,a}\pow{j}+k}=s,\;A_{\sum_{j=1}^{i-1}\tau_{s,a}\pow{j}+k}=a\right\}.
\]
\end{definition}
Assumption~\ref{ass:return-time-growth} requires the conditional expected return time to grow at most sublinearly in the visit index.
\begin{assumption}[Sublinear Return-Time Growth]
\label{ass:return-time-growth}
For any $(s,a)\in\mathcal{S}\times\mathcal{A}$, there exists $\beta\in(0,1)$ such that $\mathbb{E}\!\left[\tau_{s,a}\pow{i}\;\middle|\;\mathcal{F}_{\sum_{p=1}^{i-1}\tau_{s,a}\pow{p}}\right] \le T_i \, \, \text{a.s.}, T_i = O(i^\beta)$.
\end{assumption}
Our next assumption is on the rate of mixing.
We define the weak $\bar\eta$-mixing coefficient for the CMC $\{(X_i,A_i)\}$ as follows.
\begin{definition}[Weak $\bar\eta$-Mixing Coefficient]
\label{def:eta-mix}
For $i<j\le n$, $\mathcal{T}\subseteq(\mathcal{S}\times\mathcal{A})^{n-j+1}$, $s_1,s_2\in\mathcal{S}$, and $a_1,a_2\in\mathcal{A}$, let
\begin{small}
    \begin{align*}
        \eta_{i,j}(\mathcal{T},s_1,s_2,a_1,a_2,h_0^{i-1},n)
        \!:=\! \, &\Bigl|\mathbb{P}\bigl((X_j,A_j,\ldots,X_n,A_n)\in\mathcal{T}\mid X_i=s_1,A_i=a_1,H_0^{i-1}=h_0^{i-1}\bigr) \\
        &\quad -\mathbb{P}\bigl((X_j,A_j,\ldots,X_n,A_n)\in\mathcal{T} \mid X_i=s_2,A_i=a_2,H_0^{i-1}=h_0^{i-1}\bigr)\Bigr|.
    \end{align*}
\end{small}

The weak $\bar\eta$-mixing coefficient for the CMC is
\[
\Bar{\eta}_{i,j}(n)
:=
\sup_{\substack{\Tcal,s_1,s_2,a_1,a_2,\history_0^{i-1},\\
\prob\lp X_i=s_1,A_i=a_1,\History_0^{i-1}=\history_0^{i-1}\rp>0,\\
\prob\lp X_i=s_2,A_i=a_2,\History_0^{i-1}=\history_0^{i-1}\rp>0}}
\eta_{i,j}(\Tcal,s_1,s_2,a_1,a_2,\history_0^{i-1}, n).
\]
\end{definition}
We omit the dependence on $n$ in $\bar\eta_{i,j}$ by convention.
The coefficient $\bar\eta_{i,j}$ quantifies the worst-case influence of the state--action pair at time $i$ on the trajectory distribution from time $j$ onward. Assumption~\ref{ass:eta-mix} imposes a uniform bound on the cumulative decay of $\bar\eta$-mixing coefficients. Sufficient conditions for Assumption~\ref{ass:eta-mix} in terms of separate mixing coefficients for the state and action processes are given in Appendix~\ref{app:mixing}.
\begin{assumption}[Weak Mixing]
\label{ass:eta-mix}
There exists a finite constant $C_\Delta$ such that $\|\Delta_n\| := \max_{1\le i\le n}(1+\bar\eta_{i,i+1}+\cdots+\bar\eta_{i,n}) \le C_\Delta$ for all $n$.
\end{assumption}
We next define the ergodic occupation measure of the CMC as the Cesàro limit of the marginal state--action distributions.
\begin{definition}[Ergodic Occupation Measure]
\label{def:semi-ergodic}
The ergodic occupation measure of the CMC is $p_s^{(a)} := \lim_{n\to\infty} \tfrac{1}{n}\sum_{i=0}^{n-1}\mathbb{P}(X_i=s,\,A_i=a)$, whenever the limit exists for all $(s,a)\in\mathcal{S}\times\mathcal{A}$.
\end{definition}
Assumption~\ref{ass:semi-ergodic} requires the ergodic occupation measure to exist and be uniformly bounded away from zero.
\begin{assumption}[Positive Ergodic Occupation Measure]
\label{ass:semi-ergodic}
The CMC admits an ergodic occupation measure, and $\min_{s,a} p_s^{(a)} \ge p_0 > 0$.
\end{assumption}
\subsection{Bootstrap Distributional Consistency}
\label{sec:bootstrap-consistency}

We now state our main results. Theorem~\ref{thm:bootstrap-M} establishes bootstrap distributional consistency for the single long-chain setting; the episodic extension follows as Proposition~\ref{prop:bootstrap-M-episodic}.

\begin{theorem}[Single Chain Bootstrap Distributional Consistency]
\label{thm:bootstrap-M}
Let $\bar \xi := \sqrt{n}\,\mathrm{vec}(\hat{\mathbf{M}}^\ast - \hat{\mathbf{M}})$ be the vector of differences between $\hat{\mathbf{M}}^\ast$ and $\hat{\mathbf{M}}$.
Suppose $\mathcal{D}_n = \{(X_0,A_0,X_1), \ldots, (X_{n-1},A_{n-1},X_n)\}$ is a single trajectory of length $n$ satisfying Assumptions~\ref{ass:return-time-growth}--\ref{ass:semi-ergodic}, and there exists $(s_0,a_0)\in\mathcal{S}\times\mathcal{A}$ with $M_{s_0,s_0}^{(a_0)}>0$. 
Then, for $\mathbb{P}$-almost every sequence of datasets $(\mathcal{D}_n)_{n\ge 1}$, conditional on $\mathcal{D}_n$,
\[
\bar \xi  = \sqrt{n}\,\mathrm{vec}(\hat{\mathbf{M}}^\ast - \hat{\mathbf{M}})
\;\xrightarrow{\mathrm{d}}\;
\mathcal{N}(0,\bar\Lambda) \text{, as $n \to \infty$}.
\]
The elements of $\bar \Lambda$ are given by $\bar\Lambda_{sat,\,s'a't'}$, which denotes the covariance between the $(s-1)S \cdot A+(a-1)S+t$-th and the $(s'-1)S\cdot A+(a'-1)S+t'$-th elements of $\bar \xi$.
Value $\bar\Lambda_{sat,\,s'a't'}$ has the expression $\bar\Lambda_{sat,\,s'a't'} := \indicator\!\bigl\{(s,a)=(s',a')\bigr\} \frac{1}{p_s^{(a)}}\!\left(\indicator\!\left\{t=t'\right\}M_{s,t}^{(a)} - M_{s,t}^{(a)}M_{s,t'}^{(a)}\right).$
Matrix $\bar\Lambda$ is the asymptotic covariance of $\sqrt{n}\,\mathrm{vec}(\hat{\mathbf{M}}-\mathbf{M})$ given in \citet[Corollary~1]{su2025centrallimittheoremstransition}.
\end{theorem}

\begin{proof}[Proof sketch]
The complete proof is given in Appendix~\ref{app:proofs}; the supporting lemmas are collected in Appendix~\ref{app:lemmas}. We sketch the three main steps here.

\textbf{Step~1 (Bootstrap LLN).}
The pivotal intermediate result, established formally as Lemma~\ref{lem:bootstrap-lln} in Appendix~\ref{app:lem:bootstrap-lln}, is $N_s^{\ast(a)}/n\xrightarrow{\mathrm{a.s.}} p_s^{(a)}$. We write $\tfrac{N_s^{\ast(a)}}{n} = \tfrac{N_s^{\ast(a)}}{\mathbb{E}[N_s^{\ast(a)}\mid\mathcal{D}_n]} \cdot \tfrac{\mathbb{E}[N_s^{\ast(a)}\mid\mathcal{D}_n]}{n}$.
We show that the first term converges to $1$ a.s.
% by the Borel--Cantelli lemma, using \citet[Lemma~6]{banerjee2025off}. 
This requires $\|\Delta_n^\ast\|$ to be eventually bounded, which we establish in Lemma~\ref{lem:mixing-bound} in Appendix~\ref{app:lem:mixing-bound}.
The second term converges to $p_s^{(a)}$ a.s.\ by Lemma~\ref{lem:statdist} in Appendix~\ref{app:lem:statdist}.
\textbf{Step~2 (Decomposition).}
For each $(s,a,t)$, we write $\sqrt{n}(\hat M_{s,t}^{\ast(a)} - \hat M_{s,t}^{(a)}) = S_n^{\ast,sat} + R_n^{\ast,sat}$, where
\[
S_n^{\ast,sat} := \frac{1}{\sqrt{n}\,p_s^{(a)}}\sum_{i=0}^{n-1}\Bigl(\indicator\{X_i^\ast=s,A_i^\ast=a,X_{i+1}^\ast=t\} - \hat M_{s,t}^{(a)}\indicator\{X_i^\ast=s,A_i^\ast=a\}\Bigr)
\]
is the leading martingale sum, and
\[
R_n^{\ast,sat} := \Bigl(\frac{n}{N_s^{\ast(a)}} - \frac{1}{p_s^{(a)}}\Bigr)\frac{1}{\sqrt{n}}\sum_{i=0}^{n-1}\Bigl(\indicator\{X_i^\ast=s,A_i^\ast=a,X_{i+1}^\ast=t\} - \hat M_{s,t}^{(a)}\indicator\{X_i^\ast=s,A_i^\ast=a\}\Bigr).
\]
By Step~1, $n/N_s^{\ast(a)}-1/p_s^{(a)}\xrightarrow{\mathrm{a.s.}}0$.
We show that the sum in $R_n^{\ast,sat}$ has conditional variance $O(n)$, and divided by $\sqrt{n}$ it is $O_\mathbb{P}(1)$.
By the Slutsky's theorem, $R_n^{\ast,sat}=o_\mathbb{P}(1)$, and it suffices to analyze $S_n^{\ast,sat}$.

\textbf{Step~3 (Conditional Martingale CLT).}
We show that the sum $S_n^{\ast,sat}$ is a martingale difference array under joint filtration $\mathcal{F}_{n,i}:=\sigma(\mathcal{D}_n,X_0^\ast,A_0^\ast,\ldots,X_i^\ast,A_i^\ast)$ and $\mathbb{P}$. 
Then, we show that the Lindeberg condition holds for all large enough $n$, the predictable quadratic variation converges a.s.\ to $\bar\Lambda_{sat,sat}$ via Lemma~\ref{lem:bootstrap-lln}, and $\hat M^{(a)}_{s,t}\to M^{(a)}_{s,t}$ a.s. 
We then apply the martingale CLT \citep[Corollary 3.1]{hall1980martingale} conditioned on $\mathcal{D}_n$, yielding $S_n^{\ast,sat}\xrightarrow{\mathrm{d}}\mathcal{N}(0,\bar\Lambda_{sat,sat})$. The multivariate statement follows by the Cramér--Wold theorem  \citep{cramer1936some}.
\end{proof}
We now move to the episodic setting.
Offline RL datasets are almost always episode-structured, arising from an environment that resets between episodes under a behavior policy that may be nonstationary or history-dependent across episodes.
Theorem~\ref{thm:bootstrap-M} does not directly cover this structure, as the episodes are separated by resets and do not form a single uninterrupted CMC.

Here, we observe an offline dataset $\mathcal{D}_n = \bigl\{(X^{(k)}_i, A^{(k)}_i, X^{(k)}_{i+1}): k = 1,\dots,K,\; i = 0,\dots,T-1 \bigr\}$ of $K$ episodes of fixed length $T$ with $n=KT$.
To apply Theorem~\ref{thm:bootstrap-M}, we embed the episodic data into a single chain by inserting, between episode $k$ and episode $k+1$, a pseudo-transition under a dummy reset action $a_\dagger\notin\mathcal{A}$ that maps $X_T^{(k)}$ to the observed initial state $X_0^{(k+1)}$ of the next episode.
The extended action space is $\widetilde{\mathcal{A}}=\mathcal{A}\cup\{a_\dagger\}$, and we use $\widetilde M$ to denote the resulting transition kernel on $\mathcal{S}\times\widetilde{\mathcal{A}}$.
We note that no structural restriction is imposed linking $X_0^{(k+1)}$ to $X_{T}^{(k)}$.

Formally, let $\tilde{\mathcal{D}} := \{(\widetilde{X}_i,\widetilde{A}_i,\widetilde{X}_{i+1})\}_{i=0}^{n'-1}$, with $n'=K(T+1)-1$, be the concatenated single-chain dataset in the format of Theorem~\ref{thm:bootstrap-M}, where $\widetilde{X}_i\in\mathcal{S}$ and $\widetilde{A}_i\in\widetilde{\mathcal{A}}$ are given by 
\[
\begin{aligned}
&\widetilde{X}_{(k-1)(T+1)+j} := X_j^{(k)}, \quad k=1,\ldots,K,\; j=0,\ldots,T; \\
&\widetilde{A}_{(k-1)(T+1)+j} := A_j^{(k)}, \quad k=1,\ldots,K,\; j=0,\ldots,T-1; \\
&\widetilde{A}_{k(T+1)-1} := a_\dagger, \quad k=1,\ldots,K-1.
\end{aligned}
\]
The paired process $\{(\widetilde{X}_i,\widetilde{A}_i)\}_{i \ge 0}$ forms a CMC on $\mathcal{S}\times\widetilde{\mathcal{A}}$ with transition kernel $\widetilde{M}$, where $\widetilde{M}^{(a)}_{s,t} = M^{(a)}_{s,t}$ for $a\in\mathcal{A}$, and we can take $\widetilde{M}^{(a_\dagger)}$ as any stochastic kernel on $\mathcal{S}$ such that $\widetilde{M}^{(a_\dagger)}_{s,t} > 0$ for all $s,t \in \mathcal S$, since $a_\dagger$-transitions are excluded from $\hat{\mathbf{M}}$.

The following proposition extends Theorem~\ref{thm:bootstrap-M} to the episodic setting.
\begin{proposition}
\label{prop:bootstrap-M-episodic}
Given $\mathcal{D}_n = \bigl\{(X^{(k)}_i, A^{(k)}_i, X^{(k)}_{i+1})\bigr\}_{i \in \{0,\dots,T-1\} }$, suppose the CMC $\{(\widetilde{X}_i,\widetilde{A}_i)\}$ with transition kernel $\widetilde{M}$ satisfies Assumptions~\ref{ass:return-time-growth}--\ref{ass:semi-ergodic} on $\mathcal{S}\times\widetilde{\mathcal{A}}$, and there exists $(s_0,a_0)\in\mathcal{S}\times \mathcal{A}$ with $M_{s_0,s_0}^{(a_0)}>0$.
Then, for $\mathbb{P}$-almost every sequence of datasets $(\mathcal{D}_n)_{n\ge 1}$, conditional on $\mathcal{D}_n$, 
\[
\sqrt{n}\,\mathrm{vec}(\hat{\mathbf{M}}^\ast - \hat{\mathbf{M}})
\;\xrightarrow{\mathrm{d}}\;
\mathcal{N}(0,\bar\Lambda), \text{ as $K\to\infty$ and $T$ fixed}.
\]
\end{proposition}
The conditions on $\{(\widetilde X_i,\widetilde A_i)\}$ are satisfied whenever the individual episodes satisfy Assumptions~\ref{ass:return-time-growth}--\ref{ass:semi-ergodic}.
More generally, they accommodate policy dependence across episodes, since they are imposed on the concatenated chain rather than on individual episodes.
% They , and more generally under policy-dependent conditions such as those in \citet[Example~2]{su2025centrallimittheoremstransition}.
The proof of this proposition can be found in Appendix~\ref{app:proof-cor-episodic}.

Before moving to downstream RL applications, we note two important considerations in the model-based bootstrap construction. First, we impose the assumptions of Proposition~\ref{prop:bootstrap-M-episodic} with respect to the original observed data-generating process, not with respect to a fitted reset mechanism. 
Thus, the model-based bootstrap does not require estimating or modeling the initial state (reset) distribution. 
We therefore initialize each bootstrap episode by setting $X_0^{\ast(k)} := X_0^{(k)}$, copying the observed start states directly from $\mathcal{D}_n$. 
This is important because, in the episodic setting, the reset mechanism may be nonstationary or history-dependent and is not identifiable from $D_n$ alone. 
Proposition~\ref{prop:bootstrap-M-episodic} confirms that this choice does not affect bootstrap distributional consistency.

Second, we estimate the behavior policy by the one-step empirical estimator $\hat\pi_b(a\mid s) = N_s^{(a)}/N_s$. 
Theorem~\ref{thm:bootstrap-M} and Proposition~\ref{prop:bootstrap-M-episodic} confirm this one-step estimate suffices even when $\pi_b$ is nonstationary or history-dependent, as the asymptotic distributions depend on the ergodic occupation measure $p_s^{(a)}$ induced by $\pi_b$ rather than its history-dependent structure.

We now move to the downstream RL applications of Theorem~\ref{thm:bootstrap-M} and Proposition~\ref{prop:bootstrap-M-episodic}.
\subsection{Applications to OPE and OPR}
\label{sec:ope-opr-applications}
We consider downstream RL tasks (OPE and OPR) under the offline infinite-horizon discounted RL setting with a known reward function $r:\mathcal{S}\times\mathcal{A}\to\mathbb{R}$.
OPE concerns using $\mathcal{D}_n$ to estimate the value function $V_\pi$ and $Q$-function $Q_\pi$ for a fixed stationary target policy $\pi$.
Let $\Pi = \operatorname{diag}(\pi_1,\dots,\pi_S)\in\mathbb{R}^{S\times SA}$ denote the associated policy matrix, where $\pi_s := [\pi(s,1),\ldots,\pi(s,A)]$.
Vectors $V_\pi$ and $Q_\pi$ are the unique solutions to the Bellman equations \citep{bertsekas2011dynamic,agarwal2019reinforcement}
\begin{equation}
\label{eq:bellman}
V_\pi = (I-\gamma\Pi\mathbf{M})^{-1}g, \qquad Q_\pi = (I-\gamma\mathbf{M}\Pi)^{-1}r,
\end{equation}
where $0<\gamma<1$ is the discount factor, $g(s):=\sum_a \pi(s,a)r(s,a)$, $g := (g(s): s \in \mathcal{S})$, and $r := (r(s,a): s\in\mathcal{S}, a\in\mathcal{A})$.

OPR concerns using $\mathcal{D}_n$ to estimate the optimal policy $\pi_\star$, optimal value function $V_\star$ and optimal $Q$-function $Q_\star$, which are the unique solutions to the optimal Bellman equations
\begin{equation}
\label{eq:bellman-opt}
Q_\star(s,a) = r(s,a) + \gamma\sum_{t\in\mathcal{S}} M^{(a)}_{s,t}\max_{a'\in\mathcal{A}} Q_\star(t,a'),
\qquad
V_\star(s) = \max_{a\in\mathcal{A}} Q_\star(s,a),
\end{equation}
and the optimal policy satisfies $\pi_\star(s)\in\arg\max_{a\in\mathcal{A}}Q_\star(s,a)$ for each $s\in\mathcal{S}$.
We obtain plug-in estimators $\hat V_\pi$ and $\hat Q_\pi$ for OPE by replacing $\mathbf{M}$ with $\hat{\mathbf{M}}$ in \eqref{eq:bellman}:
\begin{equation}
\label{eq:plugin}
\hat V_\pi = (I-\gamma\Pi\hat{\mathbf{M}})^{-1}g,
\qquad
\hat Q_\pi = (I-\gamma\hat{\mathbf{M}}\Pi)^{-1}r.
\end{equation}
For OPR, we obtain plug-in estimators $\hat Q_\star$, $\hat V_\star$, and $\hat\pi_\star$ by solving~\eqref{eq:bellman-opt} with $M^{(a)}_{s,t}$ replaced by $\hat M^{(a)}_{s,t}$.

We use the bootstrap replicates $\hat{\mathbf{M}}^{\ast(1)},\dots,\hat{\mathbf{M}}^{\ast(B)}$ to obtain bootstrap replicates for OPE and OPR targets as follows.
In the $j$-th replicate, we compute bootstrap plug-in value and $Q$-function estimators $\hat V_\pi^{\ast(j)}$ and $\hat Q_\pi^{\ast(j)}$ for OPE by replacing $\mathbf{M}$ with $\hat{\mathbf{M}}^{\ast(j)}$ in \eqref{eq:bellman}:
\begin{equation}
\hat V_\pi^{\ast(j)} = (I-\gamma\Pi\hat{\mathbf{M}}^{\ast(j)})^{-1}g,
\qquad
\hat Q_\pi^{\ast(j)} = (I-\gamma\hat{\mathbf{M}}^{\ast(j)}\Pi)^{-1}r.
\end{equation}
For OPR, we obtain bootstrap plug-in estimators $\hat Q_\star^{\ast(j)}$, $\hat V_\star^{\ast(j)}$, and $\hat\pi_\star^{\ast(j)}$ by solving~\eqref{eq:bellman-opt} with $M^{(a)}_{s,t}$ replaced by $\hat M^{\ast (a,j)}_{s,t}$.
We repeat this process independently for $j=1,\ldots,B$ to obtain $B$ bootstrap replicates $\bigl\{\hat V_\pi^{\ast(j)},\hat Q_\pi^{\ast(j)},\hat V_\star^{\ast(j)}, \hat Q_\star^{\ast(j)}\bigr\}_{j=1}^B$.

We construct CIs for any coordinate of $V_\pi$ from the $B$ bootstrap replicates as follows.
Recall that we denote the plug-in estimator by $\hat V_\pi$, and the bootstrap replicates by $\hat V_\pi^{\ast (1)},\dots,\hat V_\pi^{\ast (B)}$.
We further denote the empirical $\alpha$-quantile of $\{\hat V_\pi^{\ast (j)}(s)\}_{j=1}^B$ for some $s \in \mathcal S$ by $q_\alpha^*$.
The $100(1-\alpha)\%$ bootstrap percentile CI for $V_\pi(s)$ is $\bigl[q_{\alpha/2}^*,\;q_{1-\alpha/2}^*\bigr]$.
The bootstrap pivot CI is $\bigl[2\hat V_\pi(s) - q_{1-\alpha/2}^*,\;2\hat V_\pi(s) - q_{\alpha/2}^*\bigr]$.
The CI constructions for $Q_\pi$, $V_\star$ and $Q_\star$ are analogous.
Constructing CIs for OPE and OPR targets requires solving the Bellman equation for an additional $O(B)$ times, negligible compared to the $O(Bn)$ bootstrap simulation cost.

We now state the bootstrap distributional consistency results for the plug-in value and $Q$-function estimators for OPE and OPR, both in the single-chain and episodic settings.
We define the asymptotic covariance matrices
\begin{align*}
\Sigma_V^\pi &:= \gamma^2\!\left[(I-\gamma\Pi\mathbf{M})^{-1}\Pi\otimes V_\pi^\top\right]\!\bar\Lambda\!\left[\Pi^\top(I-\gamma\mathbf{M}^\top\Pi^\top)^{-1}\otimes V_\pi\right],\\[0.4em]
\Sigma_Q^\pi &:= \gamma^2\!\left[(I-\gamma\mathbf{M}\Pi)^{-1}\otimes Q_\pi^\top\Pi^\top\right]\!\bar\Lambda\!\left[(I-\gamma\Pi^\top\mathbf{M}^\top)^{-1}\otimes \Pi Q_\pi\right].
\end{align*}
We further define $\Pi_\star$ as the policy matrix associated with $\pi_\star$, and $\pi'$ as any policy with policy matrix $\Pi'$.
The following proposition then establishes the bootstrap distributional consistency results to the plug-in value and $Q$-function estimators for OPE and OPR, both in the single-chain and episodic settings.
\begin{proposition}[Bootstrap Validity for OPE and OPR]
\label{prop:bootstrap-vq}
Under the conditions of Theorem~\ref{thm:bootstrap-M} (single-chain) or Proposition~\ref{prop:bootstrap-M-episodic} (episodic), for any fixed target policy $\pi$,
\[
\sqrt{n}\bigl(\hat V_\pi^\ast - \hat V_\pi\bigr) \;\xrightarrow{\mathrm{d}}\; \mathcal{N}(0,\Sigma_V^\pi),
\qquad
\sqrt{n}\bigl(\hat Q_\pi^\ast - \hat Q_\pi\bigr) \;\xrightarrow{\mathrm{d}}\; \mathcal{N}(0,\Sigma_Q^\pi),
\]
conditionally on $\mathcal{D}_n$, in probability, as $n\to\infty$.
Matrices $\Sigma_V^\pi$ and $\Sigma_Q^\pi$ are the asymptotic covariances of $\sqrt{n}(\hat V_\pi - V_\pi)$ and $\sqrt{n}(\hat Q_\pi - Q_\pi)$ given in \citet[Theorem~2]{su2025centrallimittheoremstransition}, and $\Sigma_V^{\pi_\star}$ is the asymptotic covariance of $\sqrt{n}(\hat V_\star - V_\star)$ given in \citet[Theorem~3]{su2025centrallimittheoremstransition}.

In addition, if the following non-degeneracy condition
\[
\min_{s\in\mathcal{S}} \min_{a \neq \pi_\star(s)}\bigl\{Q_\star(s, \pi_\star(s))-Q_\star(s,a)\bigr\}>0
\]
holds, then we have
\[
\sqrt{n}\bigl(\hat V_\star^\ast - \hat V_\star\bigr) \;\xrightarrow{\mathrm{d}}\; \mathcal{N}(0,\Sigma_V^{\pi_\star}), 
\qquad 
\sqrt{n}\bigl(\hat Q_\star^\ast - \hat Q_\star\bigr) \;\xrightarrow{\mathrm{d}}\; \mathcal{N}(0,\Sigma_Q^{\pi_\star}),
\]
conditionally on $\mathcal{D}_n$, in probability, as $n\to\infty$.
\end{proposition}
This proposition guarantees that percentile and pivot bootstrap CIs are asymptotically valid for $V_\pi$, $Q_\pi$, $V_\star$ and $Q_\star$.
The proof of this proposition can be found in Appendix~\ref{app:proof-cor-vq}.

\section{Numerical Experiments}
\label{sec:experiments}
Following existing benchmarks, which broadly focus on the stationary setting, we replicate the small state-space RiverSwim experiment of \citet{zhu2024uncertainty} to evaluate finite-sample CI coverage for the proposed model-based bootstrap.
The RiverSwim MDP has $S=6$ states and $A=2$ actions, with rewards $r(1,0)=1$, $r(6,1)=10$, and $r(s,a)=0$ otherwise, and discount factor $\gamma=0.95$.
Figure~\ref{fig:riverswim} in Appendix~\ref{app:tables} illustrates the transition and reward structure of the RiverSwim setup.
The experiment is deliberately challenging for inference because the behavior policy visits upstream states (near state $6$) rarely, while the optimal policy is nearly degenerate at the left-most state $1$.

We use the stationary behavior policy $\pi_b(1\mid s)=0.8$ for all states, $B=1{,}000$ bootstrap replicates, and $N_{\mathrm{reps}}=1{,}000$ Monte Carlo replications.
For OPE, we evaluate three fixed target policies including uniform ($\pi(1\mid s)=0.5$), mostly-right ($\pi(1\mid s)=0.8$), and mostly-left ($\pi(1\mid s)=0.2$) policies for horizon length $T=50$.
For OPR, we evaluate the optimal value and action-value targets across horizons $T\in\{10,50,100\}$.
We do not evaluate all target policies with different episode lengths in the OPE study to keep the numerical results interpretable and tractable.
Table~\ref{tab:exp-settings} in Appendix~\ref{app:tables} summarizes the settings of all experiments.
We compare percentile and pivot CIs from the proposed model-based bootstrap with the episodic bootstrap \citep{hao2021bootstrapping} and plug-in CLT CIs \citep{zhu2024uncertainty, su2025centrallimittheoremstransition}.

Across the OPE and OPR tasks, the model-based bootstrap, especially the percentile CI, substantially improves coverage relative to the episodic bootstrap and plug-in CLT baselines.
For OPE at $n\ge 500$, and for OPR with $T\in\{50,100\}$ at $n\ge 500$, the model-based bootstrap percentile CI attains near-nominal coverage at nominal $50\%$ (about $0.48$--$0.54$), $90\%$ (about $0.87$--$0.94$), and $95\%$ (about $0.92$--$0.97$).
The CI coverage for model-based bootstrap is weakest for rarely visited state--action pairs (near state $6$) and for the short-horizon OPR setting $T=10$, as expected from the strained asymptotic and non-degeneracy conditions.
Detailed coverage tables and figures are reported in Appendix~\ref{app:tables}.
%%%%%%%%%%%%%%%%%%%%%%%%%%%%%%%%%%%%%%%%%%%%%%%%%%%%%%%%%%%%
\section{Conclusion}
\label{sec:conclusion}
%%%%%%%%%%%%%%%%%%%%%%%%%%%%%%%%%%%%%%%%%%%%%%%%%%%%%%%%%%%%

We propose and analyze a model-based bootstrap for finite CMCs with possibly nonstationary or history-dependent control policies, motivated by the offline RL setting.
We establish the bootstrap distributional consistency of the transition estimators in both the single long chain and episodic RL settings. 
We lift the bootstrap distributional consistency to OPE and OPR targets, exploiting the delta method by verifying Hadamard differentiability of the Bellman operators.

While our theoretical and numerical results are promising, we acknowledge several important considerations on the limitations and future directions of our work. 
First, we assume that the reward function is known in downstream RL tasks to avoid confusion in bootstrap targets. 
When the reward function is
unknown and estimated from data, our bootstrap validity for OPE and OPR continues to hold if the reward estimate converges in probability
to the mean reward, although finite-sample performance may be affected by
the additional estimation error, which warrants further study.
Second, we limit our current analysis to a finite state--action space. 
Extending our result to function approximation or continuous state-action spaces is an important future direction.
Third, we only consider the asymptotic regime of fixed $T$ and $K\to\infty$ for the episodic setting.
We do not consider the asymptotic regime where $K$ is fixed and greater than $1$ and $T\to\infty$, and it remains an open question whether the bootstrap is consistent in this regime.
Lastly, further experiments with non-stationary behavior policies, which the theory accommodates, remain to be conducted.

% \section*{References}
\FloatBarrier
\bibliographystyle{abbrvnat}
\bibliography{biblio}
\newpage
%%%%%%%%%%%%%%%%%%%%%%%%%%%%%%%%%%%%%%%%%%%%%%%%%%%%%%%%%%%%
\FloatBarrier
\appendix

%%%%%%%%%%%%%%%%%%%%%%%%%%%%%%%%%%%%%%%%%%%%%%%%%%%%%%%%%%%%
\section{Sufficient Conditions for Assumption~\ref{ass:eta-mix}}
\label{app:mixing}
%%%%%%%%%%%%%%%%%%%%%%%%%%%%%%%%%%%%%%%%%%%%%%%%%%%%%%%%%%%%

Assumption~\ref{ass:eta-mix} requires a uniform bound on the cumulative weak $\bar\eta$-mixing coefficients of the joint state--action process $\{(X_i,A_i)\}_{i\ge 0}$, which can be difficult to verify directly. 
Following \citet{su2025centrallimittheoremstransition,banerjee2025off}, we give two separate conditions---one for the action process and one for the state process---whose conjunction implies Assumption~\ref{ass:eta-mix}. 
% The result is a special case of \citet[Lemma~3]{banerjee2025off}.

\medskip\noindent\textbf{Mixing of actions.}
% Let $H_0^i := (X_0,A_0,\ldots,X_i,A_i)$ denote the history through time $i$, with realized value $h_0^i$. 
For integers $p > i+j > i \ge 0$, we define the $\gamma$-mixing coefficient
\begin{align*}
\gamma_{p,j,i} :=
\sup_{\substack{s_p,\,\history_{i+j}^{p-1},\,\history_0^i \\ \prob\lp X_p=s_p,\,\History_{i+j}^{p-1}=\history_{i+j}^{p-1},\,\History_0^i=\history_0^i\rp>0}}
&\Bigl\|
\prob\lp A_p=\cdot \,\Big|\, X_p=s_p,\,\History_{i+j}^{p-1}=\history_{i+j}^{p-1},\,\History_0^i=\history_0^i\rp \\
&\quad -\prob\lp A_p=\cdot \,\Big|\, X_p=s_p,\,\History_{i+j}^{p-1}=\history_{i+j}^{p-1}\rp
\Bigr\|_{\mathrm{TV}},
\end{align*}
where $\|\cdot\|_{\mathrm{TV}}$ denotes the total variation distance (TV).
Thus $\gamma_{p,j,i}$ measures the worst-case sensitivity of the action distribution at time $p$ to the distant past $H_0^i$, given the more recent history $H_{i+j}^{p-1}$. 
It thereby quantifies how much the dependence on old history of the behavior policy can alter future action selection.
We now impose a summability condition on these coefficients.
\begin{assumption}[Mixing of Actions]
\label{ass:action-mix}
There exists a constant $C\ge 0$ such that
\[
\sup_{i\ge 1}\sum_{j=1}^{\infty}\sum_{p=i+j+1}^{\infty}\gamma_{p,j,i} \;\le\; \frac{C}{2}.
\]
\end{assumption}
When the sequence of actions $a_p$ is Markovian, i.e. it depends only on the current state $X_p$, then $\gamma_{p,j,i}=0$ for all $p,j,i$; thus, Assumption~\ref{ass:action-mix} holds with $C=0$. 
This covers any stationary or non-stationary Markov behavior policy.

\medskip\noindent\textbf{Mixing of states.}
For all integers $j\ge i\ge 0$, define the $\theta$-mixing coefficient
\begin{align*}
\bar\theta_{i,j} :=
\underset{\substack{(s_1,a_1), (s_2,a_2) \in \mathcal S \times \mathcal A, \;(s_1,a_1)\neq(s_2,a_2)\\ \prob(X_i=s_1,A_i=a_1)>0,\;\prob(X_i=s_2,A_i=a_2)>0}}{\sup}
&\Bigl\|\prob\lp X_j=\cdot\mid X_i=s_1,A_i=a_1\rp \\
&\quad -\prob\lp X_j=\cdot\mid X_i=s_2,A_i=a_2\rp\Bigr\|_{\mathrm{TV}}.
\end{align*}
Thus $\bar\theta_{i,j}$ measures the rate at which the state process forgets its initial state--action pair, extending Dobrushin's coefficients \citep{dobrushin1956central} to the controlled setting.
We now impose a summability condition on these coefficients.
\begin{assumption}[Mixing of States]
\label{ass:state-mix}
There exists a constant $C_\theta\ge 0$ such that
\[
\sup_{i\ge 1}\sum_{j=i+1}^{\infty}\bar\theta_{i,j} \;\le\; C_\theta.
\]
\end{assumption}

Neither assumption implies the other. A CMC with deterministic actions satisfies Assumption~\ref{ass:action-mix} with $C=0$ regardless of state mixing. Conversely, a CMC in which states are drawn i.i.d.\ satisfies Assumption~\ref{ass:state-mix} with $C_\theta=0$ regardless of the behavior policy. The two conditions are therefore independent.

Finally, if Assumptions~\ref{ass:action-mix} and \ref{ass:state-mix} hold, \citet[Lemma~3]{banerjee2025off} implies that $\|\Delta_n\|\le C+C_\theta+1$ for all $n$. 
In particular, Assumption~\ref{ass:eta-mix} holds with $C_\Delta = C+C_\theta+1$.

%%%%%%%%%%%%%%%%%%%%%%%%%%%%%%%%%%%%%%%%%%%%%%%%%%%%%%%%%%%%
\section{Auxiliary Lemmas for Theorem~\ref{thm:bootstrap-M}}
\label{app:lemmas}
%%%%%%%%%%%%%%%%%%%%%%%%%%%%%%%%%%%%%%%%%%%%%%%%%%%%%%%%%%%%

\subsection{Consistency of the Empirical Behavior Policy}
\label{app:lem:policy-consistency}
Before we introduce the consistency of the empirical behavior policy, we introduce some requisite notation.
We define a stationary behavior policy $\bar\pi_b(a\mid s) := p_s^{(a)}/p_s$, where $p_s := \sum_{a'\in\mathcal{A}} p_s^{(a')}$ is the sum of the ergodic occupation measure by actions.
We now state the consistency of the empirical behavior policy.
\begin{lemma}[Consistency of $\hat \pi_b$] 
\label{lem:policy-consistency}
Under Assumptions~\ref{ass:return-time-growth}--\ref{ass:semi-ergodic},
\[
\hat\pi_b(a\mid s) \;\xrightarrow{\mathrm{a.s.}}\; \bar\pi_b(a\mid s)
\qquad \text{for every }(s,a)\in\mathcal{S}\times\mathcal{A}.
\]
\end{lemma}
\begin{proof}
As Assumptions \ref{ass:return-time-growth}--\ref{ass:semi-ergodic} hold, by \citet[Lemma~2]{su2025centrallimittheoremstransition}, $N_s^{(a)}/n \xrightarrow{\mathrm{a.s.}} p_s^{(a)}$ for every $(s,a)$. Summing over $a \in \mathcal A$,
\[
\frac{N_s}{n} = \sum_{a} \frac{N_s^{(a)}}{n} \;\xrightarrow{\mathrm{a.s.}}\; p_s.
\]
Assumption~\ref{ass:semi-ergodic} gives $p_s^{(a)} \ge p_0 > 0$; thus, $p_s \ge p_0 > 0$. The continuous mapping theorem then yields
\[
\hat\pi_b(a\mid s)
= \frac{N_s^{(a)}/n}{N_s/n}
\;\xrightarrow{\mathrm{a.s.}}\;
\frac{p_s^{(a)}}{p_s}
= \bar\pi_b(a\mid s). \qedhere
\]
\end{proof}

\subsection{Properties of the Reference State--Action Chain}
\label{app:lem:ref-chain}

The analysis rests on a reference state--action Markov chain that couples the true transition kernel $M$ with the stationary behavior policy $\bar\pi_b$.
We begin by formally introducing the transition kernel of the reference state-action chain.
\begin{definition}[Reference State--Action Chain]
\label{def:ref-chain}
The reference state--action chain on $\mathcal{S}\times\mathcal{A}$ has one-step transition kernel
\[
K\bigl((s,a),(t,b)\bigr) := M_{s,t}^{(a)}\,\bar\pi_b(b\mid t), \qquad (s,a),(t,b)\in\mathcal{S}\times\mathcal{A}.
\]
\end{definition}

We record the key structural properties of $K$ that are used throughout the proofs.

\begin{lemma}[Properties of the Reference State--Action Chain]
\label{lem:ref-chain}
Under Assumptions~\ref{ass:return-time-growth}--\ref{ass:semi-ergodic} and the self-loop condition $M_{s_0,s_0}^{(a_0)}>0$, we have
\begin{enumerate}
  \item[(i)] $p=(p_s^{(a)})_{(s,a) \in \mathcal S \times \mathcal A}$ is the unique stationary distribution of $K$,
  \item[(ii)] $K$ is irreducible and aperiodic on $\mathcal{S}\times\mathcal{A}$.
\end{enumerate}
\end{lemma}

\begin{proof}
\textit{(i) Unique stationarity.}
% Let $p_s := \sum_{a'} p_s^{(a')}$. 
Since $p$ is the ergodic occupation measure under $(M,\pi_b)$, we have
\[
p_s = \sum_{(t,l)} p_t^{(l)}\,M_{t,s}^{(l)}.
\]
Therefore
\[
\sum_{(t,l)\in\mathcal{S}\times\mathcal{A}} p_t^{(l)}\,K\bigl((t,l),(s,a)\bigr)
= \Bigl(\sum_{(t,l)} p_t^{(l)} M_{t,s}^{(l)}\Bigr)\bar\pi_b(a\mid s)
= p_s\cdot\frac{p_s^{(a)}}{p_s} = p_s^{(a)},
\]
hence $p$ satisfies $pK=p$. An irreducible Markov chain on a finite state space has a unique stationary distribution; uniqueness then follows from part~(ii).

\textit{(ii) Irreducibility and aperiodicity.}
By Assumption~\ref{ass:return-time-growth}, in the original CMC with
control policy \(\pi_b\), the hitting and return times to every
state--action pair are finite a.s.; hence every state--action pair is hit
and returned to infinitely often a.s. Therefore, for any two
state--action pairs \((s,a),(t,b)\in\mathcal S\times\mathcal A\), there exists a
finite path from \((s,a)\) to \((t,b)\) with positive probability under the
original CMC.

By Assumption~\ref{ass:semi-ergodic}, \(p_s^{(a)}\ge p_0>0\) for all
\((s,a)\); hence
\[
\bar \pi_b(a\mid s)=\frac{p_s^{(a)}}{p_s}>0
\]
for all \((s,a)\in\mathcal S\times\mathcal A\). Therefore, every
positive-probability path segment under the original CMC has positive
probability under the reference state--action kernel \(K\), because each
state transition along the path has positive transition probability and
\(\bar\pi_b\) assigns positive probability to every action at every state.
Thus \(K\) is irreducible.

For aperiodicity, Assumption~\ref{ass:semi-ergodic} gives
\[
p_s = \sum_{a'} p_s^{(a')} \ge p_0 > 0
\qquad\text{and}\qquad
\bar\pi_b(a\mid s) = \frac{p_s^{(a)}}{p_s} > 0
\qquad\text{for all }(s,a).
\]

Then by the self-loop condition, we have
\[
K\bigl((s_0,a_0),(s_0,a_0)\bigr) = M_{s_0,s_0}^{(a_0)}\,\bar\pi_b(a_0\mid s_0) > 0.
\]
A positive self-loop in a finite irreducible chain forces the period of the entire chain to be one; thus, $K$ is aperiodic.
\end{proof}

\subsection{Eventual Mixing Bound for the Bootstrap CMC}
\label{app:lem:mixing-bound}
The bootstrap CMC $\{(X_i^\ast,A_i^\ast)\}$ replaces the true kernel $M$ and the stationary policy $\bar\pi_b$ with their empirical counterparts $\hat{M}$ and $\hat{\pi}_b$, respectively.
Let $\History_p^{\ast j}$ denote the history of the bootstrap CMC from time $p$ to $j$, and $h_p^{\ast j} := \{X_p^\ast=s_p, A_p^\ast=a_p, \ldots, X_j^\ast=s_j, A_j^\ast=a_j\}$.
We define the weak mixing coefficients for the bootstrap CMC $\{(X_i^\ast,A_i^\ast)\}$ as follows.
\begin{definition}[Weak Mixing Coefficients for the Bootstrap CMC]
% \label{def:eta-mix}
For $i<j\le n$, $\mathcal{T}\subseteq(\mathcal{S}\times\mathcal{A})^{n-j+1}$, $s_1,s_2\in\mathcal{S}$, and $a_1,a_2\in\mathcal{A}$, let $\eta_{i,j}^\ast(\mathcal{T},s_1,s_2,a_1,a_2,h_0^{\ast i-1},n):= \bigl|\mathbb{P}\bigl((X_j^\ast,A_j^\ast,\ldots,X_n^\ast,A_n^\ast)\in\mathcal{T}\mid X_i^\ast=s_1,A_i^\ast=a_1,H_0^{\ast i-1}=h_0^{\ast i-1}\bigr)-\mathbb{P}\bigl((X_j^\ast,A_j^\ast,\ldots,X_n^\ast,A_n^\ast)\in\mathcal{T}\mid X_i^\ast=s_2,A_i^\ast=a_2,H_0^{\ast i-1}=h_0^{\ast i-1}\bigr)\bigr|.$
The weak $\bar\eta$-mixing coefficient for the bootstrap CMC is 
\[
\Bar{\eta}_{i,j}^\ast(n):= \underset{\color{black}\substack{\Tcal,s_1,s_2,a_1,a_2,\history_0^{\ast i-1},\\\prob\lp X_i^\ast=s_1,A_i^\ast=a_1,\History_0^{\ast i-1}=\history_0^{\ast i-1}\rp>0,\\\prob\lp  X_i^\ast=s_2,A_i^\ast=a_2,\History_0^{\ast i-1}=\history_0^{\ast i-1}\rp>0}}{\sup} \eta_{i,j}^\ast(\Tcal,s_1,s_2,a_1,a_2,\history_0^{\ast i-1}, n).
\]
Then, the bootstrap CMC counterpart of $\|\Delta_n\|$ in Assumption~\ref{ass:eta-mix} is
$\|\Delta_n^\ast\| := \max_{1\le i\le n}(1+\bar\eta_{i,i+1}^\ast+\cdots+\bar\eta_{i,n}^\ast)$.
\end{definition}
We formally introduce the transition kernel of the Markov chain of bootstrap state--action pairs that is used throughout the proofs.
\begin{definition}[Bootstrap State--Action Chain]
\label{def:bootstrap-chain}
The \emph{bootstrap state--action chain} on $\mathcal{S}\times\mathcal{A}$ has one-step transition kernel
\[
K_n^\ast\bigl((s,a),(t,b)\bigr) := \hat M_{s,t}^{(a)}\,\hat\pi_b(b\mid t),
\qquad (s,a),(t,b)\in\mathcal{S}\times\mathcal{A}.
\]
\end{definition}
The following lemma shows that the bootstrap CMC $\{(X_i^\ast,A_i^\ast)\}$ inherits the mixing properties of the CMC $\{(X_i,A_i)\}$ in Assumption~\ref{ass:eta-mix} for all sufficiently large $n$, almost surely.
\begin{lemma}[Eventual Mixing Bound for Bootstrap Chain]
\label{lem:mixing-bound} 
There almost surely exist a finite (random) $N_0$ and a deterministic constant $C_\Delta^\ast<\infty$ such that
\[
\|\Delta_n^\ast\| \le C_\Delta^\ast \qquad \text{for all } n\ge N_0.
\]
\end{lemma}

\begin{proof}
We proceed in five steps. Let $N_0$ be the (a.s.\ finite) random index beyond which all the entrywise convergences invoked in Steps~1--4 have taken effect simultaneously; its finiteness follows from the fact that each individual a.s.\ convergence has a finite threshold, and a finite maximum of a.s.-finite random variables is a.s.\ finite. All statements below hold for all $n\ge N_0$ a.s.

\medskip\noindent\textbf{Step 1 (Entrywise convergence).}
By Lemma~\ref{lem:policy-consistency}, $\hat\pi_b\to\bar\pi_b$ a.s. By \citet[Lemma~2]{su2025centrallimittheoremstransition}, $\hat M_{s,t}^{(a)}\to M_{s,t}^{(a)}$ a.s.\ for every $(s,a,t)$. Hence $K_n^\ast\to K$ entrywise a.s.

\medskip\noindent\textbf{Step 2 (Eventual irreducibility).}
By Assumption~\ref{ass:semi-ergodic} and \citet[Lemma~2]{su2025centrallimittheoremstransition}, $N_s^{(a)}/n\to p_s^{(a)}\ge p_0>0$ a.s.
Thus, $N_s^{(a)}>0$ for all $(s,a)$ and all $n\ge N_0$. When this holds, $\hat\pi_b(a \mid s)>0$ and $\hat M_{s,t}^{(a)}>0$ for all $(s,a,t)$ such that $M_{s,t}^{(a)}>0$.

Since $K$ is irreducible by Lemma~\ref{lem:ref-chain}  (ii), for any $z,w\in\mathcal{S}\times\mathcal{A}$, there exists a finite sequence $z=z_0,z_1,\ldots,z_\ell=w$ such that $K(z_{i-1},z_i)>0$ for each $i=1,\ldots,\ell$. Each such step satisfies $K(z_{i-1},z_i)=M_{s,t}^{(a)}\bar\pi_b(l\mid t)$ for the corresponding $(s,a,t,l)$, which requires $M_{s,t}^{(a)}>0$.
Since $\hat M_{s,t}^{(a)}>0$ for all $n\ge N_0$ whenever $M_{s,t}^{(a)}>0$, we have $K_n^\ast(z_{i-1},z_i)>0$ for every step. 
Thus, the same path has strictly positive probability under $K_n^\ast$. Hence $K_n^\ast$ is irreducible for all $n\ge N_0$.

\medskip\noindent\textbf{Step 3 (Eventual aperiodicity).}
Since $N_{s_0,s_0}^{(a_0)}/n\to p_{s_0}^{(a_0)}M_{s_0,s_0}^{(a_0)}>0$ a.s., we have $N_{s_0,s_0}^{(a_0)}>0$ for all $n\ge N_0$; thus, $\hat M_{s_0,s_0}^{(a_0)}>0$, which gives
\[
K_n^\ast\bigl((s_0,a_0),(s_0,a_0)\bigr) = \hat M_{s_0,s_0}^{(a_0)}\,\hat\pi_b(a_0\mid s_0)>0.
\]
A positive self-loop in the irreducible $K_n^\ast$ forces aperiodicity for all $n\ge N_0$.

\medskip\noindent\textbf{Step 4 (Uniform Doeblin minorization and geometric ergodicity).}
Since $K$ is finite, irreducible, and aperiodic by Lemma~\ref{lem:ref-chain} (ii), it is primitive. Thus, there exists $r\ge 1$ such that $K^r(z,w)>0$ for all $z,w\in\mathcal{S}\times\mathcal{A}$, where
\[
K^r(z,w) \;:=\; \sum_{z_1,\ldots,z_{r-1}\in\mathcal{S}\times\mathcal{A}}
K(z,z_1)\,K(z_1,z_2)\cdots K(z_{r-1},w)
\]
is the $r$-step transition kernel ($K^1 := K$).
We define
\[
\eta := \min_{z,w\in\mathcal{S}\times\mathcal{A}} K^r(z,w) > 0.
\]
Since entrywise convergence is preserved under matrix multiplication, $K_n^{\ast r}\to K^r$ entrywise a.s.; thus, for all $n\ge N_0$,
\[
K_n^{\ast r}(z,w) \;\ge\; \frac{\eta}{2} \qquad \forall\, z,w\in\mathcal{S}\times\mathcal{A}.
\]
Let $\nu(z) := 1/(SA)$ for all $z\in\mathcal{S}\times\mathcal{A}$ be the uniform distribution on $\mathcal{S}\times\mathcal{A}$, and set $\delta := SA\,\eta/2 > 0$. Then for all $n\ge N_0$ a.s., $K_n^{\ast r}(z,\cdot)\ge\delta\,\nu(\cdot)$ for all $z$. This is a uniform Doeblin minorization with deterministic $r,\delta,\nu$ depending only on $K$.

We now derive the geometric ergodicity of $K_n^\ast$ for all $n\ge N_0$ a.s. from the uniform Doeblin minorization.
The first step is to derive a TV contraction for $K_n^{\ast r}$ with contraction factor $1-\delta$: for any two probability measures $\mu_1,\mu_2$ on $\mathcal{S}\times\mathcal{A}$,
\[
\|\mu_1 K_n^{\ast r} - \mu_2 K_n^{\ast r}\|_{\mathrm{TV}}
\;\le\; (1-\delta)\|\mu_1-\mu_2\|_{\mathrm{TV}}.
\] 
We now derive the TV contraction explicitly. We write each row of $K_n^{\ast r}$ as
\[
K_n^{\ast r}(z,\cdot) = \delta\,\nu(\cdot) + (1-\delta)\,R_z(\cdot),
\qquad
R_z(\cdot) := \frac{K_n^{\ast r}(z,\cdot)-\delta\,\nu(\cdot)}{1-\delta},
\]
where $R_z$ is a probability measure. 
For any probability measure $\mu$ on $\mathcal{S}\times\mathcal{A}$, we have
\begin{align*}
(\mu K_n^{\ast r})(w)
&= \sum_z \mu(z)\,K_n^{\ast r}(z,w)
 = \sum_z \mu(z)\bigl[\delta\,\nu(w)+(1-\delta)\,R_z(w)\bigr]\\
&= \delta\,\nu(w) + (1-\delta)\sum_z \mu(z)\,R_z(w),
\end{align*}
where the last step uses $\sum_z\mu(z)=1$. 
Then, for any two probability measures $\mu_1,\mu_2$ on $\mathcal{S}\times\mathcal{A}$, we subtract $\mu_1 K_n^{\ast r}$ from $\mu_2 K_n^{\ast r}$ to obtain 
\[
\mu_1 K_n^{\ast r} - \mu_2 K_n^{\ast r}
= (1-\delta)\!\left(\sum_z \mu_1(z)\,R_z - \sum_z \mu_2(z)\,R_z\right).
\]
Taking the TV norm and applying the triangle inequality, we get
\begin{align*}
\|\mu_1 K_n^{\ast r} - \mu_2 K_n^{\ast r}\|_{\mathrm{TV}}
&= (1-\delta)\Bigl\|\sum_z (\mu_1(z)-\mu_2(z))\,R_z\Bigr\|_{\mathrm{TV}}\\
&= \frac{1-\delta}{2}\sum_w \Bigl|\sum_z (\mu_1(z)-\mu_2(z))\,R_z(w)\Bigr|\\
&\le \frac{1-\delta}{2}\sum_z |\mu_1(z)-\mu_2(z)|\sum_w R_z(w)\\
&= (1-\delta)\,\|\mu_1-\mu_2\|_{\mathrm{TV}},
\end{align*}
where the final equality uses $\sum_w R_z(w)=1$ since each $R_z$ is a probability measure.

Let $\hat p^\ast:=(\hat p_s^{\ast(a)})_{(s,a) \in \mathcal S \times \mathcal A}$ denote the stationary distribution of $K_n^\ast$. 
For all $n \ge N_0$, since $K_n^\ast$ is irreducible and aperiodic a.s., $\hat p^\ast$ is the unique stationary distribution of $K_n^\ast$, and $\hat p^\ast K_n^\ast = \hat p^\ast$.

We use this stationarity to bound $\|K_n^{\ast mr}(z,\cdot)-\hat p^\ast(\cdot)\|_{\mathrm{TV}}$ by induction on $m$. When $m=0$, the bound holds trivially since $\|\delta_z-\hat p^\ast\|_{\mathrm{TV}}\le 1=(1-\delta)^0$. Suppose the bound holds for some $m\ge 0$, i.e.\ $\|\delta_z K_n^{\ast mr}-\hat p^\ast\|_{\mathrm{TV}}\le(1-\delta)^m$. Applying the TV contraction with $\mu_1=\delta_z K_n^{\ast mr}$ and $\mu_2=\hat p^\ast$, and rewriting $K_n^{\ast(m+1)r}=K_n^{\ast mr}K_n^{\ast r}$, we have
\begin{align*}
\bigl\|\delta_z K_n^{\ast(m+1)r}-\hat p^\ast\bigr\|_{\mathrm{TV}}
&= \bigl\|(\delta_z K_n^{\ast mr})K_n^{\ast r} - \hat p^\ast K_n^{\ast r}\bigr\|_{\mathrm{TV}} \\
&\le\; (1-\delta)\bigl\|\delta_z K_n^{\ast mr}-\hat p^\ast\bigr\|_{\mathrm{TV}} \\
&\le\; (1-\delta)^{m+1},
\end{align*}
where the second inequality applies the induction hypothesis. Hence for all $m\ge 0$,
\[
\sup_{z}\bigl\|K_n^{\ast mr}(z,\cdot)-\hat p^\ast(\cdot)\bigr\|_{\mathrm{TV}} \;\le\; (1-\delta)^m.
\]
For any $t\ge 0$, we can write $t=mr+s$ with $m\ge 0$ and $0\le s<r$. 
Applying the TV contraction to the pair $(\delta_z K_n^{\ast s},\,\hat p^\ast)$ over $m$ further blocks of $r$ steps, and using $\hat p^\ast K_n^{\ast mr}=\hat p^\ast$, we have
\begin{align*}
\bigl\|K_n^{\ast t}(z,\cdot)-\hat p^\ast(\cdot)\bigr\|_{\mathrm{TV}}
&= \bigl\|\delta_z K_n^{\ast s} K_n^{\ast mr} - \hat p^\ast K_n^{\ast mr}\bigr\|_{\mathrm{TV}} \\
& \le \; (1-\delta)^m \bigl\|\delta_z K_n^{\ast s}-\hat p^\ast\bigr\|_{\mathrm{TV}} \\
& \le\; (1-\delta)^m.
\end{align*}
We define $\rho_\ast:=(1-\delta)^{1/r}\in(0,1)$ and $C_\ast:=(1-\delta)^{-1}<\infty$. 
They are deterministic constants depending only on $K$, since $\delta$ depends only on $K$.
As $(1-\delta)^m=(1-\delta)^{(t-s)/r}\le(1-\delta)^{-1}(1-\delta)^{t/r}=C_\ast\rho_\ast^t$, we conclude
\[
\sup_{(s,a)\in\mathcal{S}\times\mathcal{A}}\bigl\|K_n^{\ast t}\bigl((s,a),\cdot\bigr)-\hat p^\ast(\cdot)\bigr\|_{\mathrm{TV}}
\;\le\; C_\ast\rho_\ast^t, \qquad t\ge 0.
\]

\medskip\noindent\textbf{Step 5 (Applying \citet[Lemma 3]{banerjee2025off}).}
Since $\hat\pi_b$ is a stationary Markov policy, Assumption~\ref{ass:action-mix} holds for the bootstrap CMC with $C=0$. It remains to verify Assumption~\ref{ass:state-mix} for the bootstrap CMC. 

We define the $\bar \theta$-mixing coefficient of the bootstrap CMC as 
\begin{align*}
\bar\theta_{i,j}^\ast :=
\underset{\substack{(s_1,a_1), (s_2,a_2) \in \mathcal S \times \mathcal A, \;(s_1,a_1)\neq(s_2,a_2)\\ \prob(X_i^\ast=s_1,A_i^\ast=a_1)>0,\;\prob(X_i^\ast=s_2,A_i^\ast=a_2)>0}}{\sup}
&\Bigl\|\prob\lp X_j^\ast=\cdot\mid X_i^\ast=s_1,A_i^\ast=a_1\rp \\
&\quad -\prob\lp X_j^\ast=\cdot\mid X_i^\ast=s_2,A_i^\ast=a_2\rp\Bigr\|_{\mathrm{TV}}.
\end{align*}
We connect the bound in Step~4 to the $\bar\theta$-mixing coefficient as follows.

For any two distributions $\mu,\nu$ on $\mathcal{S}\times\mathcal{A}$ with $\mathcal{S}$-marginals $\mu_X,\nu_X$, marginalization cannot increase total variation and thus for every $B\subseteq\mathcal{S}$,
\[
|\mu_X(B)-\nu_X(B)| = |\mu(B\times\mathcal{A})-\nu(B\times\mathcal{A})| \le \|\mu-\nu\|_{\mathrm{TV}}.
\]
Hence $\|\mu_X-\nu_X\|_{\mathrm{TV}}\le\|\mu-\nu\|_{\mathrm{TV}}$. Applying this to $\mu = K_n^{\ast(j-i)}((s_k,a_k),\cdot)$ and $\nu = \hat p^\ast$, whose $\mathcal{S}$-marginal is denoted by $\hat\mu_X^\ast$, gives
\[
\bigl\|\mathbb{P}(X_j^\ast=\cdot\mid X_i^\ast=s_k,A_i^\ast=a_k)-\hat\mu_X^\ast(\cdot)\bigr\|_{\mathrm{TV}}
\le \bigl\|K_n^{\ast(j-i)}\bigl((s_k,a_k),\cdot\bigr)-\hat p^\ast(\cdot)\bigr\|_{\mathrm{TV}}
\le C_\ast\rho_\ast^{j-i}.
\]
By the triangle inequality, we have
\[
\bar\theta_{i,j}^\ast \le 2C_\ast\rho_\ast^{j-i}, \qquad j>i\ge 1.
\]
Therefore
\[
\sup_{i\ge 1}\sum_{j>i}\bar\theta_{i,j}^\ast \;\le\; 2 \sum_{j=1}^{\infty} C_\ast \rho_\ast^j \;=\; \frac{2C_\ast\rho_\ast}{1-\rho_\ast} < \infty,
\]
and Assumption~\ref{ass:state-mix} holds for the bootstrap CMC. 
Let $C_\Delta^\ast := \frac{2C_\ast\rho_\ast}{1-\rho_\ast}+1$.
\citet[Lemma~3]{banerjee2025off} then implies that Assumption~\ref{ass:eta-mix} holds for the bootstrap CMC, and $\|\Delta_n^\ast\|\le 0+\frac{2C_\ast\rho_\ast}{1-\rho_\ast}+1 = C_\Delta^\ast$ for all $n\ge N_0$ a.s. 
This completes the proof.
\end{proof}

\subsection{Continuity of the Stationary Distribution Map}
\label{app:lem:statdist}

\begin{lemma}[Continuity of the Stationary Distribution Map]
\label{lem:statdist}
On the almost sure event of Lemma~\ref{lem:mixing-bound},
\[
\hat p_s^{\ast(a)} \;\xrightarrow{\mathrm{a.s.}}\; p_s^{(a)}
\qquad\text{for every }(s,a)\in\mathcal{S}\times\mathcal{A}.
\]
\end{lemma}

\begin{proof}
By Lemma~\ref{lem:ref-chain}, $K$ is irreducible and aperiodic on the finite state space $\mathcal{S}\times\mathcal{A}$. Since every irreducible finite-state chain is positive recurrent, $K$ is ergodic, and its stationary distribution $p$ is unique. 

By Lemma~\ref{lem:mixing-bound}, $K_n^\ast$ is irreducible and aperiodic for all $n\ge N_0$ a.s. Therefore, $K_n^\ast$ is eventually ergodic, with unique stationary distribution $\hat p^\ast$ for each $n\ge N_0$. Then by \citet[Theorem~4.1]{meyer1980condition}, the limiting probability vector of an ergodic chain is a continuous function of the transition matrix. Since $K_n^\ast\to K$ entrywise a.s.\ by Lemma~\ref{lem:mixing-bound} Step~1, we conclude $\hat p_s^{\ast(a)}\xrightarrow{\mathrm{a.s.}} p_s^{(a)}$ for every $(s,a)\in\mathcal{S}\times\mathcal{A}$.
\end{proof}

\subsection{Bootstrap Visitation Count Law of Large Numbers}
\label{app:lem:bootstrap-lln}

\begin{lemma}[Bootstrap Visitation Count LLN]
\label{lem:bootstrap-lln}
For each $(s,a)\in\mathcal{S}\times\mathcal{A}$,
\[
\frac{N_s^{\ast(a)}}{n} \;\xrightarrow{\mathrm{a.s.}}\; p_s^{(a)}.
\]
\end{lemma}

\begin{proof}
We begin by defining $\mu_n^\ast := \mathbb{E}[N_s^{\ast(a)}\mid\mathcal{D}_n]$ and writing
\[
\frac{N_s^{\ast(a)}}{n}
= \frac{N_s^{\ast(a)}}{\mu_n^\ast} \cdot \frac{\mu_n^\ast}{n}
= \mathrm{Term~1} \cdot \mathrm{Term~2}.
\]
We handle Term~2 first, then Term~1.

\medskip\noindent\textbf{Term 2 (Expectation convergence).}
Let $\nu_n^\ast$ denote the initial distribution of $(X_0^\ast,A_0^\ast)$ under $K_n^\ast$. Then
\[
\mu_n^\ast
= \sum_{i=0}^{n-1}(\nu_n^\ast K_n^{\ast i})(s,a),
\]
hence $\mu_n^\ast - n\hat p_s^{\ast(a)} = \sum_{i=0}^{n-1}[(\nu_n^\ast K_n^{\ast i})(s,a) - \hat p_s^{\ast(a)}]$.
By Lemma~\ref{lem:mixing-bound} (Step~4), $\sup_z\|K_n^{\ast i}(z,\cdot)-\hat p^\ast\|_{\mathrm{TV}}\le C_\ast\rho_\ast^i$ for all $n\ge N_0$ a.s. Since $\nu_n^\ast$ is a probability measure, we have
\[
\|\nu_n^\ast K_n^{\ast i}-\hat p^\ast\|_{\mathrm{TV}}
\;\le\; \sum_z \nu_n^\ast(z)\,\bigl\|K_n^{\ast i}(z,\cdot)-\hat p^\ast\bigr\|_{\mathrm{TV}}
\;\le\; C_\ast\rho_\ast^i.
\]
Therefore, for all $n\ge N_0$ a.s., we obtain
\[
\bigl|\mu_n^\ast - n\hat p_s^{\ast(a)}\bigr|
\;\le\; \sum_{i=0}^{\infty}C_\ast\rho_\ast^i = \frac{C_\ast}{1-\rho_\ast} < \infty.
\]
Hence $\mu_n^\ast/n = \hat p_s^{\ast(a)} + O(1/n)$. Since $\hat p_s^{\ast(a)}\xrightarrow{\mathrm{a.s.}}p_s^{(a)}$ by Lemma~\ref{lem:statdist}, we finally derive
\[
\frac{\mu_n^\ast}{n} \;\xrightarrow{\mathrm{a.s.}}\; p_s^{(a)}.
\]

\medskip\noindent\textbf{Term 1 (Concentration).}
On the probability-one event where $N_0<\infty$ (Lemma~\ref{lem:mixing-bound}), we choose $N_1\ge N_0$ such that $\mu_n^\ast\ge(p_s^{(a)}/2)\,n$ and $\|\Delta_n^\ast\|\le C_\Delta^\ast$ for all $n\ge N_1$ (which holds a.s.\ by Term~2 and Lemma~\ref{lem:mixing-bound}). Applying \citet[Lemma~6]{banerjee2025off} to the bootstrap chain with kernel $K_n^\ast$, we get that for every $\varepsilon>0$ and $n\ge N_1$,
\[
\mathbb{P}\!\left(\left|\frac{N_s^{\ast(a)}}{\mu_n^\ast}-1\right|>\varepsilon\right)
\;\le\;
2\exp\!\left(-\frac{{\mu_n^\ast}^2\varepsilon^2}{2n\|\Delta_n^\ast\|^2}\right)
\;\le\;
2\exp\!\left(-\frac{(p_s^{(a)})^2\,n\,\varepsilon^2}{8{C_\Delta^\ast}^2}\right).
\]
The right-hand side is summable in $n$, and thus by the Borel--Cantelli lemma, we have $N_s^{\ast(a)}/\mu_n^\ast\xrightarrow{\mathrm{a.s.}}1$.

Since both Term~1 and Term~2 converge a.s., their product converges a.s.\ to $p_s^{(a)}$.
This completes the proof.
\end{proof}

%%%%%%%%%%%%%%%%%%%%%%%%%%%%%%%%%%%%%%%%%%%%%%%%%%%%%%%%%%%%
\section{Proof of Theorem~\ref{thm:bootstrap-M}}
\label{app:proofs}
%%%%%%%%%%%%%%%%%%%%%%%%%%%%%%%%%%%%%%%%%%%%%%%%%%%%%%%%%%%%

\begin{proof}
We work on the joint probability space $(\Omega,\mathcal{F},\mathbb{P})$ carrying both the observed data $\mathcal{D}_n$ and the bootstrap randomness; all convergence statements are under $\mathbb{P}$. We define the joint filtration $\mathcal{F}_{n,i}:=\sigma(\mathcal{D}_n,X_0^\ast,A_0^\ast,\ldots,X_i^\ast,A_i^\ast)$. 

We proceed in four steps.
First, we decompose $\sqrt{n}(\hat M_{s,t}^{\ast(a)}-\hat M_{s,t}^{(a)})$ into a leading term and a remainder. 
Second, we show the remainder is negligible. 
Third, we show the leading term converges in distribution to $\mathcal{N}(0,\bar\Lambda_{sat,sat})$ conditionally on $\mathcal{D}_n$. 
This is the univariate, one-dimensional convergence of $\sqrt{n}(\hat M_{s,t}^{\ast(a)}-\hat M_{s,t}^{(a)})$.
Finally, we use the Cram\'er--Wold theorem to extend the univariate convergence to the multivariate convergence of $\sqrt{n}\,\mathrm{vec}(\hat{\mathbf{M}}^\ast-\hat{\mathbf{M}})$ to $\mathcal{N}(0,\bar\Lambda)$ conditionally on $\mathcal{D}_n$.

\medskip\noindent\textbf{Step 1: Decomposition.}
Let $(s,a,t)\in\mathcal{S}\times\mathcal{A}\times\mathcal{S}$ be fixed. We define
\[
\xi_i^\ast := \indicator\{X_i^\ast=s,A_i^\ast=a\}\bigl(\indicator\{X_{i+1}^\ast=t\}-\hat M_{s,t}^{(a)}\bigr).
\]
Since $\hat M_{s,t}^{\ast(a)}=N_{s,t}^{\ast(a)}/N_s^{\ast(a)}$, we have
\begin{align*}
\sqrt{n}\bigl(\hat M_{s,t}^{\ast(a)}-\hat M_{s,t}^{(a)}\bigr)
&= \sqrt{n}\!\left(\frac{N_{s,t}^{\ast(a)}}{N_s^{\ast(a)}}-\hat M_{s,t}^{(a)}\right)
 = \frac{\sqrt{n}}{N_s^{\ast(a)}}\bigl(N_{s,t}^{\ast(a)}-\hat M_{s,t}^{(a)}N_s^{\ast(a)}\bigr)
 = \frac{\sqrt{n}}{N_s^{\ast(a)}}\sum_{i=0}^{n-1}\xi_i^\ast.
\end{align*}
Since $\sqrt{n}/N_s^{\ast(a)} = 1/(\sqrt{n}\,p_s^{(a)}) + (1/\sqrt{n})(n/N_s^{\ast(a)}-1/p_s^{(a)})$, we write
\[
\sqrt{n}\bigl(\hat M_{s,t}^{\ast(a)}-\hat M_{s,t}^{(a)}\bigr) = S_n^{\ast,sat}+R_n^{\ast,sat},
\]
where
\begin{align*}
S_n^{\ast,sat} &:= \frac{1}{\sqrt{n}\,p_s^{(a)}}\sum_{i=0}^{n-1}\xi_i^\ast,\qquad
R_n^{\ast,sat} := \Bigl(\frac{n}{N_s^{\ast(a)}}-\frac{1}{p_s^{(a)}}\Bigr)\frac{1}{\sqrt{n}}\sum_{i=0}^{n-1}\xi_i^\ast.
\end{align*}

\medskip\noindent\textbf{Step 2: Remainder is negligible.}
We define
\[
Z_n^\ast := \frac{1}{\sqrt{n}}\sum_{i=0}^{n-1}\xi_i^\ast,
\]
such that $R_n^{\ast,sat}=\bigl(n/N_s^{\ast(a)}-1/p_s^{(a)}\bigr)\cdot Z_n^\ast$.
Our goal is to show $R_n^{\ast,sat}=o_{\mathbb{P}}(1)$, by first showing $Z_n^\ast=O_{\mathbb{P}}(1)$ and $n/N_s^{\ast(a)}-1/p_s^{(a)}\xrightarrow{\mathrm{a.s.}}0$.

We start by bounding the variance of $Z_n^\ast$.
On the event $\{X_i^\ast=s,A_i^\ast=a\}$, the bootstrap chain reaches next state $t$ with probability $\hat M_{s,t}^{(a)}$, hence
\[
\mathbb{E}[\xi_i^\ast\mid\mathcal{F}_{n,i}]
= \indicator\{X_i^\ast=s,A_i^\ast=a\}\bigl(\hat M_{s,t}^{(a)}-\hat M_{s,t}^{(a)}\bigr)=0.
\]
Therefore $(\xi_i^\ast)$ is a martingale difference sequence under $(\mathcal{F}_{n,i})$ and $\mathbb{P}$. For $i<j$, the tower property gives
\[
\mathbb{E}[\xi_i^\ast\xi_j^\ast]
= \mathbb{E}\bigl[\xi_i^\ast\,\mathbb{E}[\xi_j^\ast\mid\mathcal{F}_{n,j}]\bigr]=0,
\]
thus cross-terms in $\mathrm{Var}[Z_n^\ast]$ vanish. 
The conditional second moment of $\xi_i^\ast$ is
\[
\mathbb{E}\bigl[(\xi_i^\ast)^2\mid\mathcal{F}_{n,i}\bigr]
= \indicator\{X_i^\ast=s,A_i^\ast=a\}\,\hat M_{s,t}^{(a)}\bigl(1-\hat M_{s,t}^{(a)}\bigr) \le \frac{1}{4},
\]
since $\indicator\{X_{i+1}^\ast=t\}$ given $\mathcal{F}_{n,i}$ on $\{X_i^\ast=s,A_i^\ast=a\}$ is Bernoulli$(\hat M_{s,t}^{(a)})$.
Summing over $i$, we get
\begin{align*}
\mathrm{Var}[Z_n^\ast]
&= \frac{1}{n}\sum_{i=0}^{n-1}\mathbb{E}\bigl[(\xi_i^\ast)^2\bigr] \\
&= \frac{1}{n}\sum_{i=0}^{n-1}\mathbb{E}\left[\mathbb{E}\bigl[(\xi_i^\ast)^2 \mid \mathcal{F}_{n,i} \bigr] \right] \\
&\le \; \frac{1}{n} \sum_{i=0}^{n-1} \frac{1}{4} \\
% &= \frac{\hat M_{s,t}^{(a)}(1-\hat M_{s,t}^{(a)})}{n}\sum_{i=0}^{n-1}\mathbb{P}(X_i^\ast=s,A_i^\ast=a)\\
% &= \frac{\mathbb{E}[N_s^{\ast(a)}]}{n}\,\hat M_{s,t}^{(a)}\bigl(1-\hat M_{s,t}^{(a)}\bigr) \\
&=\frac{1}{4}.
\end{align*}

For any $M>0$, by the Chebyshev's inequality, we have
\[
\mathbb{P}\bigl(|Z_n^\ast|>M\bigr)
\le \frac{\mathrm{Var}[Z_n^\ast]}{M^2}
\le \frac{1}{4M^2},
\]
hence $Z_n^\ast=O_{\mathbb{P}}(1)$. Since $N_s^{\ast(a)}/n\xrightarrow{\mathrm{a.s.}}p_s^{(a)}>0$, the continuous mapping theorem gives $n/N_s^{\ast(a)}-1/p_s^{(a)}\xrightarrow{\mathrm{a.s.}}0$.

To show $R_n^{\ast,sat}=o_{\mathbb{P}}(1)$, we consider any $\varepsilon,\delta>0$ and choose $M$ large enough that $1/(4M^2)<\delta/2$. 
We define the event $\mathcal{E}_n := \{|n/N_s^{\ast(a)}-1/p_s^{(a)}|>\varepsilon/M\}\cup\{|Z_n^\ast|>M\}$. 
On the event $\mathcal{E}_n^c$, both $|n/N_s^{\ast(a)}-1/p_s^{(a)}|\le\varepsilon/M$ and $|Z_n^\ast|\le M$ hold. Thus, we have $|R_n^{\ast,sat}|=|n/N_s^{\ast(a)}-1/p_s^{(a)}|\cdot|Z_n^\ast|\le(\varepsilon/M)\cdot M=\varepsilon$. 
Therefore,
\begin{align*}
\mathbb{P}\bigl(|R_n^{\ast,sat}|>\varepsilon\bigr)
& = \mathbb{P}\bigl(|R_n^{\ast,sat}|>\varepsilon,\,\mathcal{E}_n^c\bigr) + \mathbb{P}\bigl(|R_n^{\ast,sat}|>\varepsilon,\,\mathcal{E}_n\bigr)\\
& = 0 + \mathbb{P}\bigl(|R_n^{\ast,sat}|>\varepsilon,\,\mathcal{E}_n\bigr) \\
&\le\; \mathbb{P}\bigl(|Z_n^\ast|>M\bigr)\\
&\le\; \frac{1}{4M^2} \\
&<\; \frac{\delta}{2}.
\end{align*}
Since $\delta$ is arbitrary, we have $R_n^{\ast,sat}=o_{\mathbb{P}}(1)$.
It remains to show $S_n^{\ast,sat}\xrightarrow{\mathrm{d}}\mathcal{N}(0,\bar\Lambda_{sat,sat})$ under $\mathbb{P}$.

\medskip\noindent\textbf{Step 3: One Dimensional Martingale CLT.}
We define
\[
d_{n,i+1}^{\ast,sat} := \frac{\xi_i^\ast}{\sqrt{n}\,p_s^{(a)}},
\]
so that $S_n^{\ast,sat}=\sum_{i=0}^{n-1}d_{n,i+1}^{\ast,sat}$. Since $\mathbb{E}[\xi_i^\ast \mid \mathcal{F}_{n,i}]=0$, we have $\mathbb{E}[d_{n,i+1}^{\ast,sat}\mid\mathcal{F}_{n,i}]=0$, and $(d_{n,i+1}^{\ast,sat})$ is a martingale difference array under $\mathbb{P}$ with respect to the joint filtration $(\mathcal{F}_{n,i})$.

To prove the CLT for $S_n^{\ast,sat}$, we verify the Lindeberg condition.
Since $|\xi_i^\ast|\le 1$, we have $|d_{n,i+1}^{\ast,sat}|\le 1/(\sqrt{n}\,p_s^{(a)})$. For any $\varepsilon>0$, we set $n_0:=\lceil(\varepsilon\, p_s^{(a)})^{-2}\rceil$, where $\lceil x \rceil$ denotes the smallest integer greater than or equal to $x$.
For all $n\ge n_0$, we have $|d_{n,i+1}^{\ast,sat}|\le\varepsilon$ and thus $\{|d_{n,i+1}^{\ast,sat}|>\varepsilon\}=\emptyset$ and
\[
\sum_{i=0}^{n-1}\mathbb{E}\Bigl[(d_{n,i+1}^{\ast,sat})^2\,\indicator\bigl\{|d_{n,i+1}^{\ast,sat}|>\varepsilon\bigr\}\;\Big|\;\mathcal{F}_{n,i}\Bigr]=0.
\]
Hence, the Lindeberg condition holds for all $n\ge n_0$.

We now derive the asymptotic variance of $S_n^{\ast,sat}$ by analyzing the predictable quadratic variation $\sum_{i=0}^{n-1}\mathbb{E}[(d_{n,i+1}^{\ast,sat})^2\mid\mathcal{F}_{n,i}]$.
Using 
\[
\mathbb{E}[(\xi_i^\ast)^2\mid\mathcal{F}_{n,i}]=\indicator\{X_i^\ast=s,A_i^\ast=a\}\hat M_{s,t}^{(a)}(1-\hat M_{s,t}^{(a)}),
\]
we have
\[
\mathbb{E}\bigl[(d_{n,i+1}^{\ast,sat})^2\mid\mathcal{F}_{n,i}\bigr]
= \frac{\indicator\{X_i^\ast=s,A_i^\ast=a\}}{n\,(p_s^{(a)})^2}\,\hat M_{s,t}^{(a)}\bigl(1-\hat M_{s,t}^{(a)}\bigr).
\]
Summing over $i=0,\ldots,n-1$, we get
\begin{align*}
\sum_{i=0}^{n-1}\mathbb{E}\bigl[(d_{n,i+1}^{\ast,sat})^2\mid\mathcal{F}_{n,i}\bigr]
=\frac{N_s^{\ast(a)}}{n\,(p_s^{(a)})^2}\,\hat M_{s,t}^{(a)}\bigl(1-\hat M_{s,t}^{(a)}\bigr).
\numberthis\label{eq:pred-quad-var}
\end{align*}
By Lemma~\ref{lem:bootstrap-lln}, $N_s^{\ast(a)}/n\xrightarrow{\mathrm{a.s.}}p_s^{(a)}$, and by \citet[Lemma~2]{su2025centrallimittheoremstransition}, $\hat M_{s,t}^{(a)}\xrightarrow{\mathrm{a.s.}} M_{s,t}^{(a)}$.
Therefore, the right-hand side of \eqref{eq:pred-quad-var} converges a.s.\ to
\[
\frac{M_{s,t}^{(a)}\bigl(1-M_{s,t}^{(a)}\bigr)}{p_s^{(a)}}=\bar\Lambda_{sat,sat}.
\]
By the martingale CLT \citep[Corollary~3.1]{hall1980martingale}, applied conditionally on $\mathcal{D}_n$ for $\mathbb{P}$-a.e.\ $\mathcal{D}_n$, we obtain that, conditional on $\mathcal{D}_n$, $S_n^{\ast,sat}\xrightarrow{\mathrm{d}}\mathcal{N}(0,\bar\Lambda_{sat,sat})$ for $\mathbb{P}$-almost every $(\mathcal{D}_n)$. Combined with $R_n^{\ast,sat}=o_{\mathbb{P}}(1)$, we conclude that, conditional on $\mathcal{D}_n$, $\sqrt{n}(\hat M_{s,t}^{\ast(a)}-\hat M_{s,t}^{(a)})\xrightarrow{\mathrm{d}}\mathcal{N}(0,\bar\Lambda_{sat,sat})$ for $\mathbb{P}$-almost every $(\mathcal{D}_n)$.

\medskip\noindent\textbf{Step 4: Multivariate CLT via Cram\'{e}r--Wold.}
Step~3 establishes the marginal CLT for each scalar coordinate $\sqrt{n}(\hat M_{s,t}^{\ast(a)}-\hat M_{s,t}^{(a)})$. We now extend this to the joint distribution over all coordinates simultaneously. 

Let $S_n^\ast:=(S_n^{\ast,sat})_{(s,a,t)}$ denote the stacked vector, such that $\sqrt{n}\,\mathrm{vec}(\hat{\mathbf{M}}^\ast-\hat{\mathbf{M}})=S_n^\ast+o_{\mathbb{P}}(1)$, using $R_n^{\ast,sat}=o_{\mathbb{P}}(1)$ from Step~2. By the Cram\'{e}r--Wold theorem, a random vector converges in distribution to a multivariate normal if and only if every fixed linear combination of its coordinates does. It therefore suffices to show that for every fixed $u\in\mathbb{R}^{S^2A}$,
\[
u^\top S_n^\ast \;\xrightarrow{\mathrm{d}}\; \mathcal{N}(0,u^\top\bar\Lambda\, u), \quad \text{conditional on } \mathcal{D}_n, \text{ for } \mathbb{P}\text{-a.e.\ } (\mathcal{D}_n).
\]

We fix such $u$ and let $p_0:=\min_{s,a}p_s^{(a)}>0$. We define the projected increments
\[
\tilde d_{n,i+1}^\ast := \sum_{s,a,t}u_{sat}\,d_{n,i+1}^{\ast,sat}
\]
and the projected partial sum $\tilde S_n^\ast := \sum_{i=0}^{n-1}\tilde d_{n,i+1}^\ast$, so that $u^\top S_n^\ast = \tilde S_n^\ast$. Since each $d_{n,i+1}^{\ast,sat}$ has zero $\mathbb{P}$-conditional mean given $\mathcal{F}_{n,i}$, the same holds for $\tilde d_{n,i+1}^\ast$. Thus, $(\tilde d_{n,i+1}^\ast)$ is a martingale difference array under $\mathbb{P}$ with respect to $(\mathcal{F}_{n,i})$. We now verify the Lindeberg condition. Since $|d_{n,i+1}^{\ast,sat}|\le 1/(\sqrt{n}\,p_s^{(a)})\le 1/(\sqrt{n}\,p_0)$, the triangle inequality gives $|\tilde d_{n,i+1}^\ast|\le \|u\|_1/(\sqrt{n}\,p_0)$. For any $\varepsilon>0$, we set $n_0:=\lceil\|u\|_1^2/(\varepsilon p_0)^2\rceil$. For all $n\ge n_0$, we have $|\tilde d_{n,i+1}^\ast|\le\varepsilon$, and thus
\[
\sum_{i=0}^{n-1}\mathbb{E}\Bigl[(\tilde d_{n,i+1}^\ast)^2\,\indicator\bigl\{|\tilde d_{n,i+1}^\ast|>\varepsilon\bigr\}\;\Big|\;\mathcal{F}_{n,i}\Bigr]=0.
\]
Hence, the Lindeberg condition holds for all $n\ge n_0$.

We next compute the predictable quadratic variation. We expand $(\tilde d_{n,i+1}^\ast)^2 = \bigl(\sum_{s,a,t}u_{sat}\,d_{n,i+1}^{\ast,sat}\bigr)^2$ to get
\[
\mathbb{E}\bigl[(\tilde d_{n,i+1}^\ast)^2\mid\mathcal{F}_{n,i}\bigr]
= \sum_{s,a,t}\sum_{s',a',t'}u_{sat}\,u_{s'a't'}\,\mathbb{E}\bigl[d_{n,i+1}^{\ast,sat}\,d_{n,i+1}^{\ast,s'a't'}\mid\mathcal{F}_{n,i}\bigr].
\]
We observe that when $(s,a)\ne(s',a')$, we have
\[
\indicator\{X_i^\ast=s,A_i^\ast=a\}\cdot\indicator\{X_i^\ast=s',A_i^\ast=a'\}=0,
\]
since the chain occupies a single state--action pair at each step. Hence $d_{n,i+1}^{\ast,sat}\cdot d_{n,i+1}^{\ast,s'a't'}=0$ and those cross-row terms vanish. For the diagonal terms $(s,a)=(s',a')$, we condition on $\mathcal{F}_{n,i}$ and use that $(\indicator\{X_{i+1}^\ast=t\})_{t\in\mathcal{S}}$ given $(X_i^\ast=s,A_i^\ast=a)$ follows a multinomial distribution with $1$ trial and event probabilities $(\hat M_{s,1}^{(a)},\ldots,\hat M_{s,S}^{(a)})$.
The second moment formula of the multinomial distribution gives
\begin{align*}
& \mathbb{E}\bigl[\xi_i^{\ast,sat}\xi_i^{\ast,sat'}\mid\mathcal{F}_{n,i}\bigr] \\
= \, & \indicator\{X_i^\ast=s,A_i^\ast=a\} \times\mathbb{E}\bigl[(\indicator\{X_{i+1}^\ast=t\}-\hat M_{s,t}^{(a)})(\indicator\{X_{i+1}^\ast=t'\}-\hat M_{s,t'}^{(a)}) \\
& \qquad\mid X_i^\ast=s,A_i^\ast=a,\mathcal{D}_n\bigr]\\
= \, & \indicator\{X_i^\ast=s,A_i^\ast=a\}\,\hat M_{s,t}^{(a)}\bigl(\indicator\{t=t'\}-\hat M_{s,t'}^{(a)}\bigr).
\end{align*}
Therefore
\begin{align*}
\sum_{i=0}^{n-1}\mathbb{E}\bigl[(\tilde d_{n,i+1}^\ast)^2\mid\mathcal{F}_{n,i}\bigr]
&=\sum_{s,a}\frac{N_s^{\ast(a)}}{n\,(p_s^{(a)})^2}
\sum_{t,t'}u_{sat}\,u_{sat'} \\
&\qquad \times \hat M_{s,t}^{(a)}
\bigl(\indicator\{t=t'\}-\hat M_{s,t'}^{(a)}\bigr).
\end{align*}
Since $N_s^{\ast(a)}/n\xrightarrow{\mathrm{a.s.}}p_s^{(a)}$ and $\hat M_{s,t}^{(a)}\to M_{s,t}^{(a)}$ a.s., the quadratic variation converges a.s.\ to
\[
\sum_{s,a}\frac{1}{p_s^{(a)}}
\sum_{t,t'}u_{sat}\,u_{sat'}\bigl(M_{s,t}^{(a)}\indicator\{t=t'\}-M_{s,t}^{(a)}M_{s,t'}^{(a)}\bigr)
= u^\top\bar\Lambda\, u.
\]
Hence, conditional on $\mathcal{D}_n$, $\tilde S_n^\ast\xrightarrow{\mathrm{d}}\mathcal{N}(0,u^\top\bar\Lambda u)$ for $\mathbb{P}$-almost every $(\mathcal{D}_n)$. Since $u$ is arbitrary, the Cram\'{e}r--Wold theorem gives that, for $\mathbb{P}$-almost every sequence of datasets $(\mathcal{D}_n)_{n\ge 1}$, conditional on $\mathcal{D}_n$, yields
\[
\sqrt{n}\,\mathrm{vec}(\hat{\mathbf{M}}^\ast-\hat{\mathbf{M}})
\;\xrightarrow{\mathrm{d}}\;\mathcal{N}(0,\bar\Lambda).
\]
This completes the proof.
\end{proof}

%%%%%%%%%%%%%%%%%%%%%%%%%%%%%%%%%%%%%%%%%%%%%%%%%%%%%%%%%%%%
\section{Proofs of Propositions}
\label{app:proof-corollaries}
%%%%%%%%%%%%%%%%%%%%%%%%%%%%%%%%%%%%%%%%%%%%%%%%%%%%%%%%%%%%

\subsection{Proof of Proposition~\ref{prop:bootstrap-M-episodic}}
\label{app:proof-cor-episodic}

\begin{proof}
We use the notation introduced in Section~\ref{sec:bootstrap-consistency}.

We recall that the concatenated single-chain dataset is $\tilde{\mathcal{D}} := \{(\widetilde{X}_i,\widetilde{A}_i,\widetilde{X}_{i+1})\}_{i=0}^{n'-1}$ with $n'=K(T+1)-1$.
The concatenated dataset is defined by setting $\widetilde{X}_{(k-1)(T+1)+j}:=X_j^{(k)}$ and $\widetilde{A}_{(k-1)(T+1)+j}:=A_j^{(k)}$ for $j=0,\ldots,T-1$, and inserting at each episode boundary $k=1,\ldots,K-1$ a pseudo-transition $\widetilde{A}_{k(T+1)-1}:=a_\dagger$ mapping $X_T^{(k)}$ to $X_0^{(k+1)}$.
The extended action space is $\widetilde{\mathcal{A}}=\mathcal{A}\cup\{a_\dagger\}$, and the resulting process $\{(\widetilde{X}_i,\widetilde{A}_i)\}$ is a CMC on $\mathcal{S}\times\widetilde{\mathcal{A}}$ with transition kernel $\widetilde{M}$, where $\widetilde{M}^{(a)}_{s,t}=M^{(a)}_{s,t}$ for $a\in\mathcal{A}$ and $\widetilde{M}^{(a_\dagger)}_{s,t}>0$ for all $s,t\in\mathcal{S}$.

By hypotheses, $\{(\widetilde{X}_i,\widetilde{A}_i)\}$ satisfies Assumptions~\ref{ass:return-time-growth}--\ref{ass:semi-ergodic} on $\mathcal{S}\times\widetilde{\mathcal{A}}$, and there exists $(s_0,a_0)\in\mathcal{S}\times\mathcal{A}$ with $M_{s_0,s_0}^{(a_0)}>0$.
Applying Theorem~\ref{thm:bootstrap-M} to $\tilde{\mathcal{D}}$ gives that, for $\mathbb{P}$-almost every $(\mathcal{D}_n)_{n\ge1}$, conditional on $\mathcal{D}_n$, we have
\[
\sqrt{n'}\,\mathrm{vec}(\hat{\widetilde{\mathbf{M}}}^\ast - \hat{\widetilde{\mathbf{M}}})
\;\xrightarrow{\mathrm{d}}\;\mathcal{N}(0,\bar{\widetilde\Lambda}), \quad \text{as }K\to\infty,
\]
where $\bar{\widetilde\Lambda}$ is the asymptotic covariance matrix.
Among the $T+1$ steps per episode, exactly $T$ involve actions in $\mathcal{A}$. Thus the ergodic occupation measure of the concatenated chain satisfies 
\[
\tilde p_s^{(a)}=\tfrac{T}{T+1}\,p_s^{(a)} \text{ for $a\in\mathcal{A}$, giving }\bar{\widetilde\Lambda}_{sat,sat}=\tfrac{T+1}{T}\,\bar\Lambda_{sat,sat} \text{ for all $(s,a,t) \in \mathcal S \times \mathcal A \times \mathcal S$.}
\]

Since $a_\dagger$-transitions are excluded from $\hat{\mathbf{M}}$ (the empirical estimator uses only $a\in\mathcal{A}$), we take the marginals of the Gaussian limit to get 
\[
\sqrt{n'}\,\mathrm{vec}(\hat{\mathbf{M}}^\ast-\hat{\mathbf{M}})\xrightarrow{\mathrm{d}}\mathcal{N}(0,\bar{\widetilde\Lambda}|_{\mathcal{S}\times\mathcal{A}\times\mathcal{S}}),
\] 
where $\bar{\widetilde\Lambda}|_{\mathcal{S}\times\mathcal{A}\times\mathcal{S}}$ denotes the restriction of $\bar{\widetilde\Lambda}$ to the entries corresponding to $(s,a,t)\in\mathcal{S}\times\mathcal{A}\times\mathcal{S}$.

Among the $T+1$ steps per episode, exactly $T$ involve actions in $\mathcal{A}$. Thus the ergodic occupation measure of the concatenated chain satisfies $\tilde p_s^{(a)}=\tfrac{T}{T+1}\,p_s^{(a)}$ for $a\in\mathcal{A}$, giving $\bar{\widetilde\Lambda}_{sat,sat}=\tfrac{T+1}{T}\,\bar\Lambda_{sat,sat}$.

Since $n'=K(T+1)-1$ and $n=K T$, we have 
\[
\frac{n}{n'} = \frac{KT}{K(T+1)-1} = \frac{T}{T+1-1/K} \to \frac{T}{T+1}, \quad \text{as } K\to\infty.
\]
Therefore, re-normalizing $\mathrm{vec}(\hat{\mathbf{M}}^\ast-\hat{\mathbf{M}})$ by $\sqrt{n}$ yields
\[
\sqrt{n}\,\mathrm{vec}(\hat{\mathbf{M}}^\ast-\hat{\mathbf{M}}) = \sqrt{\frac{n}{n'}}\,\sqrt{n'}\,\mathrm{vec}(\hat{\mathbf{M}}^\ast-\hat{\mathbf{M}})
\;\xrightarrow{\mathrm{d}}\;
\mathcal{N}\!\Bigl(0,\tfrac{T}{T+1}\,\bar{\widetilde\Lambda}|_{\mathcal{S}\times\mathcal{A}\times\mathcal{S}}\Bigr)
= \mathcal{N}(0,\bar\Lambda). \qedhere
\]
\end{proof}

\subsection{Proof of Proposition~\ref{prop:bootstrap-vq}}
\label{app:proof-cor-vq}
The proof applies the bootstrap delta method to Theorem~\ref{thm:bootstrap-M} (single-chain) or Proposition~\ref{prop:bootstrap-M-episodic} (episodic).
Before we provide the formal proof, we recall Hadamard differentiability \citep[Chapter 20.2]{van_der_vaart_asymptotic_2000}.

A map $\phi:\mathbb{D}_\phi \to\mathbb{E}$, defined on a subset $\mathbb{D}_\phi$ of a normed space $\mathbb{D}$ and taking values in another normed space $\mathbb{E}$, is Hadamard differentiable at $\theta\in\mathbb{D}_\phi$ tangentially to $\mathbb{D}_0\subseteq\mathbb{D}$ if there exists a continuous linear map $\phi'_\theta:\mathbb{D}_0\to\mathbb{E}$ such that 
\[
\left\|\frac{\phi(\theta+t h_n)-\phi(\theta)}{t} - \phi'_\theta(h)\right\|_\mathbb{E}\to 0 \, \text{ as $t \to 0$, for every $h_n\to h\in\mathbb{D}_0$}.
\]
The bootstrap delta method \citep[Theorem 23.9]{van_der_vaart_asymptotic_2000} then states that if $\sqrt{n}(\hat\theta^\ast-\hat\theta)\xrightarrow{\mathrm{d}} L$ conditionally on $\mathcal{D}_n$ in probability, and $\phi$ is Hadamard differentiable at $\theta$ tangentially to a subspace $\mathbb{D}_0$, then $\sqrt{n}(\phi(\hat\theta^\ast)-\phi(\hat\theta))\xrightarrow{\mathrm{d}}\phi'_\theta(L)$ conditionally on $\mathcal{D}_n$ in probability.

The proof of Proposition~\ref{prop:bootstrap-vq} is as follows.
\begin{proof}
\noindent\textbf{OPE targets $V_\pi$ and $Q_\pi$.}
We first derive the bootstrap consistency results for the OPE targets $V_\pi$.
The bootstrap consistency for $Q_\pi$ is entirely analogous.

We consider any fixed stationary target policy $\pi$ with block-diagonal policy matrix $\Pi$.
The value function is given by the Bellman equation
\[
V_\pi(\mathbf{M}) = (I-\gamma\Pi\mathbf{M})^{-1}g.
\]
We show that the map $V_\pi: \mathbb{R}^{S^2A} \to \mathbb{R}^S$ is Hadamard differentiable at $\mathbf{M}$.
Let $H$ be any perturbation matrix and $t>0$ sufficiently small. Using the matrix property $A^{-1}-B^{-1}=A^{-1}(B-A)B^{-1}$, we have
\begin{align*}
V_\pi(\mathbf{M}+tH)-V_\pi(\mathbf{M})
&= \bigl[(I-\gamma\Pi(\mathbf{M}+tH))^{-1}-(I-\gamma\Pi\mathbf{M})^{-1}\bigr]g \\
&= (I-\gamma\Pi(\mathbf{M}+tH))^{-1}(\gamma\Pi tH)(I-\gamma\Pi\mathbf{M})^{-1}g \\
&= t\gamma(I-\gamma\Pi(\mathbf{M}+tH))^{-1}\Pi H\,V_\pi(\mathbf{M}).
\end{align*}
Dividing by $t$ and taking $t\to 0$, we get
\[
\frac{V_\pi(\mathbf{M}+tH)-V_\pi(\mathbf{M})}{t} \;\to\; \gamma(I-\gamma\Pi\mathbf{M})^{-1}\Pi H\,V_\pi(\mathbf{M}).
\]
We define the linear map $V'_\mathbf{M}:\mathbb{R}^{S^2A}\to\mathbb{R}^S$ by $V'_\mathbf{M}(H) := \gamma(I-\gamma\Pi\mathbf{M})^{-1}\Pi H\,V_\pi(\mathbf{M})$.
The derivative $V'_\mathbf{M}$ is linear and continuous in $H$, hence for any $H_n\to H$, we have
\[
\left\|\frac{V_\pi(\mathbf{M}+tH_n)-V_\pi(\mathbf{M})}{t} - V'_\mathbf{M}(H) \right\|_{\mathbb{E}} \to 0.
\]
Therefore, $\mathbf{M}\mapsto V_\pi(\mathbf{M})$ is Hadamard differentiable at $\mathbf{M}$.
Since Theorem~\ref{thm:bootstrap-M} and Proposition~\ref{prop:bootstrap-M-episodic} establish that
$\sqrt{n}\,\mathrm{vec}(\hat{\mathbf{M}}^\ast-\hat{\mathbf{M}})\xrightarrow{\mathrm{d}}\mathcal{N}(0,\bar\Lambda)$
conditionally on $\mathcal{D}_n$---the same Gaussian law as $\sqrt{n}\,\mathrm{vec}(\hat{\mathbf{M}}-\mathbf{M})$
in \citet[Corollary~1]{su2025centrallimittheoremstransition}---the bootstrap delta method with derivative $V'_\mathbf{M}$ gives,
conditionally on $\mathcal{D}_n$ in probability,
\[
\sqrt{n}(\hat V_\pi^\ast - \hat V_\pi) \;\xrightarrow{\mathrm{d}}\; \mathcal{N}(0,\Sigma_V^\pi),
\]
where $\Sigma_V^\pi$ is the asymptotic covariance of $\sqrt{n}(\hat V_\pi - V_\pi)$ in
\citet[Theorem~2]{su2025centrallimittheoremstransition}.
The statement for $Q_\pi$ is analogous with $V'_\mathbf{M}$ replaced by the analogous $Q'_\mathbf{M}$ and $\Sigma_V^\pi$ replaced by $\Sigma_Q^\pi$.

\medskip\noindent\textbf{OPR targets $V_\star$ and $\Q_\star$.}
We next derive the bootstrap consistency result for the OPR target $V_\star$.
The bootstrap consistency for $Q_\star$ is entirely analogous.

The optimal-value functional $\mathbf{M}\mapsto V_\star(\mathbf{M})$ is not directly Hadamard differentiable, because the policy-selection map $\mathbf{M}\mapsto\pi_\star(\mathbf{M})$ involves an argmax.
We resolve this via the non-degeneracy condition
\[
\min_{s\in\mathcal{S}} \min_{a \neq \pi_\star(s)}\bigl\{Q_\star(s, \pi_\star(s))-Q_\star(s,a)\bigr\}>0.
\]
We claim that this condition implies that $\pi_\star$ is locally constant in a neighborhood of $\mathbf{M}$.
We denote the optimal $Q$-function under $\mathbf{M'}$ by $Q_\star(\mathbf{M'})$.

By the non-degeneracy condition, we define
\[
\Delta
:=
\min_{s\in\mathcal S}
\min_{a\neq \pi_\star(s)}
\{Q_\star(\mathbf{M};s,\pi_\star(\mathbf{M}; s))-Q_\star(\mathbf{M};s,a)\}
>0.
\]
Since the finite discounted optimal Bellman equation has a unique fixed point and the
Bellman optimality operator is continuous in \(\mathbf{M'}\), the map
\(\mathbf{M'}\mapsto Q_\star(\mathbf{M'})\) is continuous. 
Hence there exists a neighborhood
\(\mathcal U\) of \(\mathbf{M}\) such that, for every \(\mathbf{M'}\in\mathcal U\), we have
\[
\max_{s\in\mathcal S}\max_{a\in\mathcal A}
|Q_\star(\mathbf{M'};s,a)-Q_\star(\mathbf{M};s,a)|
< \Delta/4 .
\]
Therefore, for every \(s\in\mathcal S\) and every \(a\neq \pi_\star(\mathbf{M};s)\), we have
\[
\begin{aligned}
Q_\star(\mathbf{M'};s,\pi_\star(\mathbf{M};s))-Q_\star(\mathbf{M'};s,a)
&= Q_\star(\mathbf{M};s,\pi_\star(\mathbf{M};s)) - Q_\star(\mathbf{M};s,a) \\
&\quad+ \left[Q_\star(\mathbf{M'};s,\pi_\star(\mathbf{M};s)) - Q_\star(\mathbf{M};s,\pi_\star(\mathbf{M};s))\right] \\
& \quad \quad - \left[Q_\star(\mathbf{M'};s,a)-Q_\star(\mathbf{M};s,a)\right]
\\
&\ge
Q_\star(\mathbf{M};s,\pi_\star(\mathbf{M};s))-Q_\star(\mathbf{M};s,a)
\\
&\quad
-2\max_{s,a}|Q_\star(\mathbf{M'};s,a)-Q_\star(\mathbf{M};s,a)|
\\
&> \Delta-\Delta/2
= \Delta/2>0.
\end{aligned}
\]
Thus \(\pi_\star(\mathbf{M}) = \pi_\star(\mathbf{M'})\) for all $\mathbf{M'} \in \mathcal U$. 

By \citet[Lemma 1]{su2025centrallimittheoremstransition}, we have $\hat{\mathbf{M}}\xrightarrow{\mathrm{a.s.}}\mathbf{M}$. 
Then $\hat{\mathbf{M}}\in\mathcal{U}$ and $\hat\Pi_\star=\Pi_\star$ for large $n$ a.s.
On this probability one event, the OPR functional locally coincides with the fixed-policy map $\mathbf{M}'\mapsto V_\star(\mathbf{M}')$,
which is Hadamard differentiable by the aforementioned OPE argument.
We can therefore apply the bootstrap delta method with the same derivative $V'_\mathbf{M}$ as in the OPE case.

Since Theorem~\ref{thm:bootstrap-M} and Proposition~\ref{prop:bootstrap-M-episodic} establish that
$\sqrt{n}\,\mathrm{vec}(\hat{\mathbf{M}}^\ast-\hat{\mathbf{M}})\xrightarrow{\mathrm{d}}\mathcal{N}(0,\bar\Lambda)$ conditionally on $\mathcal{D}_n$---the same Gaussian law
as \citet[Corollary~1]{su2025centrallimittheoremstransition}---applying the bootstrap delta method yields,
conditional on $\mathcal{D}_n$ in probability, we obtain
\[
\sqrt{n}\bigl(\hat V_\star^\ast - \hat V_\star\bigr) \;\xrightarrow{\mathrm{d}}\;
\mathcal{N}(0,\Sigma_V^{\pi_\star}), \quad\text{as } n\to\infty,
\]
where $\Sigma_V^{\pi_\star}$ is the asymptotic covariance of $\sqrt{n}(\hat V_\star - V_\star)$ in
\citet[Theorem~3]{su2025centrallimittheoremstransition}. 
The statement for $Q_\star$ is analogous with $V'_\mathbf{M}$ replaced by the analogous $Q'_\mathbf{M}$ and $\Sigma_V^{\pi_\star}$ replaced by $\Sigma_Q^{\pi_\star}$.
\qedhere
\end{proof}

%%%%%%%%%%%%%%%%%%%%%%%%%%%%%%%%%%%%%%%%%%%%%%%%%%%%%%%%%%%%
\clearpage
\section{Details of Numerical Experiments}
\label{app:tables}
%%%%%%%%%%%%%%%%%%%%%%%%%%%%%%%%%%%%%%%%%%%%%%%%%%%%%%%%%%%%

We evaluate finite-sample CI coverage of the proposed bootstrap on the RiverSwim problem \citep{strehl2004empirical}, replicating the experimental setup of the small state-space regime of \citet{zhu2024uncertainty}.
Our RiverSwim state--action space has $S=6$ states $\{1,\ldots,6\}$ and $A=2$ actions ($0$: swim left, $1$: swim right).
We set the rewards as $r(1,0)=1$, $r(6,1)=10$, $r(s,a)=0$ otherwise, and discount factor $\gamma=0.95$.
Figure~\ref{fig:riverswim} illustrates the transition and reward structure of our RiverSwim setup.
RiverSwim is designed to be difficult for exploration: the transition structure induces highly imbalanced visitation frequencies across state--action pairs, with upstream states (near state~$6$) visited far less frequently than downstream states (near state~$1$).
Moreover, at its left-most state~$1$, the optimal policy is nearly degenerate, with $Q_\star(1,0)\approx Q_\star(1,1)$, thus the optimal action is only weakly identifiable.
This places the problem near the boundary of optimal policy degeneracy and the transition probabilities for rarely visited state--action pairs are weakly estimated, thus the conditions for asymptotic normality in Proposition~\ref{prop:bootstrap-vq} are most strained.

\begin{figure}[t]
\centering
\resizebox{\linewidth}{!}{%
\begin{tikzpicture}[
    >=Stealth,
    shorten >=1pt, shorten <=1pt,
    state/.style={circle, draw=black, thick, minimum size=0.7cm, font=\small},
    lbl/.style={font=\scriptsize, inner sep=1pt, fill=white},
    every loop/.style={looseness=5, min distance=5mm},
]
  \node[state] (s1) at (0,0)    {$1$};
  \node[state] (s2) at (2.4,0)  {$2$};
  \node[state] (s3) at (4.8,0)  {$3$};
  \node[state] (s4) at (7.2,0)  {$4$};
  \node[state] (s5) at (9.6,0)  {$5$};
  \node[state] (s6) at (12.0,0) {$6$};

  % State 1: left self-loop (action 0, prob 1, reward 1)
  \draw[->] (s1) edge[loop left]  node[lbl] {$(0,1,1)$} (s1);
  % State 1: right self-loop (action 1, stay 0.7, reward 0)
  \draw[->] (s1) edge[loop above] node[lbl,above] {$(1,0.7,0)$} (s1);

  % States 2--5: right self-loop (stay 0.6)
  \draw[->] (s2) edge[loop above] node[lbl,above] {$(1,0.6,0)$} (s2);
  \draw[->] (s3) edge[loop above] node[lbl,above] {$(1,0.6,0)$} (s3);
  \draw[->] (s4) edge[loop above] node[lbl,above] {$(1,0.6,0)$} (s4);
  \draw[->] (s5) edge[loop above] node[lbl,above] {$(1,0.6,0)$} (s5);

  % State 6: right self-loop (stay 0.3, reward 10)
  \draw[->] (s6) edge[loop above] node[lbl,above] {$(1,0.3,10)$} (s6);

  % Forward arcs: right action, advance to next state (prob 0.3)
  % Use out/in to keep arcs clearly below the self-loops
  \draw[->] (s1) to[out=20,in=160] node[lbl,above] {$(1,0.3,0)$} (s2);
  \draw[->] (s2) to[out=20,in=160] node[lbl,above] {$(1,0.3,0)$} (s3);
  \draw[->] (s3) to[out=20,in=160] node[lbl,above] {$(1,0.3,0)$} (s4);
  \draw[->] (s4) to[out=20,in=160] node[lbl,above] {$(1,0.3,0)$} (s5);
  \draw[->] (s5) to[out=20,in=160] node[lbl,above] {$(1,0.3,0)$} (s6);

  % Backward arcs: left action (prob 1) -- shallow curve below
  % For a rightward-to-leftward arc, bend left curves downward
  \draw[->] (s2) to[bend left=12] node[lbl,below] {$(0,1,0)$} (s1);
  \draw[->] (s3) to[bend left=12] node[lbl,below] {$(0,1,0)$} (s2);
  \draw[->] (s4) to[bend left=12] node[lbl,below] {$(0,1,0)$} (s3);
  \draw[->] (s5) to[bend left=12] node[lbl,below] {$(0,1,0)$} (s4);
  \draw[->] (s6) to[bend left=12] node[lbl,below] {$(0,1,0)$} (s5);

  % Backward slip arcs: right action, drift left -- deeper curve below
  \draw[->] (s2) to[bend left=35] node[lbl,below] {$(1,0.1,0)$} (s1);
  \draw[->] (s3) to[bend left=35] node[lbl,below] {$(1,0.1,0)$} (s2);
  \draw[->] (s4) to[bend left=35] node[lbl,below] {$(1,0.1,0)$} (s3);
  \draw[->] (s5) to[bend left=35] node[lbl,below] {$(1,0.1,0)$} (s4);
  \draw[->] (s6) to[bend left=35] node[lbl,below] {$(1,0.7,10)$} (s5);

\end{tikzpicture}}
\caption{The RiverSwim MDP with $S=6$ states. Arc labels $(a,p,r)$: action ($0$\,=\,swim left, $1$\,=\,swim right), transition probability $p$, immediate reward $r$.}
\label{fig:riverswim}
\end{figure}
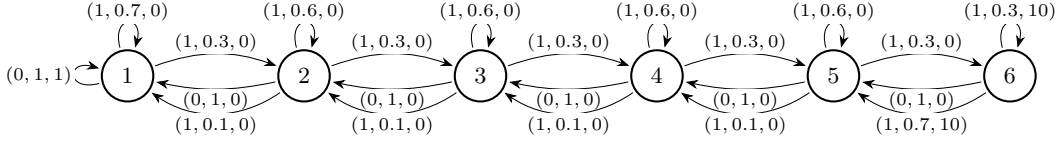

In all experiments, we employ a stationary behavior policy $\pi_b(1\mid s)=0.8$ for all $s$. 
We use $B=1{,}000$ bootstrap replicates and $N_{\mathrm{reps}}=1{,}000$ Monte Carlo replications, and compare the model-based bootstrap (percentile and pivot CIs) against the episodic bootstrap \citep{hao2021bootstrapping} and plug-in CLT CIs \citep{zhu2024uncertainty, su2025centrallimittheoremstransition}.
We evaluate coverage at the nominal $95\%$ level ($\alpha=0.05$) for the entries $Q(1,0)$, $Q(3,1)$, $Q(6,0)$, and $V(1)$--$V(6)$. Full results appear in Tables~\ref{tab:riverswim-uniform-coverage}--\ref{tab:riverswim-optimal-coverage-T100}.
Figures~\ref{fig:coverage-policy} and~\ref{fig:coverage-optimal} display the coverages of the $n\in\{500,1{,}000\}$ subset for the selected entries $Q(1,0)$, $Q(3,1)$, $Q(6,0)$, $V(2)$, $V(4)$, and $V(5)$ in OPE and OPR, respectively.
Appendices~\ref{app:ope-tables-multilevel} and~\ref{app:opr-tables-multilevel} report additional experiments at nominal $50\%$ and $90\%$, respectively.
Table~\ref{tab:exp-settings} summarizes the settings of all experiments and the corresponding coverage tables. 
All coverage tables use a green cell-shading scheme described below.

\textbf{Cell shading.}
In all coverage tables, we shade the cells green in proportion to empirical coverage
relative to the nominal level.
White indicates coverage at or below a floor value; full green indicates coverage at
or above nominal; intermediate shades indicate coverages between the floor and nominal.
The floor values are $0.75$ (nominal $95\%$), $0.70$ (nominal $90\%$), and $0.30$
(nominal $50\%$).

For OPE, we conclude that the model-based bootstrap substantially outperforms the episodic bootstrap and plug-in CLT across all three target policies, with the percentile CI delivering the strongest overall performance and near-nominal coverage at $n=500$--$1{,}000$.
The performance of the model-based bootstrap is weakest for entries associated with state~$6$, where visitation is sparse and the corresponding transition probabilities are weakly identified. Undercoverage at these entries persists at $n=100$ and is most severe for the mostly-left policy.
The episodic bootstrap fails to produce valid intervals at $n=50$, where a single episode eliminates all bootstrap variation, and suffers undercoverage at $n \le 500$.
The plug-in CI is moderately undercovered at $n\le 100$ for all policies.
Coverage improves monotonically with $n$ for all methods.
The same conclusions hold at nominal $50\%$ and $90\%$ (Appendix~\ref{app:ope-tables-multilevel}).

For OPR, we conclude that the model-based bootstrap with percentile CI performs markedly better than the baselines and is the strongest method in the OPR study, achieving near-nominal coverage at $n=500$--$1{,}000$ for $T=50$ and $T=100$.
The performance of the model-based bootstrap is weakest at $T=10$, where coverage at $n=1{,}000$ remains below nominal, and coverage improves monotonically with $T$.
The pivot CI undercovers at all $(n,T)$ combinations, most severely at $Q_\star(1,0)$, where $Q_\star(1,0)\approx Q_\star(1,1)$ and the non-degeneracy condition in Proposition~\ref{prop:bootstrap-vq} is weakest.
The episodic bootstrap collapses to zero coverage at $n=50$ when $T\ge 50$, where at most one complete episode is available. 
At $T=10$ the collapse is avoided but coverage remains well below nominal.
The plug-in CI severely undercovers at $n\le 100$ and approaches near-nominal performance only at $n=1{,}000, T=100$ for the rightmost states and state-action pairs.
Coverage improves monotonically with $n$ for all methods, and with $T$ for the model-based bootstrap and plug-in CI.
The same conclusions hold at nominal $50\%$ and $90\%$ (Appendix~\ref{app:opr-tables-multilevel}).

\begin{figure*}[t]
\centering
\includegraphics[width=\textwidth]{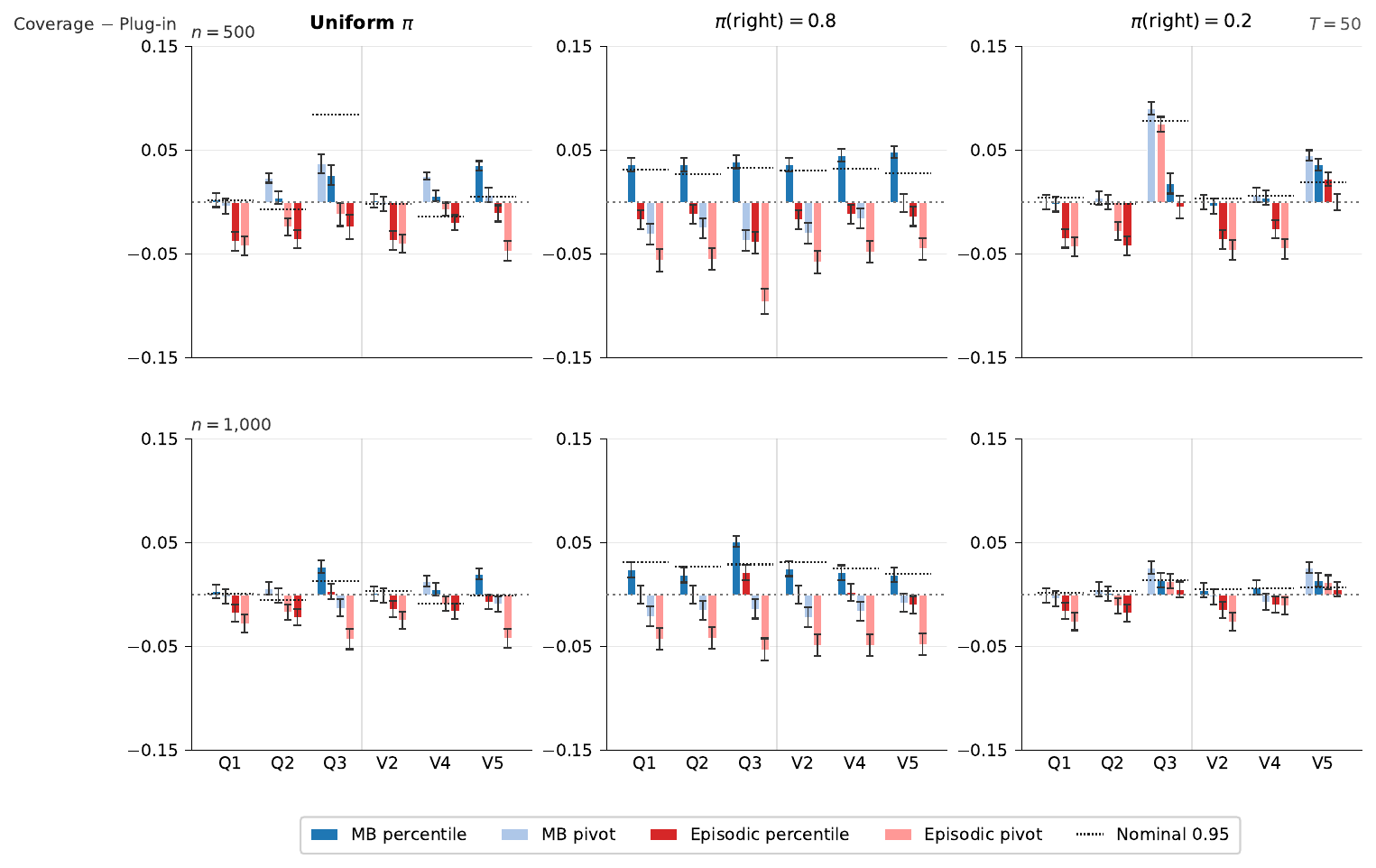}
\caption{Empirical coverage relative to the Plug-in (CLT) baseline for $Q^\pi(1,0)$, $Q^\pi(3,1)$, $Q^\pi(6,0)$, $V^\pi(2)$, $V^\pi(4)$, $V^\pi(5)$ in RiverSwim. Horizontal lines mark the nominal $0.95$ level for each entry. Columns correspond to the uniform, mostly-right, and mostly-left target policies, rows correspond to $n\in\{500,1{,}000\}$, and all panels use $T=50$. This figure visualizes the $n\in\{500,1{,}000\}$ subset of Tables~\ref{tab:riverswim-uniform-coverage}--\ref{tab:riverswim-mostly-left-coverage}. Labels Q1--Q3 and V2, V4, V5 correspond to these entries respectively.
Full results appear in Tables~\ref{tab:riverswim-uniform-coverage}--\ref{tab:riverswim-mostly-left-coverage}.}
\label{fig:coverage-policy}
\end{figure*}

\begin{figure*}[t]
\centering
\includegraphics[width=\textwidth]{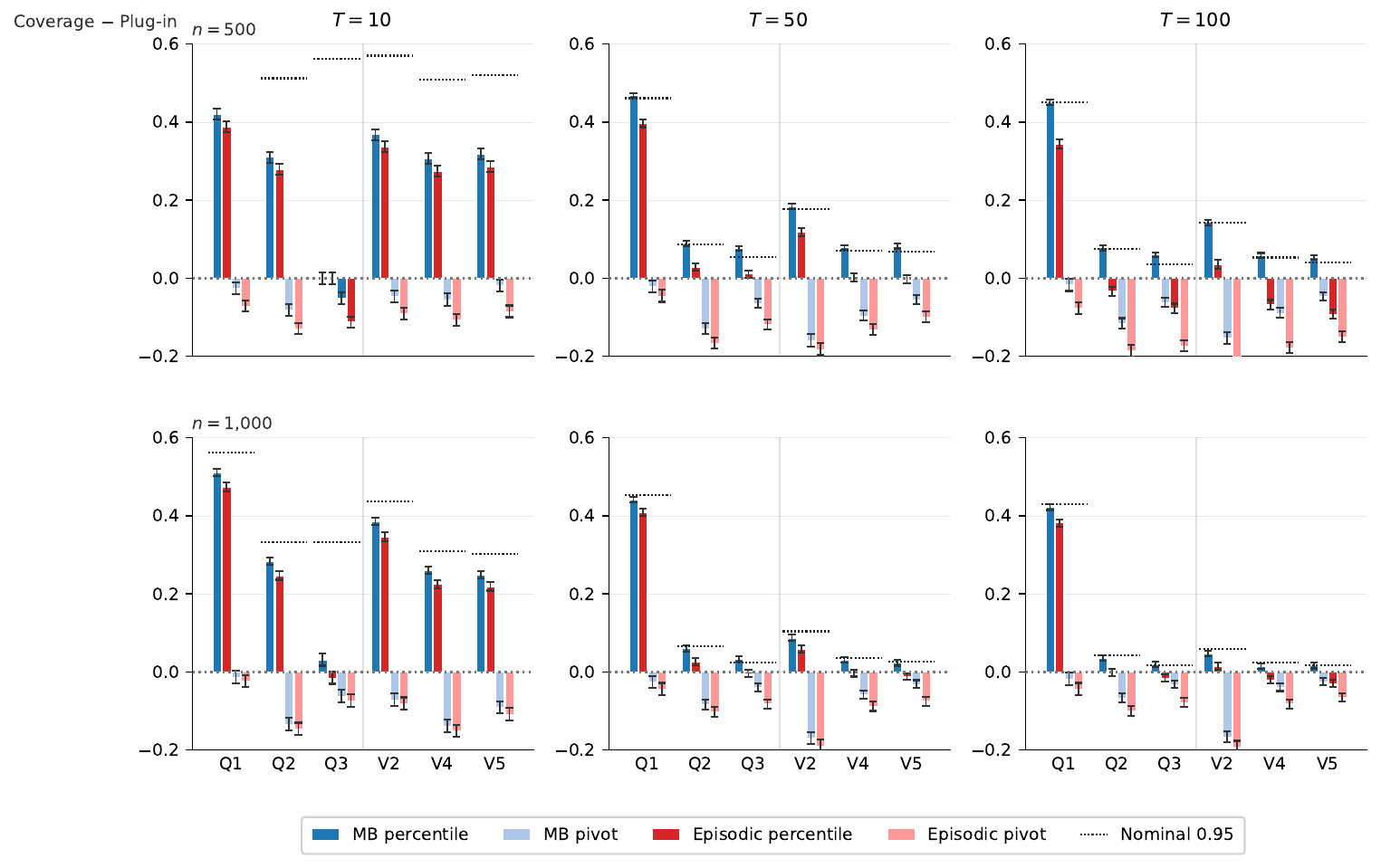}
\caption{Empirical coverage relative to the Plug-in (CLT) baseline for $Q_\star(1,0)$, $Q_\star(3,1)$, $Q_\star(6,0)$, $V_\star(2)$, $V_\star(4)$, $V_\star(5)$ in RiverSwim. Horizontal lines mark the nominal $0.95$ level for each entry. Columns correspond to $T\in\{10,50,100\}$, rows correspond to $n\in\{500,1{,}000\}$, and the target policy is optimal. This figure visualizes the $n\in\{500,1{,}000\}$ subset of Tables~\ref{tab:riverswim-optimal-coverage}--\ref{tab:riverswim-optimal-coverage-T100}. Labels Q1--Q3 and V2, V4, V5 correspond to these entries respectively.
Full results for all $n$ appear in Tables~\ref{tab:riverswim-optimal-coverage}--\ref{tab:riverswim-optimal-coverage-T100}.}
\label{fig:coverage-optimal}
\end{figure*}

\noindent

\textbf{Handling zero counts.} We handle state-action pairs where $N_s^{(a)}=0$ in the bootstrap simulation as follows.
In bootstrap simulation, we use the repaired kernel
\[
\hat M^{\mathrm{sim},(a)}_{s,\cdot}
=
\begin{cases}
\hat M^{(a)}_{s,\cdot}, & N_s^{(a)}>0,\\
\delta_s, & N_s^{(a)}=0,
\end{cases}
\]
where \(\delta_s\) is the point mass at \(s\).
We thus treat an unobserved state-action pair as a self-loop in bootstrap simulation rather than assigning fictitious transition probabilities. 
Similarly, if \(N_s=0\), we set
\(\hat\pi_b(\cdot\mid s)\) to the uniform distribution. Under Assumption~\ref{ass:semi-ergodic}, these rules are used with probability
tending to zero and hence have no asymptotic effect. 

We organize the remainder of this appendix as follows.
Appendix~\ref{app:ope-tables} reports OPE ($V^\pi$, $Q^\pi$) results at nominal $95\%$;
Appendix~\ref{app:opr-tables} reports OPR ($V_\star$, $Q_\star$) results at nominal
$95\%$;
Appendices~\ref{app:ope-tables-multilevel} and~\ref{app:opr-tables-multilevel} report OPE and OPR results at nominal $50\%$ and $90\%$, respectively.
Figures~\ref{fig:coverage-policy-50} and~\ref{fig:coverage-policy-90} display OPE coverages for the $n\in\{500,1{,}000\}$ subset at nominal $50\%$ and $90\%$;
Figures~\ref{fig:coverage-optimal-50} and~\ref{fig:coverage-optimal-90} display the corresponding OPR results.

\begin{table}[ht]
\centering
\caption{Settings for all RiverSwim coverage tables in this appendix.
$K = n/T$ is the number of episodes.
Policy abbreviations: uniform policy $\pi(a\mid s)=0.5$;
mostly-right policy $\pi(1\mid s)=0.8$; mostly-left policy $\pi(0\mid s)=0.8$;
optimal policy $\pi^\star$.
All experiments use $B=1{,}000$ bootstrap replicates and $N_{\mathrm{reps}}=1{,}000$
Monte Carlo replications.
Tables~\ref{tab:riverswim-uniform-coverage}--\ref{tab:riverswim-mostly-left-coverage} report OPE
coverages and Tables~\ref{tab:riverswim-optimal-coverage}--\ref{tab:riverswim-optimal-coverage-T100}
report OPR coverages at nominal $95\%$.
The $50\%$ and $90\%$ rows refer to additional-nominal-level tables in
Appendices~\ref{app:ope-tables-multilevel} and~\ref{app:opr-tables-multilevel}.}
\label{tab:exp-settings}
\vspace{1ex}
\small
\begin{tabular}{|c|c|c|c|c|l|}
\hline
Table & $n$ & $T$ & $K$ & Nominal & Policy \\
\hline
\ref{tab:riverswim-uniform-coverage} & $\{50,100,500,1{,}000\}$ & $50$ & $\{1,2,10,20\}$ & $95\%$ & uniform \\
\hline
\ref{tab:riverswim-mostly-right-coverage} & $\{50,100,500,1{,}000\}$ & $50$ & $\{1,2,10,20\}$ & $95\%$ & mostly-right \\
\hline
\ref{tab:riverswim-mostly-left-coverage} & $\{50,100,500,1{,}000\}$ & $50$ & $\{1,2,10,20\}$ & $95\%$ & mostly-left \\
\hline
\ref{tab:riverswim-optimal-coverage} & $\{50,100,500,1{,}000\}$ & $50$ & $\{1,2,10,20\}$ & $95\%$ & optimal \\
\hline
\ref{tab:riverswim-optimal-coverage-T10} & $\{50,100,500,1{,}000\}$ & $10$ & $\{5,10,50,100\}$ & $95\%$ & optimal \\
\hline
\ref{tab:riverswim-optimal-coverage-T100} & $\{100,500,1{,}000\}$ & $100$ & $\{1,5,10\}$ & $95\%$ & optimal \\
\hline
\ref{tab:riverswim-uniform-coverage-50} / \ref{tab:riverswim-uniform-coverage-90} & $\{50,100,500,1{,}000\}$ & $50$ & $\{1,2,10,20\}$ & $50\%$, $90\%$ & uniform \\
\hline
\ref{tab:riverswim-mostly-right-coverage-50} / \ref{tab:riverswim-mostly-right-coverage-90} & $\{50,100,500,1{,}000\}$ & $50$ & $\{1,2,10,20\}$ & $50\%$, $90\%$ & mostly-right \\
\hline
\ref{tab:riverswim-mostly-left-coverage-50} / \ref{tab:riverswim-mostly-left-coverage-90} & $\{50,100,500,1{,}000\}$ & $50$ & $\{1,2,10,20\}$ & $50\%$, $90\%$ & mostly-left \\
\hline
\ref{tab:riverswim-optimal-coverage-T10-50} / \ref{tab:riverswim-optimal-coverage-T10-90} & $\{50,100,500,1{,}000\}$ & $10$ & $\{5,10,50,100\}$ & $50\%$, $90\%$ & optimal \\
\hline
\ref{tab:riverswim-optimal-coverage-T50-50} / \ref{tab:riverswim-optimal-coverage-T50-90} & $\{50,100,500,1{,}000\}$ & $50$ & $\{1,2,10,20\}$ & $50\%$, $90\%$ & optimal \\
\hline
\ref{tab:riverswim-optimal-coverage-T100-50} / \ref{tab:riverswim-optimal-coverage-T100-90} & $\{100,500,1{,}000\}$ & $100$ & $\{1,5,10\}$ & $50\%$, $90\%$ & optimal \\
\hline
\end{tabular}
\end{table}
\subsection{OPE ($V^\pi$, $Q^\pi$)}
\label{app:ope-tables}

\begin{table}[ht]
\centering
\caption{Empirical coverage (nominal $95\%$) for a fixed target policy $\pi$ (uniform over actions) in RiverSwim. Comparison of model-based bootstrap (ours), episodic bootstrap, and plug-in CI, episode length $T=50$.}
\label{tab:riverswim-uniform-coverage}
\begin{tabular}{c|ccc|cccccc}
\hline
& \multicolumn{3}{c|}{$Q^\pi$ coverage} & \multicolumn{6}{c}{$V^\pi$ coverage} \\ \hline
$n$
& $Q^\pi(1,0)$ & $Q^\pi(3,1)$ & $Q^\pi(6,0)$
& $V^\pi(1)$ & $V^\pi(2)$ & $V^\pi(3)$ & $V^\pi(4)$ & $V^\pi(5)$ & $V^\pi(6)$ \\
\hline
\multicolumn{10}{l}{\textbf{Model-based bootstrap with pivot CI}} \\ \hline
50
& \cellcolor{covheat!62!white}0.873 & \cellcolor{covheat!28!white}0.806 & 0.373
& \cellcolor{covheat!62!white}0.873 & \cellcolor{covheat!53!white}0.856 & \cellcolor{covheat!22!white}0.794 & 0.664 & 0.494 & 0.374 \\
100
& \cellcolor{covheat!99!white}0.948 & \cellcolor{covheat!100!white}0.965 & 0.635
& \cellcolor{covheat!99!white}0.948 & \cellcolor{covheat!100!white}0.957 & \cellcolor{covheat!100!white}0.959 & \cellcolor{covheat!98!white}0.946 & \cellcolor{covheat!27!white}0.804 & 0.629 \\
500
& \cellcolor{covheat!97!white}0.944 & \cellcolor{covheat!100!white}0.980 & \cellcolor{covheat!77!white}0.903
& \cellcolor{covheat!97!white}0.944 & \cellcolor{covheat!100!white}0.950 & \cellcolor{covheat!100!white}0.959 & \cellcolor{covheat!100!white}0.989 & \cellcolor{covheat!100!white}0.952 & \cellcolor{covheat!62!white}0.874 \\
1{,}000
& \cellcolor{covheat!98!white}0.947 & \cellcolor{covheat!100!white}0.961 & \cellcolor{covheat!87!white}0.924
& \cellcolor{covheat!98!white}0.947 & \cellcolor{covheat!98!white}0.946 & \cellcolor{covheat!100!white}0.952 & \cellcolor{covheat!100!white}0.972 & \cellcolor{covheat!96!white}0.942 & \cellcolor{covheat!89!white}0.927 \\
\hline
\multicolumn{10}{l}{\textbf{Model-based bootstrap with percentile CI}} \\
\hline
\end{tabular}
\end{table}

\begin{table}[ht]
\centering
\caption{Empirical coverage (nominal $95\%$) for a fixed target policy $\pi$ with $\pi(\text{right}\mid s)=0.8$ for all $s$ in RiverSwim. Comparison of model-based bootstrap (ours), episodic bootstrap, and plug-in CI, episode length $T=50$.}
\label{tab:riverswim-mostly-right-coverage}
\begin{tabular}{c|ccc|cccccc}
\hline
& \multicolumn{3}{c|}{$Q^\pi$ coverage} & \multicolumn{6}{c}{$V^\pi$ coverage} \\ \hline
$n$
& $Q^\pi(1,0)$ & $Q^\pi(3,1)$ & $Q^\pi(6,0)$
& $V^\pi(1)$ & $V^\pi(2)$ & $V^\pi(3)$ & $V^\pi(4)$ & $V^\pi(5)$ & $V^\pi(6)$ \\
\hline
\multicolumn{10}{l}{\textbf{Model-based bootstrap with pivot CI}} \\ \hline
50
& 0.430 & 0.407 & 0.403
& 0.430 & 0.430 & 0.408 & 0.395 & 0.371 & 0.358 \\
100
& 0.646 & 0.630 & 0.608
& 0.646 & 0.642 & 0.632 & 0.623 & 0.591 & 0.558 \\
500
& \cellcolor{covheat!69!white}0.888 & \cellcolor{covheat!74!white}0.898 & \cellcolor{covheat!65!white}0.880
& \cellcolor{covheat!69!white}0.888 & \cellcolor{covheat!70!white}0.890 & \cellcolor{covheat!73!white}0.895 & \cellcolor{covheat!76!white}0.902 & \cellcolor{covheat!86!white}0.921 & \cellcolor{covheat!65!white}0.879 \\
1{,}000
& \cellcolor{covheat!74!white}0.898 & \cellcolor{covheat!79!white}0.908 & \cellcolor{covheat!79!white}0.907
& \cellcolor{covheat!74!white}0.898 & \cellcolor{covheat!74!white}0.897 & \cellcolor{covheat!78!white}0.906 & \cellcolor{covheat!80!white}0.909 & \cellcolor{covheat!86!white}0.922 & \cellcolor{covheat!88!white}0.925 \\
\hline
\multicolumn{10}{l}{\textbf{Model-based bootstrap with percentile CI}} \\
\hline
\end{tabular}
\end{table}

\begin{table}[ht]
\centering
\caption{Empirical coverage (nominal $95\%$) for a fixed target policy $\pi$ with $\pi(\text{right}\mid s)=0.2$ for all $s$ in RiverSwim. Comparison of model-based bootstrap (ours), episodic bootstrap, and plug-in CI, episode length $T=50$.}
\label{tab:riverswim-mostly-left-coverage}
\begin{tabular}{c|ccc|cccccc}
\hline
& \multicolumn{3}{c|}{$Q^\pi$ coverage} & \multicolumn{6}{c}{$V^\pi$ coverage} \\ \hline
$n$
& $Q^\pi(1,0)$ & $Q^\pi(3,1)$ & $Q^\pi(6,0)$
& $V^\pi(1)$ & $V^\pi(2)$ & $V^\pi(3)$ & $V^\pi(4)$ & $V^\pi(5)$ & $V^\pi(6)$ \\
\hline
\multicolumn{10}{l}{\textbf{Model-based bootstrap with pivot CI}} \\ \hline
50
& \cellcolor{covheat!48!white}0.845 & 0.721 & 0.281
& \cellcolor{covheat!48!white}0.845 & \cellcolor{covheat!37!white}0.825 & 0.705 & 0.501 & 0.331 & 0.276 \\
100
& \cellcolor{covheat!86!white}0.921 & \cellcolor{covheat!92!white}0.933 & 0.573
& \cellcolor{covheat!86!white}0.921 & \cellcolor{covheat!85!white}0.919 & \cellcolor{covheat!79!white}0.908 & \cellcolor{covheat!40!white}0.830 & 0.718 & 0.571 \\
500
& \cellcolor{covheat!98!white}0.946 & \cellcolor{covheat!100!white}0.956 & \cellcolor{covheat!100!white}0.962
& \cellcolor{covheat!98!white}0.946 & \cellcolor{covheat!98!white}0.947 & \cellcolor{covheat!98!white}0.945 & \cellcolor{covheat!100!white}0.951 & \cellcolor{covheat!100!white}0.976 & \cellcolor{covheat!100!white}0.959 \\
1{,}000
& \cellcolor{covheat!97!white}0.944 & \cellcolor{covheat!100!white}0.952 & \cellcolor{covheat!100!white}0.962
& \cellcolor{covheat!97!white}0.944 & \cellcolor{covheat!96!white}0.943 & \cellcolor{covheat!97!white}0.944 & \cellcolor{covheat!94!white}0.937 & \cellcolor{covheat!100!white}0.969 & \cellcolor{covheat!100!white}0.963 \\
\hline
\multicolumn{10}{l}{\textbf{Model-based bootstrap with percentile CI}} \\
\hline
\end{tabular}
\end{table}

\clearpage
\subsection{OPR ($V_\star$, $Q_\star$)}
\label{app:opr-tables}

\begin{table}[ht]
\centering
\caption{Empirical coverage (nominal $95\%$) for selected entries of $Q_\star$ and $V_\star$ in RiverSwim. Comparison of model-based bootstrap, episodic bootstrap and plug-in CI, episode length $T=50$.}
\label{tab:riverswim-optimal-coverage}
\begin{tabular}{c|ccc|cccccc}
\hline
& \multicolumn{3}{c|}{$Q_\star$ coverage} & \multicolumn{6}{c}{$V_\star$ coverage} \\ \hline
$n$
& $Q_\star(1,0)$ & $Q_\star(3,1)$ & $Q_\star(6,0)$
& $V_\star(1)$ & $V_\star(2)$ & $V_\star(3)$ & $V_\star(4)$ & $V_\star(5)$ & $V_\star(6)$ \\
\hline
\multicolumn{10}{l}{\textbf{Model-based bootstrap with pivot CI}} \\ \hline
50
& 0.246 & 0.720 & 0.306
& 0.246 & 0.609 & 0.742 & 0.316 & 0.309 & 0.286 \\
100
& 0.361 & 0.555 & 0.503
& 0.361 & 0.498 & 0.554 & 0.503 & 0.509 & 0.466 \\
500
& 0.468 & 0.735 & \cellcolor{covheat!40!white}0.831
& 0.468 & 0.613 & 0.707 & \cellcolor{covheat!17!white}0.784 & \cellcolor{covheat!38!white}0.827 & \cellcolor{covheat!42!white}0.834 \\
1{,}000
& 0.471 & \cellcolor{covheat!26!white}0.801 & \cellcolor{covheat!69!white}0.887
& 0.471 & 0.676 & \cellcolor{covheat!16!white}0.781 & \cellcolor{covheat!53!white}0.856 & \cellcolor{covheat!72!white}0.893 & \cellcolor{covheat!76!white}0.901 \\
\hline
\multicolumn{10}{l}{\textbf{Model-based bootstrap with percentile CI}} \\
\hline
\end{tabular}
\end{table}

\begin{table}[ht]
\centering
\caption{Empirical coverage (nominal $95\%$) for selected entries of $Q_\star$ and $V_\star$ in RiverSwim. Comparison of model-based bootstrap, episodic bootstrap and plug-in CI, episode length $T=10$.}
\label{tab:riverswim-optimal-coverage-T10}
\begin{tabular}{c|ccc|cccccc}
\hline
& \multicolumn{3}{c|}{$Q_\star$ coverage} & \multicolumn{6}{c}{$V_\star$ coverage} \\ \hline
$n$
& $Q_\star(1,0)$ & $Q_\star(3,1)$ & $Q_\star(6,0)$
& $V_\star(1)$ & $V_\star(2)$ & $V_\star(3)$ & $V_\star(4)$ & $V_\star(5)$ & $V_\star(6)$ \\
\hline
\multicolumn{10}{l}{\textbf{Model-based bootstrap with pivot CI}} \\ \hline
50
& 0.047 & 0.517 & 0.049
& 0.047 & 0.546 & 0.667 & 0.052 & 0.052 & 0.046 \\
100
& 0.102 & 0.468 & 0.110
& 0.102 & 0.221 & 0.589 & 0.120 & 0.118 & 0.104 \\
500
& 0.302 & 0.357 & 0.388
& 0.302 & 0.333 & 0.350 & 0.386 & 0.410 & 0.347 \\
1{,}000
& 0.375 & 0.484 & 0.555
& 0.375 & 0.443 & 0.476 & 0.502 & 0.557 & 0.500 \\
\hline
\multicolumn{10}{l}{\textbf{Model-based bootstrap with percentile CI}} \\
\hline
\end{tabular}
\end{table}

\begin{table}[ht]
\centering
\caption{Empirical coverage (nominal $95\%$) for selected entries of $Q_\star$ and $V_\star$ in RiverSwim. Comparison of model-based bootstrap, episodic bootstrap, and plug-in CI, episode length $T=100$.}
\label{tab:riverswim-optimal-coverage-T100}
\begin{tabular}{c|ccc|cccccc}
\hline
& \multicolumn{3}{c|}{$Q_\star$ coverage} & \multicolumn{6}{c}{$V_\star$ coverage} \\ \hline
$n$
& $Q_\star(1,0)$ & $Q_\star(3,1)$ & $Q_\star(6,0)$
& $V_\star(1)$ & $V_\star(2)$ & $V_\star(3)$ & $V_\star(4)$ & $V_\star(5)$ & $V_\star(6)$ \\
\hline
\multicolumn{10}{l}{\textbf{Model-based bootstrap with pivot CI}} \\ \hline
100
& 0.350 & 0.549 & 0.524
& 0.350 & 0.521 & 0.549 & 0.527 & 0.545 & 0.492 \\
500
& 0.482 & \cellcolor{covheat!5!white}0.759 & \cellcolor{covheat!52!white}0.853
& 0.482 & 0.655 & 0.742 & \cellcolor{covheat!29!white}0.808 & \cellcolor{covheat!57!white}0.863 & \cellcolor{covheat!61!white}0.872 \\
1{,}000
& 0.502 & \cellcolor{covheat!46!white}0.842 & \cellcolor{covheat!76!white}0.902
& 0.502 & 0.724 & \cellcolor{covheat!40!white}0.830 & \cellcolor{covheat!69!white}0.887 & \cellcolor{covheat!79!white}0.907 & \cellcolor{covheat!79!white}0.908 \\
\hline
\multicolumn{10}{l}{\textbf{Model-based bootstrap with percentile CI}} \\ \hline
100
& 0.689 & 0.693 & 0.566
& 0.689 & 0.691 & 0.693 & 0.692 & 0.695 & 0.611 \\
\hline
\end{tabular}
\end{table}

\clearpage
\subsection{OPE: Additional Nominal Levels ($50\%$ and $90\%$)}
\label{app:ope-tables-multilevel}

\begin{figure}[ht]
\centering
\includegraphics[width=\linewidth]{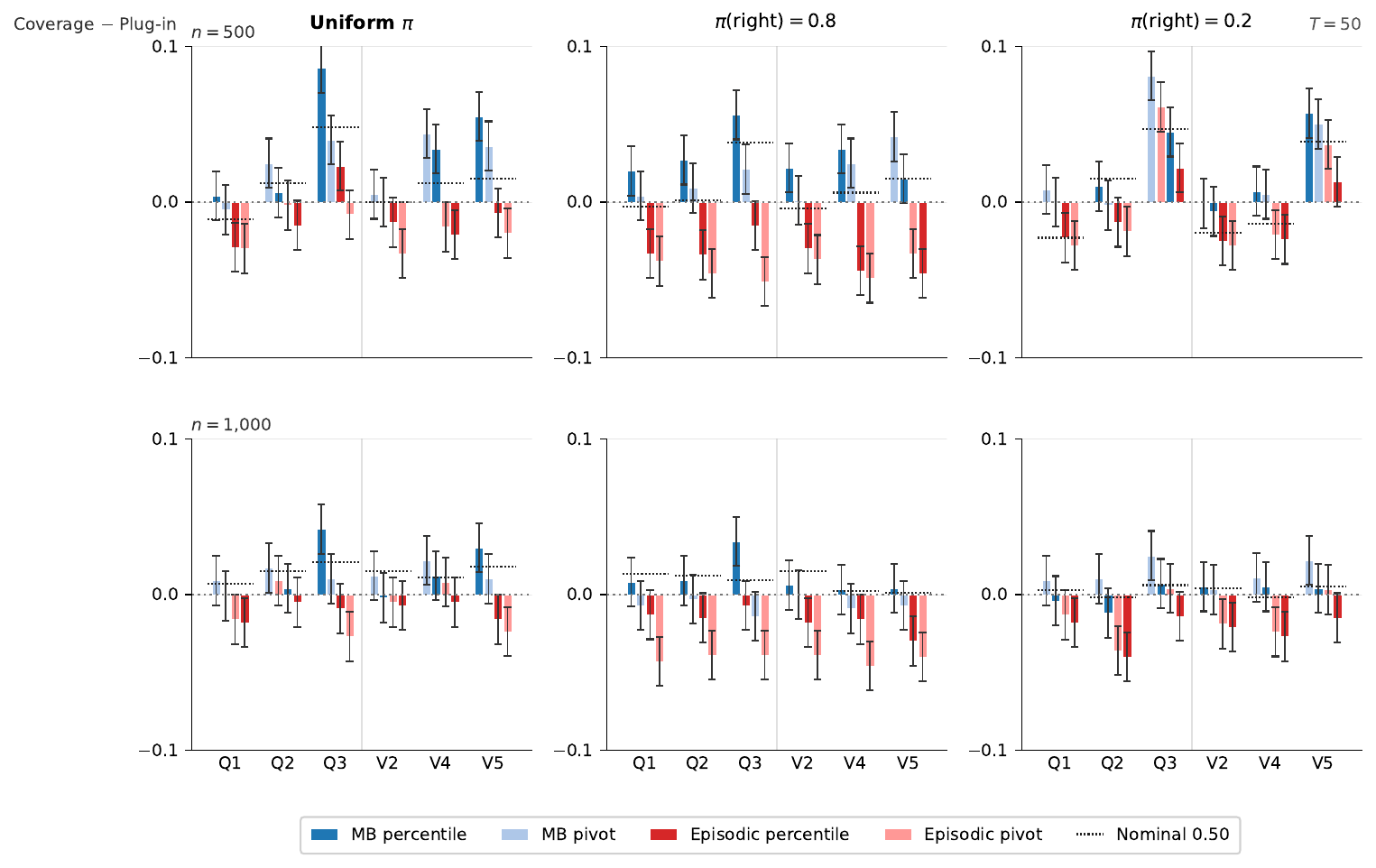}
\caption{Empirical coverage relative to the Plug-in (CLT) baseline for
$Q^\pi(1,0)$, $Q^\pi(3,1)$, $Q^\pi(6,0)$, $V^\pi(2)$, $V^\pi(4)$, $V^\pi(5)$
in RiverSwim at nominal $50\%$. Horizontal lines mark the nominal $0.50$ level
for each entry. Columns correspond to the uniform, mostly-right, and mostly-left
target policies, rows correspond to $n\in\{500,1{,}000\}$, and all panels use $T=50$.
Full results appear in Tables~\ref{tab:riverswim-uniform-coverage-50}--\ref{tab:riverswim-mostly-left-coverage-50}.}
\label{fig:coverage-policy-50}
\end{figure}

\begin{figure}[ht]
\centering
\includegraphics[width=\linewidth]{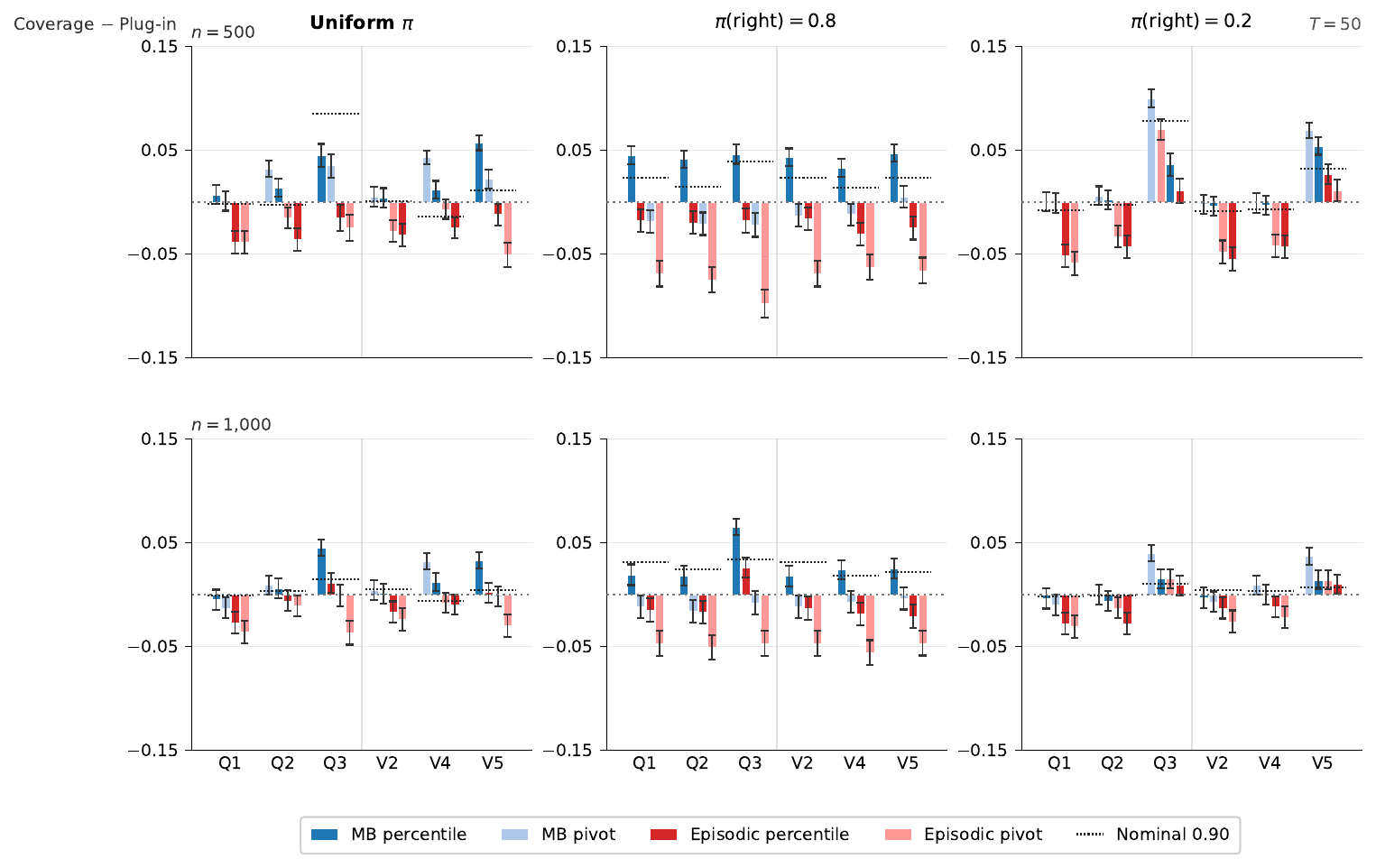}
\caption{Empirical coverage relative to the Plug-in (CLT) baseline for
$Q^\pi(1,0)$, $Q^\pi(3,1)$, $Q^\pi(6,0)$, $V^\pi(2)$, $V^\pi(4)$, $V^\pi(5)$
in RiverSwim at nominal $90\%$. Horizontal lines mark the nominal $0.90$ level
for each entry. Columns correspond to the uniform, mostly-right, and mostly-left
target policies, rows correspond to $n\in\{500,1{,}000\}$, and all panels use $T=50$.
Full results appear in Tables~\ref{tab:riverswim-uniform-coverage-90}--\ref{tab:riverswim-mostly-left-coverage-90}.}
\label{fig:coverage-policy-90}
\end{figure}

\begin{table}[ht]
\centering
\caption{Empirical coverage (nominal $50\%$) for a fixed target policy $\pi$ (uniform over actions) in RiverSwim. Comparison of model-based bootstrap, episodic bootstrap, and plug-in (CLT) CI, episode length $T=50$.}
\label{tab:riverswim-uniform-coverage-50}
\begin{tabular}{c|ccc|cccccc}
\hline
& \multicolumn{3}{c|}{$Q^\pi$ coverage} & \multicolumn{6}{c}{$V^\pi$ coverage} \\ \hline
$n$
& $Q^\pi(1,0)$ & $Q^\pi(3,1)$ & $Q^\pi(6,0)$
& $V^\pi(1)$ & $V^\pi(2)$ & $V^\pi(3)$ & $V^\pi(4)$ & $V^\pi(5)$ & $V^\pi(6)$ \\
\hline
\multicolumn{10}{l}{\textbf{Model-based bootstrap with pivot CI}} \\ \hline
50
& \cellcolor{covheat!100!white}0.543 & \cellcolor{covheat!89!white}0.478 & 0.253
& \cellcolor{covheat!100!white}0.543 & \cellcolor{covheat!100!white}0.548 & \cellcolor{covheat!79!white}0.457 & \cellcolor{covheat!38!white}0.375 & 0.262 & 0.255 \\
100
& \cellcolor{covheat!100!white}0.550 & \cellcolor{covheat!100!white}0.602 & \cellcolor{covheat!68!white}0.435
& \cellcolor{covheat!100!white}0.550 & \cellcolor{covheat!100!white}0.571 & \cellcolor{covheat!100!white}0.586 & \cellcolor{covheat!100!white}0.582 & \cellcolor{covheat!100!white}0.504 & \cellcolor{covheat!77!white}0.453 \\
500
& \cellcolor{covheat!100!white}0.506 & \cellcolor{covheat!100!white}0.513 & \cellcolor{covheat!96!white}0.492
& \cellcolor{covheat!100!white}0.506 & \cellcolor{covheat!100!white}0.505 & \cellcolor{covheat!100!white}0.515 & \cellcolor{covheat!100!white}0.532 & \cellcolor{covheat!100!white}0.521 & \cellcolor{covheat!100!white}0.502 \\
1{,}000
& \cellcolor{covheat!100!white}0.502 & \cellcolor{covheat!100!white}0.502 & \cellcolor{covheat!94!white}0.489
& \cellcolor{covheat!100!white}0.502 & \cellcolor{covheat!98!white}0.497 & \cellcolor{covheat!98!white}0.496 & \cellcolor{covheat!100!white}0.511 & \cellcolor{covheat!96!white}0.492 & \cellcolor{covheat!100!white}0.506 \\
\hline
\multicolumn{10}{l}{\textbf{Model-based bootstrap with percentile CI}} \\
\hline
\end{tabular}
\end{table}

\begin{table}[ht]
\centering
\caption{Empirical coverage (nominal $50\%$) for a fixed target policy $\pi$ (mostly-right: $\pi(1\mid s)=0.8$) in RiverSwim. Comparison of model-based bootstrap, episodic bootstrap, and plug-in (CLT) CI, episode length $T=50$.}
\label{tab:riverswim-mostly-right-coverage-50}
\begin{tabular}{c|ccc|cccccc}
\hline
& \multicolumn{3}{c|}{$Q^\pi$ coverage} & \multicolumn{6}{c}{$V^\pi$ coverage} \\ \hline
$n$
& $Q^\pi(1,0)$ & $Q^\pi(3,1)$ & $Q^\pi(6,0)$
& $V^\pi(1)$ & $V^\pi(2)$ & $V^\pi(3)$ & $V^\pi(4)$ & $V^\pi(5)$ & $V^\pi(6)$ \\
\hline
\multicolumn{10}{l}{\textbf{Model-based bootstrap with pivot CI}} \\ \hline
50
& 0.271 & 0.224 & 0.232
& 0.271 & 0.261 & 0.229 & 0.198 & 0.180 & 0.219 \\
100
& \cellcolor{covheat!64!white}0.427 & \cellcolor{covheat!53!white}0.406 & \cellcolor{covheat!61!white}0.422
& \cellcolor{covheat!64!white}0.427 & \cellcolor{covheat!62!white}0.424 & \cellcolor{covheat!57!white}0.415 & \cellcolor{covheat!45!white}0.389 & \cellcolor{covheat!26!white}0.352 & \cellcolor{covheat!45!white}0.390 \\
500
& \cellcolor{covheat!100!white}0.507 & \cellcolor{covheat!100!white}0.508 & \cellcolor{covheat!91!white}0.483
& \cellcolor{covheat!100!white}0.507 & \cellcolor{covheat!100!white}0.505 & \cellcolor{covheat!100!white}0.509 & \cellcolor{covheat!100!white}0.519 & \cellcolor{covheat!100!white}0.527 & \cellcolor{covheat!100!white}0.533 \\
1{,}000
& \cellcolor{covheat!90!white}0.480 & \cellcolor{covheat!92!white}0.485 & \cellcolor{covheat!88!white}0.477
& \cellcolor{covheat!90!white}0.480 & \cellcolor{covheat!92!white}0.485 & \cellcolor{covheat!93!white}0.486 & \cellcolor{covheat!94!white}0.489 & \cellcolor{covheat!96!white}0.492 & \cellcolor{covheat!94!white}0.487 \\
\hline
\multicolumn{10}{l}{\textbf{Model-based bootstrap with percentile CI}} \\
\hline
\end{tabular}
\end{table}

\begin{table}[ht]
\centering
\caption{Empirical coverage (nominal $50\%$) for a fixed target policy $\pi$ (mostly-left: $\pi(0\mid s)=0.8$) in RiverSwim. Comparison of model-based bootstrap, episodic bootstrap, and plug-in (CLT) CI, episode length $T=50$.}
\label{tab:riverswim-mostly-left-coverage-50}
\begin{tabular}{c|ccc|cccccc}
\hline
& \multicolumn{3}{c|}{$Q^\pi$ coverage} & \multicolumn{6}{c}{$V^\pi$ coverage} \\ \hline
$n$
& $Q^\pi(1,0)$ & $Q^\pi(3,1)$ & $Q^\pi(6,0)$
& $V^\pi(1)$ & $V^\pi(2)$ & $V^\pi(3)$ & $V^\pi(4)$ & $V^\pi(5)$ & $V^\pi(6)$ \\
\hline
\multicolumn{10}{l}{\textbf{Model-based bootstrap with pivot CI}} \\ \hline
50
& \cellcolor{covheat!100!white}0.533 & \cellcolor{covheat!80!white}0.459 & 0.234
& \cellcolor{covheat!100!white}0.533 & \cellcolor{covheat!100!white}0.530 & \cellcolor{covheat!75!white}0.449 & 0.294 & 0.145 & 0.227 \\
100
& \cellcolor{covheat!100!white}0.557 & \cellcolor{covheat!100!white}0.573 & \cellcolor{covheat!70!white}0.440
& \cellcolor{covheat!100!white}0.557 & \cellcolor{covheat!100!white}0.559 & \cellcolor{covheat!100!white}0.560 & \cellcolor{covheat!100!white}0.500 & \cellcolor{covheat!30!white}0.359 & \cellcolor{covheat!73!white}0.445 \\
500
& \cellcolor{covheat!100!white}0.531 & \cellcolor{covheat!91!white}0.483 & \cellcolor{covheat!100!white}0.534
& \cellcolor{covheat!100!white}0.531 & \cellcolor{covheat!100!white}0.519 & \cellcolor{covheat!100!white}0.514 & \cellcolor{covheat!100!white}0.519 & \cellcolor{covheat!100!white}0.511 & \cellcolor{covheat!100!white}0.520 \\
1{,}000
& \cellcolor{covheat!100!white}0.506 & \cellcolor{covheat!100!white}0.512 & \cellcolor{covheat!100!white}0.519
& \cellcolor{covheat!100!white}0.506 & \cellcolor{covheat!100!white}0.499 & \cellcolor{covheat!100!white}0.505 & \cellcolor{covheat!100!white}0.513 & \cellcolor{covheat!100!white}0.517 & \cellcolor{covheat!100!white}0.521 \\
\hline
\multicolumn{10}{l}{\textbf{Model-based bootstrap with percentile CI}} \\
\hline
\end{tabular}
\end{table}

\clearpage

\begin{table}[ht]
\centering
\caption{Empirical coverage (nominal $90\%$) for a fixed target policy $\pi$ (uniform over actions) in RiverSwim. Comparison of model-based bootstrap, episodic bootstrap, and plug-in (CLT) CI, episode length $T=50$.}
\label{tab:riverswim-uniform-coverage-90}
\begin{tabular}{c|ccc|cccccc}
\hline
& \multicolumn{3}{c|}{$Q^\pi$ coverage} & \multicolumn{6}{c}{$V^\pi$ coverage} \\ \hline
$n$
& $Q^\pi(1,0)$ & $Q^\pi(3,1)$ & $Q^\pi(6,0)$
& $V^\pi(1)$ & $V^\pi(2)$ & $V^\pi(3)$ & $V^\pi(4)$ & $V^\pi(5)$ & $V^\pi(6)$ \\
\hline
\multicolumn{10}{l}{\textbf{Model-based bootstrap with pivot CI}} \\ \hline
50
& \cellcolor{covheat!71!white}0.842 & \cellcolor{covheat!42!white}0.784 & 0.365
& \cellcolor{covheat!71!white}0.842 & \cellcolor{covheat!67!white}0.834 & \cellcolor{covheat!36!white}0.772 & 0.649 & 0.479 & 0.362 \\
100
& \cellcolor{covheat!100!white}0.912 & \cellcolor{covheat!100!white}0.952 & 0.621
& \cellcolor{covheat!100!white}0.912 & \cellcolor{covheat!100!white}0.925 & \cellcolor{covheat!100!white}0.936 & \cellcolor{covheat!100!white}0.934 & \cellcolor{covheat!43!white}0.785 & 0.614 \\
500
& \cellcolor{covheat!100!white}0.903 & \cellcolor{covheat!100!white}0.935 & \cellcolor{covheat!75!white}0.850
& \cellcolor{covheat!100!white}0.903 & \cellcolor{covheat!100!white}0.904 & \cellcolor{covheat!100!white}0.919 & \cellcolor{covheat!100!white}0.957 & \cellcolor{covheat!100!white}0.911 & \cellcolor{covheat!64!white}0.829 \\
1{,}000
& \cellcolor{covheat!94!white}0.888 & \cellcolor{covheat!100!white}0.906 & \cellcolor{covheat!92!white}0.884
& \cellcolor{covheat!94!white}0.888 & \cellcolor{covheat!100!white}0.899 & \cellcolor{covheat!100!white}0.905 & \cellcolor{covheat!100!white}0.938 & \cellcolor{covheat!97!white}0.894 & \cellcolor{covheat!89!white}0.879 \\
\hline
\multicolumn{10}{l}{\textbf{Model-based bootstrap with percentile CI}} \\
\hline
\end{tabular}
\end{table}

\begin{table}[ht]
\centering
\caption{Empirical coverage (nominal $90\%$) for a fixed target policy $\pi$ (mostly-right: $\pi(1\mid s)=0.8$) in RiverSwim. Comparison of model-based bootstrap, episodic bootstrap, and plug-in (CLT) CI, episode length $T=50$.}
\label{tab:riverswim-mostly-right-coverage-90}
\begin{tabular}{c|ccc|cccccc}
\hline
& \multicolumn{3}{c|}{$Q^\pi$ coverage} & \multicolumn{6}{c}{$V^\pi$ coverage} \\ \hline
$n$
& $Q^\pi(1,0)$ & $Q^\pi(3,1)$ & $Q^\pi(6,0)$
& $V^\pi(1)$ & $V^\pi(2)$ & $V^\pi(3)$ & $V^\pi(4)$ & $V^\pi(5)$ & $V^\pi(6)$ \\
\hline
\multicolumn{10}{l}{\textbf{Model-based bootstrap with pivot CI}} \\ \hline
50
& 0.424 & 0.390 & 0.393
& 0.424 & 0.418 & 0.393 & 0.372 & 0.349 & 0.349 \\
100
& 0.637 & 0.618 & 0.595
& 0.637 & 0.632 & 0.622 & 0.605 & 0.566 & 0.542 \\
500
& \cellcolor{covheat!79!white}0.858 & \cellcolor{covheat!82!white}0.864 & \cellcolor{covheat!69!white}0.839
& \cellcolor{covheat!79!white}0.858 & \cellcolor{covheat!82!white}0.864 & \cellcolor{covheat!83!white}0.866 & \cellcolor{covheat!87!white}0.874 & \cellcolor{covheat!91!white}0.882 & \cellcolor{covheat!73!white}0.846 \\
1{,}000
& \cellcolor{covheat!78!white}0.857 & \cellcolor{covheat!80!white}0.860 & \cellcolor{covheat!79!white}0.858
& \cellcolor{covheat!78!white}0.857 & \cellcolor{covheat!78!white}0.857 & \cellcolor{covheat!80!white}0.861 & \cellcolor{covheat!87!white}0.875 & \cellcolor{covheat!87!white}0.874 & \cellcolor{covheat!90!white}0.880 \\
\hline
\multicolumn{10}{l}{\textbf{Model-based bootstrap with percentile CI}} \\
\hline
\end{tabular}
\end{table}

\begin{table}[ht]
\centering
\caption{Empirical coverage (nominal $90\%$) for a fixed target policy $\pi$ (mostly-left: $\pi(0\mid s)=0.8$) in RiverSwim. Comparison of model-based bootstrap, episodic bootstrap, and plug-in (CLT) CI, episode length $T=50$.}
\label{tab:riverswim-mostly-left-coverage-90}
\begin{tabular}{c|ccc|cccccc}
\hline
& \multicolumn{3}{c|}{$Q^\pi$ coverage} & \multicolumn{6}{c}{$V^\pi$ coverage} \\ \hline
$n$
& $Q^\pi(1,0)$ & $Q^\pi(3,1)$ & $Q^\pi(6,0)$
& $V^\pi(1)$ & $V^\pi(2)$ & $V^\pi(3)$ & $V^\pi(4)$ & $V^\pi(5)$ & $V^\pi(6)$ \\
\hline
\multicolumn{10}{l}{\textbf{Model-based bootstrap with pivot CI}} \\ \hline
50
& \cellcolor{covheat!60!white}0.821 & 0.694 & 0.280
& \cellcolor{covheat!60!white}0.821 & \cellcolor{covheat!51!white}0.802 & 0.684 & 0.468 & 0.292 & 0.274 \\
100
& \cellcolor{covheat!89!white}0.878 & \cellcolor{covheat!100!white}0.911 & 0.567
& \cellcolor{covheat!89!white}0.878 & \cellcolor{covheat!92!white}0.884 & \cellcolor{covheat!86!white}0.873 & \cellcolor{covheat!47!white}0.794 & 0.668 & 0.565 \\
500
& \cellcolor{covheat!100!white}0.907 & \cellcolor{covheat!100!white}0.909 & \cellcolor{covheat!100!white}0.922
& \cellcolor{covheat!100!white}0.907 & \cellcolor{covheat!100!white}0.907 & \cellcolor{covheat!100!white}0.911 & \cellcolor{covheat!100!white}0.906 & \cellcolor{covheat!100!white}0.937 & \cellcolor{covheat!100!white}0.910 \\
1{,}000
& \cellcolor{covheat!96!white}0.892 & \cellcolor{covheat!100!white}0.901 & \cellcolor{covheat!100!white}0.930
& \cellcolor{covheat!96!white}0.892 & \cellcolor{covheat!94!white}0.889 & \cellcolor{covheat!98!white}0.897 & \cellcolor{covheat!100!white}0.906 & \cellcolor{covheat!100!white}0.930 & \cellcolor{covheat!100!white}0.922 \\
\hline
\multicolumn{10}{l}{\textbf{Model-based bootstrap with percentile CI}} \\
\hline
\end{tabular}
\end{table}

\clearpage
\subsection{OPR: Additional Nominal Levels ($50\%$ and $90\%$)}
\label{app:opr-tables-multilevel}

\begin{figure}[ht]
\centering
\includegraphics[width=\linewidth]{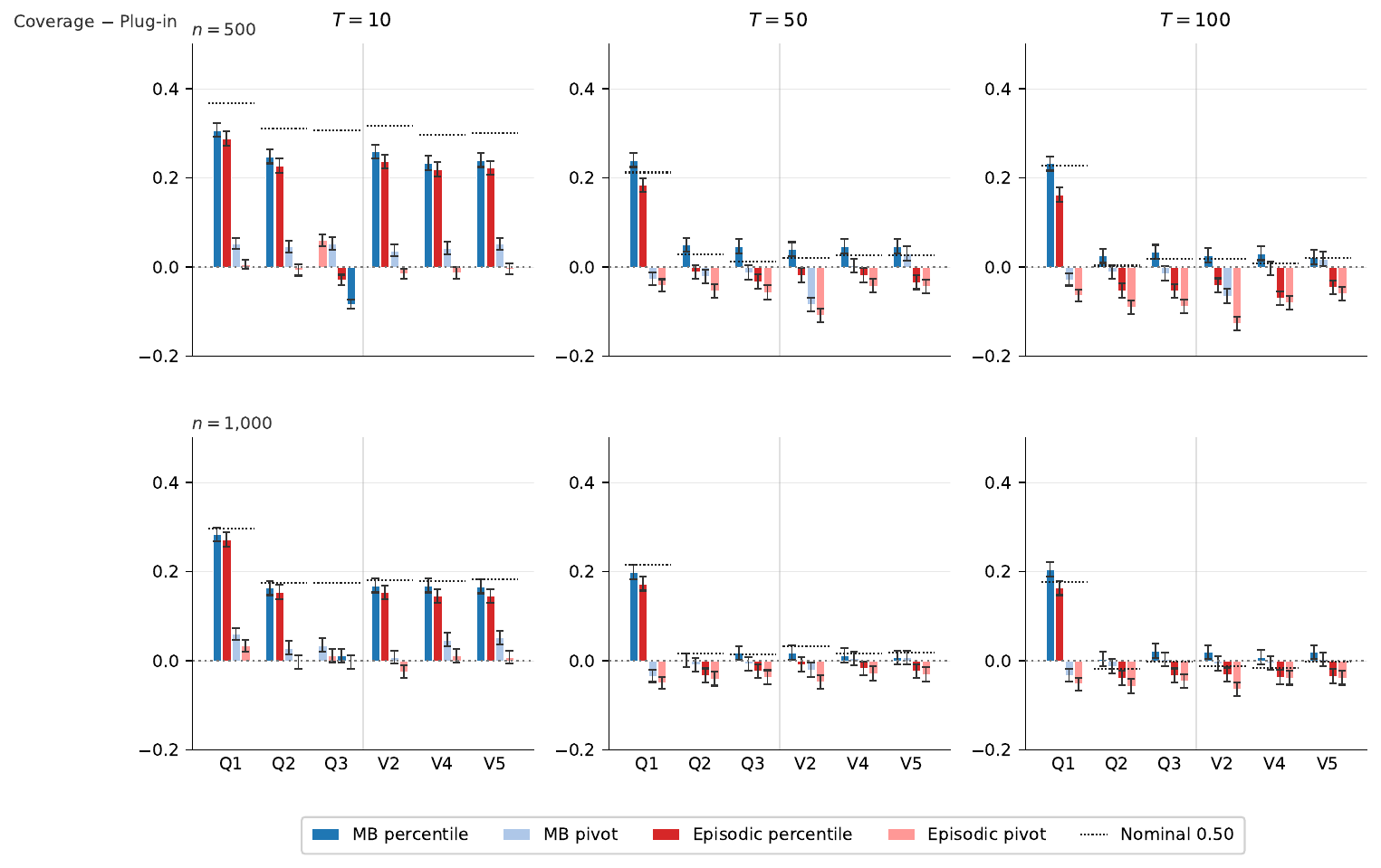}
\caption{Empirical coverage relative to the Plug-in (CLT) baseline for
$Q_\star(1,0)$, $Q_\star(3,1)$, $Q_\star(6,0)$, $V_\star(2)$, $V_\star(4)$, $V_\star(5)$
in RiverSwim at nominal $50\%$. Horizontal lines mark the nominal $0.50$ level
for each entry. Columns correspond to $T\in\{10,50,100\}$, rows correspond to
$n\in\{500,1{,}000\}$.
Full results appear in Tables~\ref{tab:riverswim-optimal-coverage-T10-50}--\ref{tab:riverswim-optimal-coverage-T100-50}.}
\label{fig:coverage-optimal-50}
\end{figure}

\begin{figure}[ht]
\centering
\includegraphics[width=\linewidth]{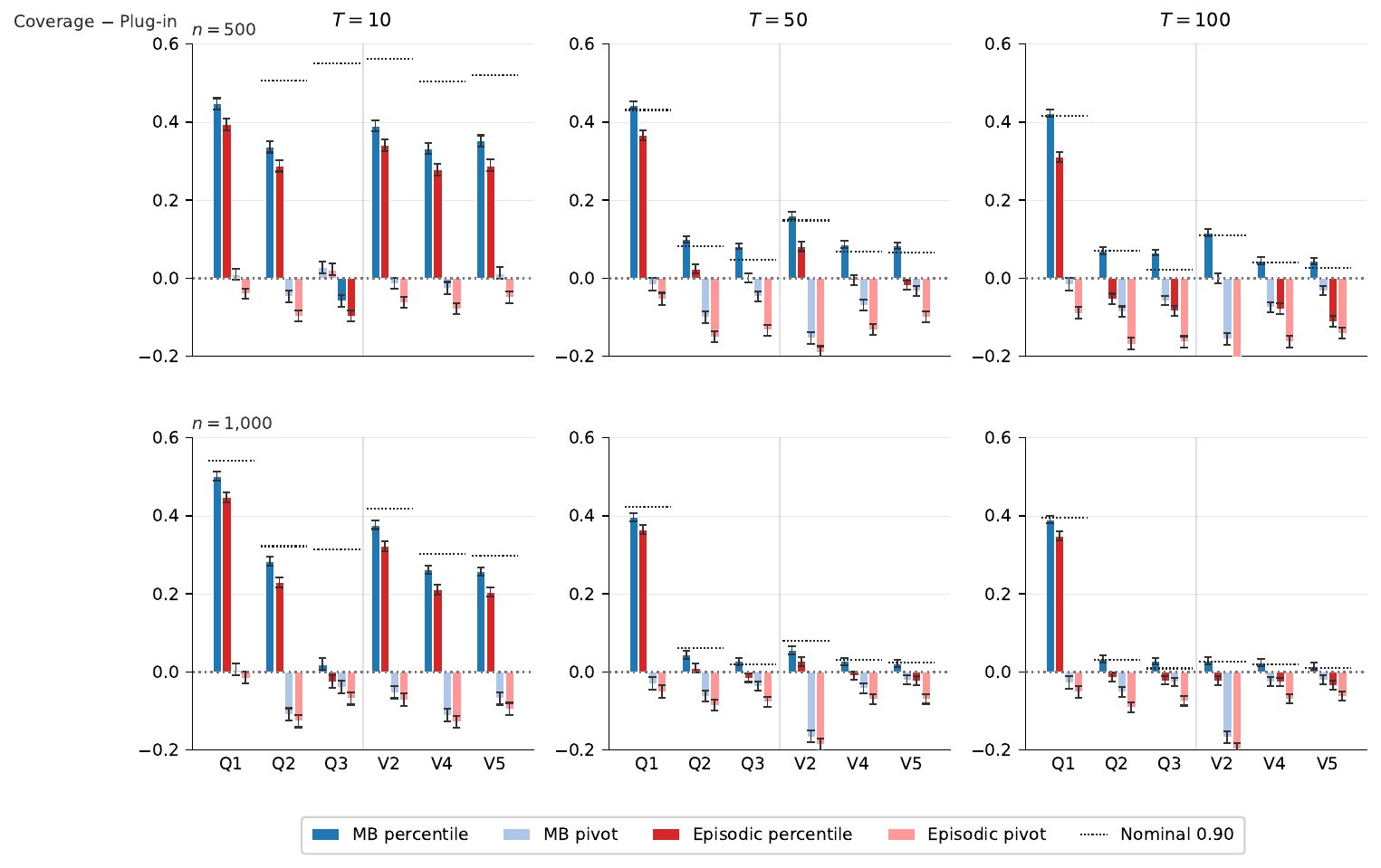}
\caption{Empirical coverage relative to the Plug-in (CLT) baseline for
$Q_\star(1,0)$, $Q_\star(3,1)$, $Q_\star(6,0)$, $V_\star(2)$, $V_\star(4)$, $V_\star(5)$
in RiverSwim at nominal $90\%$. Horizontal lines mark the nominal $0.90$ level
for each entry. Columns correspond to $T\in\{10,50,100\}$, rows correspond to
$n\in\{500,1{,}000\}$.
Full results appear in Tables~\ref{tab:riverswim-optimal-coverage-T10-90}--\ref{tab:riverswim-optimal-coverage-T100-90}.}
\label{fig:coverage-optimal-90}
\end{figure}

\begin{table}[ht]
\centering
\caption{Empirical coverage (nominal $50\%$) for $Q_\star$ and $V_\star$ in RiverSwim. Comparison of model-based bootstrap, episodic bootstrap, and plug-in (CLT) CI, episode length $T=10$.}
\label{tab:riverswim-optimal-coverage-T10-50}
\begin{tabular}{c|ccc|cccccc}
\hline
& \multicolumn{3}{c|}{$Q_\star$ coverage} & \multicolumn{6}{c}{$V_\star$ coverage} \\ \hline
$n$
& $Q_\star(1,0)$ & $Q_\star(3,1)$ & $Q_\star(6,0)$
& $V_\star(1)$ & $V_\star(2)$ & $V_\star(3)$ & $V_\star(4)$ & $V_\star(5)$ & $V_\star(6)$ \\
\hline
\multicolumn{10}{l}{\textbf{Model-based bootstrap with pivot CI}} \\ \hline
50
& 0.013 & 0.177 & 0.028
& 0.013 & 0.102 & 0.245 & 0.017 & 0.016 & 0.027 \\
100
& 0.037 & 0.087 & 0.066
& 0.037 & 0.053 & 0.108 & 0.053 & 0.055 & 0.057 \\
500
& 0.186 & 0.236 & 0.246
& 0.186 & 0.220 & 0.234 & 0.247 & 0.251 & 0.223 \\
1{,}000
& 0.263 & \cellcolor{covheat!27!white}0.355 & \cellcolor{covheat!30!white}0.361
& 0.263 & \cellcolor{covheat!14!white}0.328 & \cellcolor{covheat!25!white}0.351 & \cellcolor{covheat!34!white}0.368 & \cellcolor{covheat!35!white}0.370 & \cellcolor{covheat!13!white}0.326 \\
\hline
\multicolumn{10}{l}{\textbf{Model-based bootstrap with percentile CI}} \\
\hline
\end{tabular}
\end{table}

\begin{table}[ht]
\centering
\caption{Empirical coverage (nominal $90\%$) for $Q_\star$ and $V_\star$ in RiverSwim. Comparison of model-based bootstrap, episodic bootstrap, and plug-in (CLT) CI, episode length $T=10$.}
\label{tab:riverswim-optimal-coverage-T10-90}
\begin{tabular}{c|ccc|cccccc}
\hline
& \multicolumn{3}{c|}{$Q_\star$ coverage} & \multicolumn{6}{c}{$V_\star$ coverage} \\ \hline
$n$
& $Q_\star(1,0)$ & $Q_\star(3,1)$ & $Q_\star(6,0)$
& $V_\star(1)$ & $V_\star(2)$ & $V_\star(3)$ & $V_\star(4)$ & $V_\star(5)$ & $V_\star(6)$ \\
\hline
\multicolumn{10}{l}{\textbf{Model-based bootstrap with pivot CI}} \\ \hline
50
& 0.047 & 0.479 & 0.047
& 0.047 & 0.407 & 0.620 & 0.052 & 0.052 & 0.045 \\
100
& 0.098 & 0.364 & 0.106
& 0.098 & 0.155 & 0.464 & 0.118 & 0.118 & 0.098 \\
500
& 0.293 & 0.347 & 0.377
& 0.293 & 0.325 & 0.341 & 0.372 & 0.392 & 0.337 \\
1{,}000
& 0.364 & 0.469 & 0.548
& 0.364 & 0.430 & 0.461 & 0.487 & 0.535 & 0.492 \\
\hline
\multicolumn{10}{l}{\textbf{Model-based bootstrap with percentile CI}} \\
\hline
\end{tabular}
\end{table}

\begin{table}[ht]
\centering
\caption{Empirical coverage (nominal $50\%$) for $Q_\star$ and $V_\star$ in RiverSwim. Comparison of model-based bootstrap, episodic bootstrap, and plug-in (CLT) CI, episode length $T=50$.}
\label{tab:riverswim-optimal-coverage-T50-50}
\begin{tabular}{c|ccc|cccccc}
\hline
& \multicolumn{3}{c|}{$Q_\star$ coverage} & \multicolumn{6}{c}{$V_\star$ coverage} \\ \hline
$n$
& $Q_\star(1,0)$ & $Q_\star(3,1)$ & $Q_\star(6,0)$
& $V_\star(1)$ & $V_\star(2)$ & $V_\star(3)$ & $V_\star(4)$ & $V_\star(5)$ & $V_\star(6)$ \\
\hline
\multicolumn{10}{l}{\textbf{Model-based bootstrap with pivot CI}} \\ \hline
50
& 0.100 & 0.234 & 0.191
& 0.100 & 0.181 & 0.251 & 0.154 & 0.161 & 0.176 \\
100
& 0.164 & 0.237 & \cellcolor{covheat!18!white}0.335
& 0.164 & 0.212 & 0.232 & 0.256 & 0.264 & \cellcolor{covheat!5!white}0.309 \\
500
& 0.262 & \cellcolor{covheat!76!white}0.451 & \cellcolor{covheat!88!white}0.476
& 0.262 & \cellcolor{covheat!48!white}0.395 & \cellcolor{covheat!76!white}0.451 & \cellcolor{covheat!88!white}0.477 & \cellcolor{covheat!100!white}0.504 & \cellcolor{covheat!95!white}0.490 \\
1{,}000
& 0.251 & \cellcolor{covheat!87!white}0.475 & \cellcolor{covheat!89!white}0.479
& 0.251 & \cellcolor{covheat!74!white}0.448 & \cellcolor{covheat!87!white}0.475 & \cellcolor{covheat!94!white}0.488 & \cellcolor{covheat!94!white}0.488 & \cellcolor{covheat!91!white}0.482 \\
\hline
\multicolumn{10}{l}{\textbf{Model-based bootstrap with percentile CI}} \\
\hline
\end{tabular}
\end{table}

\begin{table}[ht]
\centering
\caption{Empirical coverage (nominal $90\%$) for $Q_\star$ and $V_\star$ in RiverSwim. Comparison of model-based bootstrap, episodic bootstrap, and plug-in (CLT) CI, episode length $T=50$.}
\label{tab:riverswim-optimal-coverage-T50-90}
\begin{tabular}{c|ccc|cccccc}
\hline
& \multicolumn{3}{c|}{$Q_\star$ coverage} & \multicolumn{6}{c}{$V_\star$ coverage} \\ \hline
$n$
& $Q_\star(1,0)$ & $Q_\star(3,1)$ & $Q_\star(6,0)$
& $V_\star(1)$ & $V_\star(2)$ & $V_\star(3)$ & $V_\star(4)$ & $V_\star(5)$ & $V_\star(6)$ \\
\hline
\multicolumn{10}{l}{\textbf{Model-based bootstrap with pivot CI}} \\ \hline
50
& 0.230 & 0.633 & 0.298
& 0.230 & 0.531 & 0.650 & 0.300 & 0.295 & 0.278 \\
100
& 0.333 & 0.476 & 0.487
& 0.333 & 0.422 & 0.474 & 0.453 & 0.468 & 0.461 \\
500
& 0.453 & \cellcolor{covheat!9!white}0.718 & \cellcolor{covheat!53!white}0.806
& 0.453 & 0.599 & 0.694 & \cellcolor{covheat!32!white}0.763 & \cellcolor{covheat!52!white}0.803 & \cellcolor{covheat!55!white}0.809 \\
1{,}000
& 0.449 & \cellcolor{covheat!39!white}0.777 & \cellcolor{covheat!72!white}0.844
& 0.449 & 0.655 & \cellcolor{covheat!31!white}0.762 & \cellcolor{covheat!64!white}0.828 & \cellcolor{covheat!77!white}0.855 & \cellcolor{covheat!80!white}0.861 \\
\hline
\multicolumn{10}{l}{\textbf{Model-based bootstrap with percentile CI}} \\
\hline
\end{tabular}
\end{table}

\begin{table}[ht]
\centering
\caption{Empirical coverage (nominal $50\%$) for $Q_\star$ and $V_\star$ in RiverSwim. Comparison of model-based bootstrap, episodic bootstrap, and plug-in (CLT) CI, episode length $T=100$.}
\label{tab:riverswim-optimal-coverage-T100-50}
\begin{tabular}{c|ccc|cccccc}
\hline
& \multicolumn{3}{c|}{$Q_\star$ coverage} & \multicolumn{6}{c}{$V_\star$ coverage} \\ \hline
$n$
& $Q_\star(1,0)$ & $Q_\star(3,1)$ & $Q_\star(6,0)$
& $V_\star(1)$ & $V_\star(2)$ & $V_\star(3)$ & $V_\star(4)$ & $V_\star(5)$ & $V_\star(6)$ \\
\hline
\multicolumn{10}{l}{\textbf{Model-based bootstrap with pivot CI}} \\ \hline
100
& 0.156 & 0.226 & \cellcolor{covheat!22!white}0.343
& 0.156 & 0.207 & 0.219 & 0.263 & 0.278 & \cellcolor{covheat!6!white}0.311 \\
500
& 0.244 & \cellcolor{covheat!93!white}0.486 & \cellcolor{covheat!84!white}0.467
& 0.244 & \cellcolor{covheat!58!white}0.416 & \cellcolor{covheat!92!white}0.484 & \cellcolor{covheat!94!white}0.489 & \cellcolor{covheat!98!white}0.497 & \cellcolor{covheat!98!white}0.496 \\
1{,}000
& 0.291 & \cellcolor{covheat!100!white}0.507 & \cellcolor{covheat!100!white}0.505
& 0.291 & \cellcolor{covheat!100!white}0.506 & \cellcolor{covheat!100!white}0.507 & \cellcolor{covheat!100!white}0.511 & \cellcolor{covheat!100!white}0.505 & \cellcolor{covheat!100!white}0.506 \\
\hline
\multicolumn{10}{l}{\textbf{Model-based bootstrap with percentile CI}} \\ \hline
100
& \cellcolor{covheat!74!white}0.447 & \cellcolor{covheat!72!white}0.444 & 0.276
& \cellcolor{covheat!74!white}0.447 & \cellcolor{covheat!76!white}0.451 & \cellcolor{covheat!72!white}0.444 & \cellcolor{covheat!76!white}0.452 & \cellcolor{covheat!77!white}0.453 & \cellcolor{covheat!20!white}0.339 \\
\hline
\end{tabular}
\end{table}

\begin{table}[ht]
\centering
\caption{Empirical coverage (nominal $90\%$) for $Q_\star$ and $V_\star$ in RiverSwim. Comparison of model-based bootstrap, episodic bootstrap, and plug-in (CLT) CI, episode length $T=100$.}
\label{tab:riverswim-optimal-coverage-T100-90}
\begin{tabular}{c|ccc|cccccc}
\hline
& \multicolumn{3}{c|}{$Q_\star$ coverage} & \multicolumn{6}{c}{$V_\star$ coverage} \\ \hline
$n$
& $Q_\star(1,0)$ & $Q_\star(3,1)$ & $Q_\star(6,0)$
& $V_\star(1)$ & $V_\star(2)$ & $V_\star(3)$ & $V_\star(4)$ & $V_\star(5)$ & $V_\star(6)$ \\
\hline
\multicolumn{10}{l}{\textbf{Model-based bootstrap with pivot CI}} \\ \hline
100
& 0.333 & 0.492 & 0.512
& 0.333 & 0.460 & 0.485 & 0.477 & 0.515 & 0.483 \\
500
& 0.469 & \cellcolor{covheat!22!white}0.743 & \cellcolor{covheat!60!white}0.821
& 0.469 & 0.635 & \cellcolor{covheat!13!white}0.726 & \cellcolor{covheat!43!white}0.786 & \cellcolor{covheat!70!white}0.841 & \cellcolor{covheat!71!white}0.843 \\
1{,}000
& 0.479 & \cellcolor{covheat!59!white}0.819 & \cellcolor{covheat!82!white}0.865
& 0.479 & \cellcolor{covheat!3!white}0.706 & \cellcolor{covheat!56!white}0.811 & \cellcolor{covheat!77!white}0.855 & \cellcolor{covheat!85!white}0.870 & \cellcolor{covheat!87!white}0.875 \\
\hline
\multicolumn{10}{l}{\textbf{Model-based bootstrap with percentile CI}} \\ \hline
100
& 0.662 & 0.662 & 0.529
& 0.662 & 0.661 & 0.662 & 0.660 & 0.667 & 0.580 \\
\hline
\end{tabular}
\end{table}

%%%%%%%%%%%%%%%%%%%%%%%%%%%%%%%%%%%%%%%%%%%%%%%%%%%%%%%%%%%%
\clearpage
%%%%%%%%%%%%%%%%%%%%%%%%%%%%%%%%%%%%%%%%%%%%%%%%%%%%%%%%%%%%
\clearpage

\end{document}